%% file: main.tex
\documentclass{article}

\usepackage[in]{fullpage}
\usepackage[numbers]{natbib}
\usepackage{mathpazo}
\usepackage[utf8]{inputenc}
\usepackage[T1]{fontenc}
\usepackage{tabulary}
\usepackage{arydshln}
\usepackage{tablefootnote}
\usepackage{threeparttable}
\usepackage{footnote}
\usepackage{authblk}

\AtBeginDocument{\setlength\abovedisplayskip{4pt}}
\AtBeginDocument{\setlength\belowdisplayskip{4pt}}

\input{packages/defs}
\input{packages/header}

\input{packages/math_commands}
\usepackage{enumitem}
\usepackage{mathabx}

\makesavenoteenv{tabular}
\makesavenoteenv{table}

\definecolor{dmorange500}{HTML}{FF5F19}
\definecolor{dmblue300}{HTML}{2267EB}
\definecolor{dmred300}{HTML}{FF617B}
\hypersetup{
    colorlinks=true,
    citecolor=dmorange500,
    linkcolor=dmorange500,
    urlcolor=dmorange500
}

\renewcommand{\opt}{\operatorname{opt}}

\newcommand{\piE}{\pi^{\operatorname{E}}}

\renewcommand{\opt}{\operatorname{OPT}}
\newcommand{\truereward}{r^{\operatorname{true}}}
\newcommand{\gDE}{\mathcal{D}^{\operatorname{E}}}
\newcommand{\BE}{\operatorname{BE}}
\newcommand{\gec}{\operatorname{GEC}}
\newcommand{\SH}{D^2_{H}}
\newcommand{\discount}{\gamma}
\newcommand{\LOG}{\operatorname{NLL}}
\renewcommand{\textsf}[1]{#1}
\setlist{topsep=0pt,parsep=0pt,partopsep=0pt}

\theoremstyle{plain}
\newtheorem{theorem}{Theorem}[section]

\theoremstyle{definition}

\theoremstyle{remark}
\newtheorem{remark}[theorem]{Remark}

\title{Adversarial Imitation Learning with General Function Approximation: Theoretical Analysis and Practical Algorithms}

\author[1]{Tian Xu\thanks{Tian Xu and Zhilong Zhang contributed to this work equally. Email: \texttt{\{xut, zhangzl\}@lamda.nju.edu.cn}.}}
\author[1]{Zhilong Zhang$^*$}
\author[1]{Zexuan Chen}
\author[1,2]{Ruishuo Chen}
\author[4]{Yihao Sun}
\author[1]{Yang Yu\thanks{Corresponding author. Email: \texttt{yuy@nju.edu.cn}.}}

\affil[1]{National Key Laboratory for Novel Software Technology and School of Artificial Intelligence, Nanjing University}
\affil[2]{School of Mathematics, Nanjing University}
\affil[3]{Mila-Quebec Artificial Intelligence Institute and University of Montreal}

\date{\today}

\begin{document}

\maketitle

\begingroup
\renewcommand{\thefootnote}{}
\footnotetext{This paper is accepted in IEEE Transactions on Pattern Analysis and Machine Intelligence (TPAMI).}
\addtocounter{footnote}{0}
\endgroup

\begin{abstract}
Adversarial imitation learning (AIL), a prominent approach in imitation learning, has achieved significant practical success powered by neural network approximation. However, existing theoretical analyses of AIL are primarily confined to simplified settings—such as tabular and linear function approximation—and involve complex algorithmic designs that impede practical implementation. This creates a substantial gap between theory and practice. This paper bridges this gap by exploring the theoretical underpinnings of online AIL with general function approximation. We introduce a novel framework called optimization-based AIL (OPT-AIL), which performs online optimization for reward learning coupled with optimism-regularized optimization for policy learning. Within this framework, we develop two concrete methods: model-free OPT-AIL and model-based OPT-AIL. Our theoretical analysis demonstrates that both variants achieve polynomial expert sample complexity and interaction complexity for learning near-expert policies. To the best of our knowledge, they represent the first provably efficient AIL methods under general function approximation. From a practical standpoint, OPT-AIL requires only the approximate optimization of two objectives, thereby facilitating practical implementation. Empirical studies demonstrate that OPT-AIL outperforms previous state-of-the-art deep AIL methods across several challenging tasks.

\end{abstract}

\noindent\textbf{Keywords:} Imitation learning, adversarial learning, general function approximation, learning theory.

\section{Introduction}\label{sec:introduction}

Sequential decision-making tasks pervade real-world applications, where agents seek policies that maximize long-term returns. Reinforcement learning (RL) \citep{sutton2018reinforcement} provides a principled framework for developing such policies through environmental interaction and feedback. However, RL faces significant practical challenges: it demands carefully engineered reward functions and often requires millions of environment interactions to achieve acceptable performance \citep{mnih2015human, Janner19mbpo}. Imitation learning (IL) offers a compelling alternative by learning effective policies directly from expert demonstrations, eliminating the need for explicit reward design while dramatically reducing the number of required interactions. This sample efficiency has made IL increasingly attractive for real-world deployment, with demonstrated successes spanning recommendation systems \citep{chen2019recomendation, shi2019taobao} and generalist robot learning \citep{kim2024openvla, black2410pi0}.

IL encompasses two primary categories of methods: behavioral cloning (BC) and adversarial imitation learning (AIL). BC employs supervised learning to directly infer expert policies from demonstration data \citep{Pomerleau91bc, ross11dagger, brantley2020disagreement}. AIL takes a fundamentally different approach, employing an adversarial framework to match the expert's state-action distribution. This process involves the learner recovering an adversarial reward to maximize the policy value gap and subsequently learning a policy that minimizes this gap under the recovered reward. Building on these foundational principles, numerous practical algorithms have been developed \citep{Torabi18bco, Kostrikov19dac, brantley2020disagreement,jiang2020offline, ke19imitation_learning_as_f_divergence, ghasemipour2019divergence, garg2021iq-learn,luo2022transferable, li2024imitation}, achieving significant empirical advancements.

\begin{table*}[tbp]
\centering
\begin{threeparttable}
\caption{A summary of the expert sample complexity and interaction complexity.
Here $H$ is the horizon length, $\varepsilon$ is the desired imitation gap,
$|\gS|$ is the state space size, $|\gA|$ is the action space size,
$|\Pi|$ is the cardinality of the finite policy class $\Pi$,
$d$ is the feature dimension,
$d_{\gec}$ is the generalized eluder coefficient,
$\gN(\gR_h)$ and $\gN(\gQ_h)$ denote the covering numbers of the reward and Q-value classes.
We use $\widetilde{\gO}$ to hide logarithmic factors.\tnote{1}}
\label{table:summary-of-results}

\scriptsize
\setlength{\tabcolsep}{3pt}
\renewcommand{\arraystretch}{1.12}

\begingroup
\setlength{\aboverulesep}{0pt}
\setlength{\belowrulesep}{0pt}

\begin{tabular}{
  >{\centering\arraybackslash}p{2.0cm}
  !{\vrule width 0.4pt}
  >{\centering\arraybackslash}p{1.6cm}
  !{\vrule width 0.4pt}
  >{\centering\arraybackslash}p{4.2cm}
  !{\vrule width 0.4pt}
  >{\centering\arraybackslash}p{6.6cm}
}
\toprule
\addlinespace[0.4ex]
\textbf{Setting} &
\textbf{Algorithm} &
\textbf{Expert Sample Complexity} &
\textbf{Interaction Complexity} \\
\addlinespace[0.3ex]
\midrule
\addlinespace[0.3ex]

General Function Approximation
& BC \citep{foster2024behavior}\tnote{2}
& $\widetilde{\gO}\!\left(\dfrac{H^3 \log (\max_{h \in [H]} |\Pi_h|)}{\varepsilon^2}\right)$
& $0$ \\

Tabular MDPs
& OAL \citep{shani21online-al}
& $\widetilde{\gO}\!\left(\dfrac{H^2 |\gS|}{\varepsilon^2}\right)$
& $\widetilde{\gO}\!\left(\dfrac{H^4 |\gS|^2 |\gA|}{\varepsilon^2}\right)$ \\

Tabular MDPs
& MB-TAIL \citep{xu2023provably}
& $\widetilde{\gO}\!\left(\dfrac{H^{3/2} |\gS|}{\varepsilon}\right)$
& $\widetilde{\gO}\!\left(\dfrac{H^3 |\gS|^2 |\gA|}{\varepsilon^2}\right)$ \\

Linear Mixture MDPs
& OGAIL \citep{liu2021provably}
& $\widetilde{\gO}\!\left(\dfrac{H^3 d^2}{\varepsilon^2}\right)$
& $\widetilde{\gO}\!\left(\dfrac{H^4 d^3}{\varepsilon^2}\right)$ \\

Linear MDPs
& BRIG \citep{viano2024better}
& $\widetilde{\gO}\!\left(\dfrac{H^2 d}{\varepsilon^2}\right)$
& $\widetilde{\gO}\!\left(\dfrac{H^4 d^3}{\varepsilon^2}\right)$ \\

\addlinespace[0.4ex]
\midrule
\addlinespace[0.4ex]

General Function Approximation
& Model-free OPT-AIL
& $\widetilde{\gO}\!\left(\dfrac{H^2 \log \big(\max_{h \in [H]} \gN(\gR_h)\big)}{\varepsilon^2}\right)$
& $\widetilde{\gO}\!\left(
\dfrac{H^4 d_{\gec} \log \big(\max_{h \in [H]} \gN(\gQ_h)\gN(\gR_h)\big) + H^2}{\varepsilon^2}
\right)$ \\

General Function Approximation
& Model-based OPT-AIL
& $\widetilde{\gO}\!\left(\dfrac{H^2 \log \big(\max_{h \in [H]} \gN(\gR_h)\big)}{\varepsilon^2}\right)$
& $\widetilde{\gO}\!\left(
\dfrac{(d_{\gec} H + H^2)
\log \!\big(H \max_{h \in [H]} \gN(\gP_h; \log)\big)}{\varepsilon^2}
\right)$ \\

\addlinespace[0.2ex]
\bottomrule
\end{tabular}
\endgroup

\begin{tablenotes}
\item[1] We do not hide $\log(\gN(\gF))$ in $\widetilde{\gO}$ since it may be large for many function classes.
\item[2] We report the worst-case bound for BC, consistent with this paper.
\end{tablenotes}

\end{threeparttable}
\end{table*}

A notable empirical observation from these advances is that AIL often significantly outperforms BC \citep{ghasemipour2019divergence, ke19imitation_learning_as_f_divergence, Kostrikov19dac, garg2021iq-learn}. Understanding the theoretical foundations behind this superior performance has become a central focus of recent research \citep{xu2020error, shani21online-al,xu2022understanding, liu2021provably, xu2023provably, viano2024better}, particularly in the online setting. This theoretical analysis centers on two critical complexity measures for practical applications: \emph{expert sample complexity}, which quantifies the number of expert trajectories required, and \emph{interaction complexity}, which measures the number of trajectories when interacting with the environment. In the tabular setting, the best-known complexity result is achieved in \citep{xu2023provably}. They developed the MB-TAIL algorithm, which leverages advanced distribution estimation, achieving the expert sample complexity $\widetilde{\mathcal{O}} ( H^{3/2} |\mathcal{S}| / \varepsilon )$ and interaction complexity $\widetilde{\mathcal{O}} (  H^3 |\mathcal{S}|^2|\mathcal{A}| / \varepsilon^2 )$, where $|\gS|$ and $|\gA|$ are the state space size and action space size, respectively, $H$ is the horizon length and $\varepsilon$ is the desired value gap. Furthermore, \citep{liu2021provably,viano2024better} investigated the theory of AIL with linear function approximation. Notably, the BRIG approach proposed in \citep{viano2024better} uses linear regression for policy evaluation, achieving the expert sample complexity $\widetilde{\mathcal{O}}( H^2 d / {\varepsilon^2} )$ and interaction complexity $\tilde{\mathcal{O}} (H^4d^3 / \varepsilon^2 )$, where $d$ is the feature dimension. A complete summary of related results is provided in \cref{table:summary-of-results}.

Despite significant theoretical progress, a substantial gap persists between AIL theory and practice. First, current theoretical analyses are predominantly confined to restrictive settings—either tabular \citep{rajaraman2020fundamental, shani21online-al, xu2023provably} or linear function approximation \citep{liu2021provably, viano2024better}—which diverge markedly from practice where AIL algorithms typically employ general function approximation, particularly neural networks. Besides, most previous theoretical works involve algorithmic designs such as count-based \citep{shani21online-al, xu2023provably} or covariance-matrix-based \citep{liu2021provably, viano2024better} bonuses, which are tailored to their respective settings. Implementing such algorithmic designs in practice, where neural network approximation is employed, presents significant challenges \citep{yang2021exploration,tiapkin2022dirichlet}.

This paper aims to bridge the gap between theory and practice in AIL by developing provably efficient algorithms with general function approximation and providing practical implementations equipped with neural networks.

First, we introduce a new AIL framework called optimization-based adversarial imitation learning (OPT-AIL) for general function approximation. OPT-AIL decomposes adversarial imitation into two coupled optimization problems for reward learning and policy learning, respectively. For reward learning, recognizing that the reward loss evolves dynamically as the policy updates, we formulate reward learning as an online stochastic optimization problem and propose to invoke a no-regret approach to solve it. For policy learning, inspired by \citep{liu2024maximize}, we propose to solve an optimism-regularized optimization problem with the currently learned reward. Guided by this principle, we propose two concrete algorithms depending on the detailed policy update mechanism. In particular, model-free OPT-AIL first infers the Q-value functions by minimizing the optimism-regularized Bellman error and then derives the corresponding greedy policies. Model-based OPT-AIL instead learns the transition functions based on optimism-regularized maximum likelihood estimation and then derives policies by planning on the learned transition model.

Furthermore, we provide a comprehensive theoretical analysis for OPT-AIL in the general function approximation setup. Under mild assumptions, we prove that model-free OPT-AIL achieves the expert sample complexity $\widetilde{\gO} (H^2 \log (\max_{h \in [H]} \gN(\gR_h)) / \varepsilon^2)$ and interaction complexity $\widetilde{\gO} ((H^4 d_{\text{GEC}} \log (\max_{h \in [H]} \gN(\gQ_h) \gN(\gR_h) ) +H^2 ) / \varepsilon^2)$. Moreover, model-based OPT-AIL achieves the same expert sample complexity and an interaction complexity of $\widetilde{\gO} ((H d_{\gec}  + H^2) \log \lp H \max_{h \in [H]} \gN (\gP_h; \log) \rp / \varepsilon^2)$. Here $d_{\text{GEC}}$ is the generalized eluder coefficient, originally proposed in \citep{zhong2022posterior} to measure the complexity of RL with function approximation, which we adapt to the AIL problem. $\gN(\gR_h)$, $\gN (\gQ_h)$ and $\gN (\gP_h; \log)$ are the covering numbers of the reward class $\gR_h$, Q-value class $\gQ_h$ and transition class $\gP_h$, respectively. To the best of our knowledge, model-free and model-based OPT-AIL are the first provably efficient AIL approaches with general function approximation.

Finally, we offer a practical implementation of OPT-AIL, demonstrating its competitive performance on standard benchmarks. Notably, both model-free and model-based OPT-AIL require only the approximate optimization of two objectives, thereby facilitating their practical implementations with deep neural networks. Leveraging this advantage, we implement model-free and model-based OPT-AIL using neural network approximations and compare their performance against prior state-of-the-art (SOTA) deep AIL methods, which often lack theoretical guarantees. Experimental results indicate that OPT-AIL outperforms SOTA deep AIL approaches across several challenging DMControl tasks.\footnote{The code is available at \href{https://github.com/LAMDA-RL/OPT-AIL}{https://github.com/LAMDA-RL/OPT-AIL}.}

We conclude our contributions from three perspectives.

\begin{enumerate}
    \item This work introduces a new optimization-based adversarial imitation learning framework that accommodates general function approximation.
    \item This work establishes the first polynomial expert sample complexity and interaction complexity guarantees for adversarial imitation learning with general function approximation. 
    \item This work designs a practical neural-network-based implementation of optimization-based adversarial imitation learning, demonstrating its superior performance on standard benchmarks. 
\end{enumerate}

\section{Related Work}  \label{sec:related_work}
\subsection{Adversarial Imitation Learning}

The theoretical foundations of AIL have been extensively explored in numerous studies \citep{pieter04apprentice, syed07game, Sun19provably_efficient_ilfo, rajaraman2020fundamental, shani21online-al, liu2021provably,rajaraman2021value, xu2022understanding, swamy2022minimax,viano2022proximal, xu2023provably, viano2024better}. Early research \citep{pieter04apprentice, syed07game, Sun19provably_efficient_ilfo, rajaraman2020fundamental,xu2021error, swamy2022minimax, viano2022proximal} focused on ideal settings where either the transition function is known or exploratory data distributions are available, primarily addressing expert sample efficiency. Notably, under mild conditions, \citep{xu2022understanding} proved that AIL can achieve a horizon-free imitation gap bound $\gO ( \min\{ 1, \sqrt{|\gS| / N} \} )$, where $N$ denotes the number of expert trajectories. 
More recent work has shifted toward practically relevant scenarios, specifically online AIL with unknown transitions \citep{shani21online-al, liu2021provably, xu2023provably, viano2024better}. This line of work investigates both expert sample complexity and interaction complexity. These recent advancements were discussed in the previous section and thus will not be reiterated here. Most existing theoretical works focus on either tabular \citep{rajaraman2020fundamental, shani21online-al, xu2023provably} or linear function approximation settings \citep{liu2021provably, viano2024better}, and often lack practical implementations due to algorithmic designs tailored to specific settings. Our work addresses both limitations by providing theoretical guarantees for general function approximation while delivering a practical implementation with competitive empirical performance.

On the empirical side, there has been extensive research \citep{ho2016gail, Kostrikov19dac, Kostrikov20value_dice, ghasemipour2019divergence, ke19imitation_learning_as_f_divergence, garg2021iq-learn} developing practical AIL approaches that leverage general function approximation, particularly neural networks. A seminal method in this field is generative adversarial imitation learning (GAIL) \citep{ho2016gail}. In GAIL, a discriminator is trained to distinguish expert demonstrations from policy-generated trajectories, while the policy (or generator) learns to maximize the reward signal provided by the discriminator. Building on these foundations, recent methods have explored alternative formulations. Inverse Q-Learning \citep{garg2021iq-learn} and proximal point imitation learning \citep{viano2022proximal} represent a notable departure from GAIL by directly learning Q-value functions instead of reward models, achieving state-of-the-art performance in standard benchmarks.

The above approaches are model-free. In contrast, model-based AIL methods aim to leverage learned dynamics to enhance interaction efficiency. For instance, \citep{baram2017end} developed a model-based variant of GAIL that learns a transition model to render the full GAIL procedure differentiable. \citep{kolev2024efficient} introduced an ensemble-based approach termed CMIL, which learns an ensemble of transition models and uses the disagreement among them as a regularization term to constrain the policy toward the demonstration distribution. More recently, \citep{ren2024hybrid} introduced a hybrid model-based approach called HyPER that incorporates both expert and online data for policy learning, achieving high interaction efficiency. Despite their empirical progress, these model-free and model-based advances generally lack rigorous theoretical guarantees under general function approximation.

\subsection{General Function Approximation in Reinforcement Learning}

Our work is closely related to a body of research focused on general function approximation in RL \citep{osband2014eluder, jin2021bellman, liu2024maximize}. Notably, \citep{liu2024maximize} proposed an algorithmic framework that incorporates a unified objective to balance exploration and exploitation in RL, demonstrating a sublinear regret bound. In this paper, we adapt this algorithmic design to address several RL sub-problems within the context of AIL. While RL operates with a fixed, known reward, AIL must simultaneously infer reward functions from expert demonstrations and learn policies through environmental interaction. This dual learning process—where the reward function evolves as the policy improves—requires a fundamentally different theoretical analysis that accounts for the interdependence between reward estimation and policy optimization, highlighting a unique challenge in AIL compared to traditional RL.

\section{Preliminary}
\label{sec:preliminary}
\subsection{Markov Decision Process}

In this paper, we consider episodic Markov Decision Processes (MDPs), represented by the tuple $\mathcal{M} = (\mathcal{S}, \mathcal{A}, P, \truereward, H, s_1)$. Here, $\mathcal{S}$ and $\mathcal{A}$ denote the state and action spaces, respectively. $H$ signifies the planning horizon, while $s_1$ stands for the fixed initial state. The set $P^\star = \{ P^\star_1, \ldots, P^\star_{H} \}$ characterizes the non-stationary transition function of this MDP. Specifically, $P^\star_h(s_{h+1}|s_h, a_h)$ determines the probability of transitioning to state $s_{h+1}$ given state $s_h$ and action $a_h$ at time step $h$, where $h \in [H]$. Similarly, $\truereward = \{ \truereward_1, \ldots, \truereward_{H} \}$ outlines the unknown true reward function of this MDP. Without loss of generality, we assume $\truereward_h: \mathcal{S} \times \mathcal{A} \rightarrow [0, 1]$ for $h \in [H]$. A non-stationary policy is denoted by $\pi = \{ \pi_1, \ldots, \pi_H \}$ with $\pi_h: \mathcal{S} \rightarrow \Delta(\mathcal{A})$, where $\Delta(\mathcal{A})$ denotes the probability simplex. Here, $\pi_h(a|s)$ represents the probability of selecting action $a$ in state $s$ at time step $h$, for $h \in [H]$.

The quality of policy $\pi$ is evaluated by policy value: $V^{\pi} = \expect [ \sum_{h=1}^{H} \truereward_h (s_h, a_h) | a_h \sim \pi_h (\cdot|s_h), s_{h+1} \sim P^\star_h(\cdot|s_h, a_h), \forall h \in [H] ]$. We denote the Q-value function of policy $\pi$ at time step $h$ as $Q^{\pi}_h: \gS \times \gA \rightarrow \reals$, where $Q^{\pi}_h (s, a) = \expect_{\pi} [\sum_{\ell=h}^{H} \truereward_{\ell} (s_\ell, a_\ell) | s_h=s, a_h = a]$. The optimal Q-value function $Q^\star_h : \gS \times \gA \rightarrow \reals$ is defined as $Q^\star_h (s, a) := \sup_{\pi \in \Pi} Q^{\pi}_h (s, a)$. It is known that $Q^\star_h$ is the fixed point of the Bellman operator $\gT_{h}$: $Q^\star_h (s, a) = (\gT_h Q^\star_{h+1}) (s, a) := \truereward_h (s, a) + \expect_{s^\prime \sim P_h (\cdot|s, a)} [\max_{a^\prime \in \gA} Q^{\star}_{h+1} (s^\prime, a^\prime)] $. In other words, $Q^\star$ has zero Bellman error, i.e., $Q^\star_h (s, a) - (\gT_h Q^\star_{h+1}) (s, a) = 0$. 

\subsection{Imitation Learning}
The goal of IL is to learn a high-quality policy \emph{without} knowledge of the reward function $\truereward$. In pursuit of this objective, we typically posit the existence of a near-optimal expert policy $\piE$ capable of interacting with the environment to generate a dataset, comprising $N$ trajectories each of length $H$: $\gDE = \{ \tau = \lp s_1, a_1, s_2, a_2, \ldots, s_H, a_H \rp; a_h \sim \piE_h(\cdot|s_h), s_{h+1} \sim P^\star_h(\cdot|s_h, a_h), \forall h \in [H] \}$. Subsequently, the learner leverages this dataset $\gDE$ to mimic the behavior of the expert and thereby derives an effective policy. The quality of this imitation is measured by the \emph{imitation gap}~\citep{pieter04apprentice, ross2010efficient, rajaraman2020fundamental}: $V^{\piE} - V^{\pi}$, where $\pi$ represents the learned policy. Essentially, we hope that the learned policy can perfectly mimic the expert such that the imitation gap is small.

AIL is a prominent class of IL methods that imitate expert behavior through an adversarial learning process defined by $\min_{\pi} \max_{r} V^{\piE}_{r} - V^{\pi}_{r}$, where $V^{\pi}_{r}$ denotes the value of policy $\pi$ under reward $r$. In this framework, AIL infers a reward function that maximizes the value gap between the expert policy and the learning policy. Subsequently, it learns a policy that minimizes this value gap using the inferred reward. Essentially, AIL involves solving several RL sub-problems, as the outer optimization problem concerning the policy is equivalent to an RL problem under the inferred reward $r$.

\subsection{AIL with General Function Approximation}
This work considers AIL with general function approximation. In this setup, the learner first has access to a reward class $\gR = \gR_1 \times \gR_2 \times \ldots \times \gR_H$ with $\forall h \in [H], \gR_h \subseteq (\gS \times \gA \rar [0, 1])$ to infer the reward. We assume that $\gR$ captures the unknown true reward.
\begin{asmp}[Realizability of $\gR$]
\label{asmp:realizability_reward_class}
    The unknown true reward lies in the reward class, i.e., $\truereward \in \gR$.
\end{asmp}
Besides, we consider two types of AIL methods: model-free AIL and model-based AIL, which refer to solving several RL sub-problems in a model-free (or model-based) manner. For model-free AIL, the learner has access to a Q-value function class $\gQ = \gQ_1 \times \gQ_2 \times \ldots \times \gQ_H$ with $\forall h \in [H], \gQ_{h} \subseteq ( \gS \times \gA \rar [0, H])$. Since there is no reward at step $H+1$, we always set $Q_{H+1} \equiv 0$. Below, we present a standard assumption about the function class $\gQ$ that is commonly adopted in the literature of RL with function approximation \citep{jin2021bellman, zhong2022posterior, liu2024maximize}.

\begin{asmp}[Realizability and Bellman Completeness of $\gQ$]
\label{asmp:realizability_and_bellman_completeness_q_class}
    For reward $r \in \gR$, $Q^{\star, r} \in \gQ$, where $Q^{\star, r}$ denotes the optimal Q-value function under reward $r$. Besides, for reward $r \in \gR$, $\gT^{r}_{h} \gQ_{h+1} \subseteq \gQ_{h}, \; \forall h \in [H] $, where $\gT^{r}_{h}$ denotes the Bellman operator under reward $r$ and $\gT^{r}_{h} \gQ_{h+1} = \{ \gT^{r}_{h} Q_{h+1}: Q_{h+1} \in \gQ_{h+1}  \}$.  
\end{asmp}

In short, \cref{asmp:realizability_and_bellman_completeness_q_class} states that the Q-value class $\gQ$ should capture the optimal Q-value function and $\gQ$ is closed under the Bellman update.

As for model-based AIL, the learner has access to a transition function class $\gP = \gP_1 \times \gP_2 \times \ldots \gP_H$ with $\forall h \in [H], \gP_h \subseteq (\gS \times \gA \rightarrow \Delta (\gS))$. We assume that $\gP$ captures the true transition function.

\begin{asmp}[Realizability of $\gP$]
\label{asmp:realizability_model_class}
    The true transition function lies in the transition function class, i.e., $P^\star \in \gP$.
\end{asmp}
It is easy to verify that Assumptions \ref{asmp:realizability_reward_class} and \ref{asmp:realizability_and_bellman_completeness_q_class} (or Assumptions \ref{asmp:realizability_reward_class} and \ref{asmp:realizability_model_class}) are more general than the tabular MDP \citep{shani21online-al,xu2023provably}, linear mixture MDP \citep{liu2021provably} and linear MDP \citep{viano2024better} assumptions used in previous works.

When the function class contains a finite number of elements, its cardinality can be used to quantify its ``size''. However, for general function approximation, where the function class may contain an infinite number of elements, we utilize the standard $\varepsilon$-covering number \citep{wainwright2019high} to measure its complexity. In the following part, we present the definitions of $\varepsilon$-covering number for Q-value functions and transition functions, respectively.

\begin{defn}[$\varepsilon$-covering number]
\label{def:covering_number}
    For a function class $\gF \subseteq ( \gX \rightarrow \reals)$ (or $\gF \subseteq ( \gX \rightarrow \Delta (\gY))$), the $\varepsilon$-covering number of $\gF$, denoted as $\gN_{\varepsilon} (\gF)$ (or $\gN_{\varepsilon} (\gF, \log)$), is defined as the minimum integer $n$ such that there exists a finite subset $\gF^{\prime} \subseteq \gF$ with $ \labs \gF^{\prime} \rabs = n $ such that for any function $f \in \gF$, there exists $f^\prime \in \gF^\prime$ satisfying that $\max_{x \in \gX} \labs f (x) - f^{\prime} (x) \rabs \leq \varepsilon$ (or $\max_{x \in \gX, y \in \gY} \labs \log (f (y|x)) - \log (f^{\prime} (y|x)) \rabs \leq \varepsilon$).  
\end{defn}

\section{Optimization-based Adversarial Imitation Learning}
\label{sec:ail_oe}
In this section, we introduce Optimization-Based Adversarial Imitation Learning (OPT-AIL), a provably efficient framework comprising two specific algorithms: model-free OPT-AIL and model-based OPT-AIL. In \cref{subsec:theoretical_analysis_of_ail_oe}, we delve into the core components of OPT-AIL, which involves online reward optimization and optimism-regularized policy optimization. Subsequently, in \cref{subsec:theoretical_analysis}, we explore the underlying principles of OPT-AIL and establish theoretical guarantees under general function approximation.

\subsection{Algorithm Description}
\label{subsec:theoretical_analysis_of_ail_oe}

In this part, we present our provably efficient framework, OPT-AIL, with general function approximation, comprising two concrete algorithms: model-free OPT-AIL and model-based OPT-AIL (see Algorithms \ref{alg:mf_ail_oe} and \ref{alg:mb_ail_oe} for overviews).

We begin by recalling our theoretical objective: ensuring the algorithm outputs a policy with $\varepsilon$-imitation gap using finite expert samples and environment interactions. To obtain the final policy, we employ the standard online-to-batch conversion technique \citep{Hazan16introduction-to-oco}. During the learning process, the algorithm iteratively generates sequences of rewards $\{ r^k \}_{k=1}^K$ and policies $\{ \pi^k \}_{k=1}^K$, then outputs the policy $\widebar{\pi}$ uniformly sampled from $\{ \pi^k \}_{k=1}^K$. To analyze the imitation gap of $\widebar{\pi}$, we employ the following standard error decomposition lemma.

\begin{lem}
\label{lem:error_decomposition}
    Consider sequences of rewards $\{ r^k \}_{k=1}^K$ and  policies $\{ \pi^k \}_{k=1}^K$, and the policy $\widebar{\pi}$ uniformly sampled from $\{ \pi^k \}_{k=1}^K$. Then it holds that
    \begin{equation}
    \label{eq:error_decomposition}
        \begin{split}
            V^{\piE} - V^{\widebar{\pi}} = \underbrace{  \frac{1}{K} \sum_{k=1}^K \lp  V^{\piE}_{\truereward} - V^{\pi^k}_{\truereward} - \lp V^{\piE}_{r^k} - V^{\pi^k}_{r^k} \rp \rp }_{\text{reward error}}  + \underbrace{\frac{1}{K} \sum_{k=1}^K \lp V^{\piE}_{r^k} - V^{\pi^k}_{ r^k} \rp }_{\text{policy error}}.
        \end{split}
    \end{equation}
\end{lem}

\cref{lem:error_decomposition} demonstrates that achieving a small imitation gap requires controlling both reward error and policy error. The reward error quantifies the distance between the true reward $\truereward$ and the learned reward $r^k$ through the imitation gap, while the policy error measures the value difference between the expert policy $\piE$ and the learned policy $\pi^k$ under the inferred reward $r^k$. This policy error differs from the concept of regret in RL \citep{jin2021bellman, liu2024maximize}, where the reward function remains fixed.

Importantly, \cref{lem:error_decomposition} converts the adversarial formulation in AIL into two coupled optimization problems. To theoretically solve these coupled problems, we adopt an iterative approach where each iteration first updates the reward function and subsequently derives the corresponding policy. The following parts detail these reward and policy updates, which involve solving two optimization problems.

\subsubsection{Reward Update via Online Optimization (Line 3 in Algorithms \ref{alg:mf_ail_oe} and \ref{alg:mb_ail_oe}).}

This step aims to control the reward error. Specifically, in iteration $k$, we seek to learn a reward $r^{k}$ such that the error $V^{\pi^{k}}_{r^{k}} - V^{\piE}_{r^{k}} - (V^{\pi^{k}}_{\truereward} - V^{\piE}_{\truereward})$ remains small, which motivates minimizing the loss function $V^{\pi^{k}}_{r} - V^{\piE}_{r}$. However, since reward learning precedes policy learning, the loss function $V^{\pi^{k}}_{r} - V^{\piE}_{r}$ is unknown to the reward learner when learning $r^k$ because $\pi^k$ has not yet been determined. The loss can only be evaluated after the reward learner commits to its decision $r^k$. This sequential structure naturally motivates formulating the reward learning problem as an \emph{online} optimization problem \citep{Hazan16introduction-to-oco}, where the player cannot observe the loss function beforehand and only receives feedback after making a commitment.

Specifically, in iteration $k$, the reward learner selects \( r^{k} \) based on the previous loss functions $\{ V^{\pi^{i}}_{r} - V^{\piE}_{r} \}_{i=0}^{k-1}$, after which the current loss function $V^{\pi^{k}}_{r} - V^{\piE}_{r}$ is determined. Since the previous \emph{expected} loss functions $\{ V^{\pi^{i}}_{r} - V^{\piE}_{r} \}_{i=0}^{k-1}$ are unavailable, we instead minimize the \emph{estimated} loss functions. In particular, we construct an unbiased estimation $\gL^{i} (r)  =\widehat{V}^{\pi^i}_{r} - \widehat{V}^{\piE}_{r} $ for $V^{\pi^i}_{r} - V^{\piE}_{r} $ using expert demonstrations $\gDE$ and the trajectory $\tau^i$ collected by policy $\pi^i$, where
\begin{align*}
    \widehat{V}^{\pi^i}_{r} = \sum_{h=1}^H r_h (s^i_h, a^i_h), \; \widehat{V}^{\piE}_{r} = \frac{1}{N} \sum_{\tau \in \gDE} \sum_{h=1}^H r_h (\tau (s_h), \tau (a_h)) .
\end{align*}
Here $(\tau (s_h), \tau (a_h))$ is the state-action pair of trajectory $\tau$ visited at time step $h$ and $\tau^{i} = \{ s^i_1, a^i_1, \ldots, s^i_H, a^i_H \}$ is the trajectory collected by policy $\pi^i$. The ultimate goal of the reward learner is to minimize the cumulative losses $\sum_{k=1}^K \widehat{V}^{\pi^k}_{r^k} - \widehat{V}^{\piE}_{r^k}$. To achieve this, we employ a no-regret algorithm \citep{Hazan16introduction-to-oco}. We now formally define the reward optimization error resulting from running the no-regret algorithm.

\begin{defn}[Reward Optimization Error]
\label{def:rew_opt_error}
For any sequence of policies $\{ \pi^k \}_{k=1}^K$, the no-regret reward optimization algorithm sequentially outputs rewards $r^1, \ldots, r^K$. The reward optimization error $\varepsilon^{r}_{\operatorname{opt}}$ is defined as
\begin{align*}
    \varepsilon^{r}_{\operatorname{opt}}  :=  \frac{1}{K} \max_{r \in \gR} \sum_{k=1}^K  \widehat{V}^{\pi^k}_{r^k} - \widehat{V}^{\piE}_{r^k} -(\widehat{V}^{\pi^k}_{r}- \widehat{V}^{\piE}_{r}).
\end{align*}
\end{defn}
The reward optimization error aligns with the standard average regret in online optimization \citep{Hazan16introduction-to-oco}, a concept not extensively explored in the context of AIL. Various no-regret algorithms can achieve sublinear reward optimization error rates. For convex loss functions $\{ \gL^k (r) \}_{k=0}^K$ and convex reward class $\gR$, online projected gradient descent \citep{Hazan16introduction-to-oco} guarantees $\varepsilon^{r}_{\operatorname{opt}} = \gO (1/\sqrt{K})$. As for non-convex functions and sets, Follow-the-Perturbed-Leader achieves the same $\varepsilon^{r}_{\operatorname{opt}} = \gO (1/\sqrt{K})$ \citep{suggala2020online}.

In summary, we formulate reward learning in AIL as an online stochastic optimization problem and employ no-regret algorithms to solve it. 

\subsubsection{Policy Update via Optimism-Regularized Optimization (Lines 4-5 in Algorithms \ref{alg:mf_ail_oe} and \ref{alg:mb_ail_oe})}
Policy updates aim to control the policy error. In iteration $k$, the policy learner seeks to learn a policy $\pi^{k}$ such that the policy error $V^{\piE}_{r^{k}} - V^{\pi^k}_{r^{k}}$ is small, where $r^{k}$ is the learned reward function from the current iteration. This reduces to an RL problem with reward function $r^{k}$. Building upon \citep{liu2024maximize}, we propose model-free and model-based approaches that leverage optimism-regularized optimization to solve this RL subproblem.

\begin{algorithm}[htbp]
\caption{Model-free Optimization-based Adversarial Imitation Learning}
\label{alg:mf_ail_oe}

\begin{algorithmic}[1]
\REQUIRE{Reward class $\gR$, Q-value class $\gQ$, initialized reward $r^0$, policy $\pi^0$ and dataset $\gD^{0} = \emptyset$.}
\FOR{$k = 1, 2, \ldots, K$}
\STATE{Apply $\pi^{k-1}$ to roll out a trajectory $\tau^{k-1}$ and append it to the dataset $\gD^{k} = \gD^{k-1} \cup \{ \tau^{k-1} \}$.}
\STATE{Obtain $r^{k}$ by running a no-regret algorithm to solve the online optimization problem with observed loss functions  $\{ \gL^{i} (r)  \}_{i=0}^{k-1}$ up to an error $\varepsilon^{r}_{\opt}$, where $\gL^{i} (r)  =\widehat{V}^{\pi^i}_{r} - \widehat{V}^{\piE}_{r}$.}
\label{alg_line:reward_update}
\STATE{Obtain $Q^{k}$ by solving the optimization problem \eqref{eq:q_opt_obj} up to an error $\varepsilon^{Q}_{\opt}$.}
\label{alg_line:Q_update}
\STATE{Obtain $\pi^{k}$ by $\pi^{k}_h (s) = \argmax_{a \in \gA} Q^{k}_h (s, a)$.}
\label{alg_line:policy_update}
\ENDFOR
\ENSURE{$\widebar{\pi}$ sampled uniformly from $\{ \pi^k \}_{k=1}^K$.}
\end{algorithmic}
\end{algorithm}

\textbf{Model-free Policy Update. }The model-free approach learns Q-value functions based on Bellman error minimization and then derives the greedy policy. In particular, we first learn Q-value functions by solving the optimization problem of
\begin{equation}
\label{eq:q_opt_obj}
    \begin{split}
        &\min_{Q \in \gQ} \gL^{k} (Q) := \BE^{k} (Q) - \lambda_{Q} \max_{a \in \gA} Q_1 (s_1, a),
    \\
    & \text{with } \BE^{k} (Q) = \sum_{h=1}^H \gE_{h} (Q_h, Q_{h+1}; \gD^{k}, r^{k}) - \inf_{Q^\prime_h \in \gQ_h} \gE_{h} (Q^\prime_h, Q_{h+1}; \gD^{k}, r^{k}).
    \end{split}
\end{equation}
Here $\gE_{h} (Q_h, Q_{h+1}; \gD^{k}, r^{k}) = \sum_{i=0}^{k-1} (Q_h (s^i_h, a^i_h) - r^k_h - \max_{a^\prime \in \gA} Q_{h+1} (s^i_{h+1}, a^\prime) )^2$, $\gD^{k} = \{ \tau^{i} \}_{i=0}^{k-1}$ with $\tau^{i} = \{ s^i_1, a^i_1, \ldots, s^i_H, a^i_H \}$ and $\lambda_{Q} > 0 $ is the regularization coefficient. As shown in \citep{antos2008learning, jin2021bellman}, $\BE^{k} (Q)$ is an estimation of the true squared Bellman error of $Q$ with respect to reward $r^k$ and dataset $\gD^{k}$, i.e., $\sum_{h=1}^H \sum_{i=0}^{k-1} (Q_h (s^i_h, a^i_h) - (\gT^{r^k}_h Q_h) (s^i_h, a^i_h) )^2$. In \eqref{eq:q_opt_obj}, the primary term $\BE^{k}(Q)$ enforces Bellman consistency, while the optimism regularization term $\max_{a \in \mathcal{A}} Q_1(s_1,a)$, which serves as a proxy for the optimal initial value, biases the optimization toward optimistic value estimates and thereby promotes effective exploration. Moreover, this design is theoretically grounded, as it directly controls the gap between the true optimal value and its Q-based estimate, a key term in our policy error analysis. Notably, \cref{alg:mf_ail_oe} only requires approximately solving the optimization problem up to an error $\varepsilon^{Q}_{\opt}$ with $\varepsilon^{Q}_{\opt} = \gL^{k} (Q^k) - \min_{Q \in \gQ} \gL^{k} (Q)$. After obtaining $Q^{k}$, we derive $\pi^{k}$ as its greedy policy.

\textbf{Model-based Policy Update. } The model-based approach learns a transition model from the online data, then derives the policy through planning in the learned transition model. Concretely, the transition model is learned based on optimism-regularized maximum likelihood estimation (MLE).
\begin{equation}
\label{eq:model_opt_obj}
\min_{P \in \gP} \gL^k (P) := \text{NLL}^k (P) - \lambda_{P} V^{*}_{P, r^k}, \quad \text{with } \text{NLL}^k (P) = - \sum_{i=0}^{k-1} \sum_{h=1}^H \log \lp P_h \lp s^i_{h+1} | s^i_h, a^i_h \rp \rp.
\end{equation}
Here, $V^{\star}_{P, r^k}$ denotes the optimal value function induced by the learned transition model $P$ and reward $r^k$, serving as a model-based estimate of the true optimal value. The objective balances two distinct goals: the negative log-likelihood term $\text{NLL}^k (P)$ ensures the learned model fits the observed transition data, while the regularization term $V^{*}_{P, r^k}$ favors models that yield optimistic value estimates, thereby encouraging exploration. After learning the model, we perform planning within this model to derive the corresponding policy.

\begin{algorithm}[htbp]
\caption{Model-based Optimization-based Adversarial Imitation Learning}
\label{alg:mb_ail_oe}

\begin{algorithmic}[1]
\REQUIRE{Reward class $\gR$, Transition class $\gP$, initialized reward $r^0$, policy $\pi^0$ and dataset $\gD^{0} = \emptyset$.}
\FOR{$k = 1, 2, \ldots, K$}
\STATE{Apply $\pi^{k-1}$ to roll out a trajectory $\tau^{k-1}$ and append it to the dataset $\gD^{k} = \gD^{k-1} \cup \{ \tau^{k-1} \}$.}
\STATE{Obtain $r^{k}$ by running a no-regret algorithm to solve the online optimization problem with observed loss functions  $\{ \gL^{i} (r)  \}_{i=0}^{k-1}$ up to an error $\varepsilon^{r}_{\opt}$, where $\gL^{i} (r)  =\widehat{V}^{\pi^i}_{r} - \widehat{V}^{\piE}_{r}$.}
\STATE{Obtain $P^{k}$ by solving the optimization problem \eqref{eq:model_opt_obj} up to an error $\varepsilon^{P}_{\opt}$.}
\STATE{Obtain $\pi^{k}$ by $\pi^{k} = \argmax_{\pi} V^{\pi}_{P^k, r^k}$, where $V^{\pi}_{P^k, r^k}$ denotes the value of $\pi$ under $P^k$ and $r^k$.}
\ENDFOR
\ENSURE{$\widebar{\pi}$ sampled uniformly from $\{ \pi^k \}_{k=1}^K$.}
\end{algorithmic}
\end{algorithm}

\subsection{Theoretical Analysis}
\label{subsec:theoretical_analysis}

Having explained the algorithmic mechanisms of OPT-AIL, we now present its theoretical guarantees. To ensure the sample efficiency of solving RL sub-problems within AIL, we make a structural assumption on the underlying MDP. In particular, we assume that the MDP has a small generalized eluder coefficient (GEC). This coefficient, introduced in \citep{zhong2022posterior}, quantifies the inherent difficulty of exploring the MDP with function approximation in RL. We adapt this concept to AIL, where the reward function evolves across iterations.

\begin{asmp}[Low generalized eluder coefficient \citep{zhong2022posterior}]
\label{asmp:low_gec}
We assume that given a function class $\gF$, a discrepancy function $D$ and $\varepsilon > 0$, the generalized eluder coefficient $d_{\gec} (\varepsilon)$ is the smallest $d $ ($d \geq 0$) such that for any sequence of $\{ r^{k}  \}_{k=1}^K \subseteq \gR$, $\{ f^{k} \}_{k=1}^K \subseteq \gF$ and the deriving policies $\{ \pi^k = \pi (f^k; r^k) \}_{k=1}^{K}$,
\begin{align*}
     \sum_{k=1}^K V (f^k; r^k) - V^{\pi^k}_{r^k} \leq \inf_{\mu > 0} \frac{\mu}{2} \sum_{k=1}^K \sum_{i=1}^{k-1} \expect \ls \sum_{h=1}^H D (f^k, s_h, a_h; r^k) \bigg| \pi^{i} \rs + \frac{d}{2 \mu} + \sqrt{d H K} + \varepsilon H K.
\end{align*}
Here $V (f^k; r^k)$ is a prediction of the optimal value under $r^k$, building upon the function $f^k$.
\end{asmp}
\begin{remark}
\label{rem:gec}
    GEC is a generic complexity measure that can be applied in both the model-free function class and the model-based function class. For the model-free method where $\gF = \gQ$, we can choose $D (f^k, s_h, a_h; r^k) = (f^k(s_h, a_h) - (\gT_h^{r^k} f^k) (s_h, a_h))^2$, $\pi^k$ as the greedy policy regarding $f^k$ and $V (f^k; r^k) = \max_{a \in \gA} Q^k_1 (s_1, a)$. 
    
    As for the model-based case where $\gF = \gP$, we choose $D (f^k, s_h, a_h; r^k) = \SH (P^\star_h (\cdot|s_h, a_h), f^k_h (\cdot|s_h, a_h) )$, where $\SH$ is the squared Hellinger distance between two distributions. Besides, we derive $\pi^k$ as the optimal policy under $f^k$ and $r^k$, i.e., $\pi^k = \argmax_{\pi} V^{\pi}_{f^k, r^k}$, and choose $V (f^k; r^k) = \max_{\pi} V^{\pi}_{f^k, r^k}$.
\end{remark}

\begin{remark}
    To understand GEC, we can interpret $V (f^k; r^k) - V^{\pi^k}_{r^k}$ as the "out-of-distribution" prediction error evaluated on the next distribution induced by $\pi^k$, and interpret $\sum_{i=1}^{k-1} \expect [ \sum_{h=1}^H D (f^k, s_h, a_h; r^k) | \pi^{i} ]$ as the "in-distribution" training error evaluated on historical distributions generated by $\{ \pi^i \}_{i=1}^{k-1}$. From this perspective, GEC is to measure how well the "in-distribution" training error approximates the "out-of-distribution" prediction error. 
\end{remark}

\begin{remark}
    As demonstrated in \citep{zhong2022posterior}, the MDPs with low generalized eluder coefficient form a rich class of MDPs, which covers many well-known MDP instances such tabular MDPs, linear MDPs \citep{jin2020linear} and MDPs with low Bellman eluder dimension \citep{jin2021bellman}. Therefore, the assumption of low GEC is weaker than the tabular \citep{shani21online-al, xu2023provably} and linear MDP assumptions \citep{liu2021provably, viano2024better} used in previous works.   
\end{remark}
Now we are ready to present the theoretical guarantee of OPT-AIL.

\begin{thm}[Complexity Analysis of Model-free OPT-AIL]
\label{thm:mf_opt_ail_complexity}
Under Assumptions \ref{asmp:realizability_reward_class}, \ref{asmp:realizability_and_bellman_completeness_q_class}, and \ref{asmp:low_gec}. Fix any $\varepsilon \in (0, 1]$ and $\delta \in (0, 1]$. Consider Algorithm \ref{alg:mf_ail_oe} with regularization parameter $\lambda_{Q} = c_1 \sqrt{ (K H^3 \log (4 K H \gN_{\rho} (\gQ) \gN_{\rho} (\gR)/\delta) + K^2 H^3 \rho) / d_{\gec} }$, where $d_{\gec} := d_{\operatorname{GEC}} (\varepsilon/H)$, $\rho := c_2 \varepsilon^2 / (H^2 d_{\gec} + H)$, and $c_1, c_2$ are absolute constants. Then with probability at least $1-\delta$, the imitation gap satisfies $V^{\piE} - V^{\widebar{\pi}} \leq \varepsilon + \varepsilon^{r}_{\opt} + (\varepsilon^{Q}_{\opt}/\lambda_Q)$, provided the expert sample complexity and interaction complexity satisfy
\begin{align*}
    &N \gtrsim  \big( H^2 \log ( \max_{h \in [H]} \gN_{\rho} (\gR_h) /\delta )  \big) / \varepsilon^2,
    \\
    &K \gtrsim  \big( H^4 d_{\gec}  \log ( H d_{\gec} \max_{h \in [H]} \gN_{\rho} (\gQ_h) \gN_{\rho} (\gR_h) / (\delta \varepsilon)  )  + H^2 \log (1/\delta) \big) / \varepsilon^2.
\end{align*}
\end{thm}

\begin{thm}[Complexity Analysis of Model-based OPT-AIL]
\label{thm:mb_opt_ail_complexity}
    Under Assumptions \ref{asmp:realizability_reward_class}, \ref{asmp:realizability_model_class} and \ref{asmp:low_gec}. Fix any $\varepsilon \in (0, 1]$ and $\delta \in (0, 1]$, consider \cref{alg:mb_ail_oe} with regularization parameter $\lambda_{P} = c_1 \sqrt{ (K H \log ( H \max_{h \in [H]} \gN_{\rho} (\gP_h; \log)  /\delta) + K^2 H \rho) / d_{\gec} }$, where $d_{\gec} := d_{\operatorname{GEC}} (\varepsilon/H)$, $\rho := c_2 \varepsilon^2 / (H d_{\gec} + H)$, and $c_1, c_2$ are absolute constants. Then with probability at least $1-\delta$, the imitation gap satisfies $V^{\piE} - V^{\widebar{\pi}} \leq \varepsilon + \varepsilon^{r}_{\opt} + (\varepsilon^{P}_{\opt}/\lambda_P)$, provided the expert sample complexity and interaction complexity satisfy
    \begin{align*}
    &N \gtrsim  \big( H^2 \log (\max_{h \in [H]} \gN_{\rho} (\gR_h) /\delta )  \big) / \varepsilon^2,
    \\
    &K \gtrsim  (d_{\gec}  H + H^2) \log \big( H \max_{h \in [H]} \gN_{\rho} (\gP_h; \log) / \delta \big) / \varepsilon^2.
\end{align*}
\end{thm}
The proof of Theorems \ref{thm:mf_opt_ail_complexity} and \ref{thm:mb_opt_ail_complexity} is deferred to Appendix A.

\begin{remark}
Theorems \ref{thm:mf_opt_ail_complexity} and \ref{thm:mb_opt_ail_complexity} establish that both model-free and model-based OPT-AIL achieve polynomial expert sample complexity and interaction complexity under general function approximation. To the best of our knowledge, these results provide the first provably efficient online AIL algorithms for the general function approximation setting.
\end{remark}

\begin{remark}
OPT-AIL achieves an improvement in expert sample complexity compared to behavioral cloning \citep{foster2024behavior}, reducing the dependence by a factor of $\mathcal{O}(H)$. This improvement demonstrates that OPT-AIL provably mitigates the compounding error problem inherent in behavioral cloning under general function approximation.
\end{remark}

\begin{remark}
Finally, both model-free and model-based OPT-AIL only require the approximate optimization of two objectives, thereby facilitating practical implementation with neural networks, which will be presented in the next section.
\end{remark}

\section{Practical Implementation of OPT-AIL}
\label{sec:practical_implementation_of_ail_oe}
In this section, we provide a practical implementation for OPT-AIL, which is based on the stochastic-gradient-based methods; see Algorithms \ref{alg:practical_ail_oe} and \ref{alg:practical_mb_opt_ail} for an overview. We elaborate on the practical reward update and policy update in detail as follows.

\subsection{Practical Reward Update}
\label{subsec:practical_reward_update}
We now detail the practical implementation of the reward update using an online optimization approach. Recall that line 3 of Algorithms \ref{alg:mf_ail_oe} and \ref{alg:mb_ail_oe} employs a no-regret algorithm to solve the online optimization problem. We implement this using the Follow-the-Regularized-Leader (FTRL) algorithm \citep{Hazan16introduction-to-oco}, a classical no-regret optimization method. In iteration $k$, FTRL minimizes the sum of all historical loss functions with a regularization.
\begin{equation}
\label{eq:practical_reward_update}
     \min_{r \in \gR} \ell^{k} (r) := \sum_{i=0}^{k-1} \gL^{i} (r) + \beta \psi (r) = k \bigg( \expect_{\tau \sim \gD^{k}} \bigg[ \sum_{h=1}^H r_h (s^i_h, a^i_h) \bigg] - \expect_{\tau \sim \gDE}\bigg[  \sum_{h=1}^H r_h (s^i_h, a^i_h) \bigg] \bigg)  + \beta \psi (r),  
\end{equation}
where the expectation $\expect_{\gD} [\cdot]$ is taken over the empirical distribution of dataset $\gD$. Here $\psi (r)$ is the regularization term. In practice, we choose $\psi (r)$ as the gradient penalty \citep{Arjovsky2017wgan} of the reward model, which helps stabilize the learning process \citep{Kostrikov19dac}. According to \cref{eq:practical_reward_update}, the reward learner seeks to maximize the value gap between the expert policy and all previous policies.

Notably, \cref{eq:practical_reward_update} utilizes all historical samples in $\gD^k$ for the reward update, which aligns with off-policy reward learning \citep{Kostrikov19dac, Kostrikov20value_dice}. In particular, applying FTRL for the reward update and off-policy reward learning share the same main objective. While previous works \citep{Kostrikov19dac, Kostrikov20value_dice} demonstrated the practical effectiveness of off-policy reward learning, they lacked theoretical justification. Our work provides this missing explanation through an online optimization lens: the off-policy learning paradigm, which inherently aligns with FTRL, effectively controls the reward optimization error.

\subsection{Practical Policy Update}
In the following part, we introduce the practical implementations of the model-free and model-based policy updates.
\subsubsection{Practical Model-free Policy Update}
\begin{algorithm}[htbp]
\caption{Practical Implementation of Model-free OPT-AIL}
\label{alg:practical_ail_oe}
\begin{algorithmic}[1]
\REQUIRE{Initialized reward $r^0$, Q-value $Q^0$, target Q-value $\widebar{Q}^0 = Q^0$, policy $\pi^0$ and dataset $\gD^{0} = \emptyset$.}
\FOR{$k = 1, 2, \ldots, K$}
\STATE{Apply $\pi^{k-1}$ to roll out a trajectory $\tau^{k-1}$ and append it to the dataset $\gD^{k} = \gD^{k-1} \cup \{ \tau^{k-1} \}$.}
\STATE{Update the reward function by $r^{k} \leftarrow r^{k-1} - \alpha_{r} \nabla \ell^{k} (r)$ from \cref{eq:practical_reward_update}.}
\STATE{Update the Q-value function by $Q^{k} \leftarrow Q^{k-1} - \alpha_{Q} \nabla \ell^{k} (Q) $ from \cref{eq:practical_q_update}.}
\STATE{Update the policy by $\pi^{k} \leftarrow \pi^{k-1} + \alpha_{\pi} \nabla \ell^{k} (\pi)$ from \cref{eq:mf_practica_policy_update}.}
\STATE{Update the target Q-value by $\widebar{Q}^{k} \leftarrow \tau Q^{k} + (1-\tau) \widebar{Q}^{k-1}$}
\ENDFOR
\end{algorithmic}
\end{algorithm}

For the practical model-free policy update, we adopt the actor-critic framework \citep{haarnoja2018sac, Kostrikov19dac}, maintaining both a policy model $\pi$ and a Q-function model $Q$. Recall that line 4 of \cref{alg:mf_ail_oe} learns the Q-value function by minimizing the optimism-regularized Bellman error. Following \citep{cheng2022adversarially, bhardwaj2024adversarial}, we implement this principle using the temporal difference (TD) loss \citep{li2022note} of the Q-function model and its delayed target to approximate the theoretical Bellman error. Then we arrive at the following objective.
\begin{equation}
\label{eq:practical_q_update}
\min_{Q \in \mathcal{Q}} \ell^{k} (Q) := \expect_{\tau \sim \gD^k} \bigg[ \sum_{h=1}^H \bigg( Q_h (s_h, a_h) - r^k_h - \widebar{Q}^{k-1}_{h+1} (s_{h+1}, \pi^{k-1}) \bigg)^2 \bigg] - \lambda_{Q} Q_1 (s_1, \pi^{k-1}).
\end{equation}
Here $\widebar{Q} = \{ \widebar{Q}_1, \ldots, \widebar{Q}_H \}$ is the delayed target Q-function model. Besides, we define that $\widebar{Q}^{k-1}_{h+1} (s_{h+1}, \pi^{k-1}) := \expect_{a^\prime \sim \pi^{k-1}_{h+1} (\cdot|s_{h+1})} [ \widebar{Q}_{h+1} (s_{h+1}, a^\prime) ]$ where the previous greedy policy $\pi^{k-1}$ is used to approximate the maximum operator \citep{haarnoja2018sac}. Consequently, we derive the greedy policy by optimizing the objective of 
\begin{align}
\label{eq:mf_practica_policy_update}
    \max_{\pi} \ell^k (\pi) := \expect_{\tau \sim \gD^k} \ls \sum_{h=1}^H Q^{k}_{h} (s_h, \pi) \rs.
\end{align}

\subsubsection{Practical Model-based Policy Update}
\begin{algorithm}[htbp]
\caption{Practical Implementation of Model-based OPT-AIL}
\label{alg:practical_mb_opt_ail}
\begin{algorithmic}[1]
\REQUIRE{Initialized reward $r^0$, transition $P^0$, Q-value $Q^0$, target Q-value $\widebar{Q}^0 = Q^0$, policy $\pi^0$ and dataset $\gD^{0} = \emptyset$.}
\FOR{$k = 1, 2, \ldots, K$}
\STATE{Apply $\pi^{k-1}$ to roll out a trajectory $\tau^{k-1}$ and append it to the dataset $\gD^{k} = \gD^{k-1} \cup \{ \tau^{k-1} \}$.}
\STATE{Update the reward function by $r^{k} \leftarrow r^{k-1} - \alpha_{r} \nabla \ell^{k} (r)$ from \cref{eq:practical_reward_update}.}
\STATE{Rollout $\pi^{k-1}$ in $P^{k-1}$ to collect the dataset $\gD^{\pi^{k-1}, P^{k-1}}$.}
\STATE{Update the Q-value function by $Q^{k} \leftarrow Q^{k-1} - \alpha_{Q} \nabla \ell^{k} (Q) $ from \cref{eq:mb_practical_Q_update}.}
\STATE{Update the transition by $P^{k} \leftarrow P^{k-1} - \alpha_{P} \nabla \ell^{k} (P) $ from \cref{eq:model_gradient}.}
\STATE{Update the policy by $\pi^{k} \leftarrow \pi^{k-1} + \alpha_{\pi} \nabla \ell^{k} (\pi)$ from \cref{eq:mb_practical_policy_update}.}
\STATE{Update the target Q-value by $\widebar{Q}^{k} \leftarrow \tau Q^{k} + (1-\tau) \widebar{Q}^{k-1}$}
\ENDFOR
\end{algorithmic}
\end{algorithm}
For the practical model-based policy update, we also employ the actor-critic framework with three model components: a transition model $P$, policy model $\pi$ and Q-function model $Q$. See lines 4-8 in \cref{alg:practical_mb_opt_ail} for an overview. First, in line 4 of \cref{alg:mb_ail_oe}, the transition model is learned based on optimism-regularized MLE. 
\begin{align*}
    \min_{P} \min_{\pi}  - \expect_{\tau \sim \gD^k} \ls \sum_{h=1}^H \log \lp P_h \lp s^i_{h+1} | s^i_h, a^i_h \rp \rp \rs - \lambda_{P} V^{\pi}_{P, r^k}.  
\end{align*}
We optimize this joint objective using an alternating update strategy. Specifically, in iteration $k$, we update the transition model by optimizing the following objective.
\begin{align*}
\min_{P} \ell^k (P) := - \expect_{\tau \sim \gD^k} \ls \sum_{h=1}^H \log \lp P_h \lp s^i_{h+1} | s^i_h, a^i_h \rp \rp \rs - \lambda_{P} V^{\pi^{k-1}}_{P, r^k}.
\end{align*}
Here $\pi^{k-1}$ is the policy model obtained in the previous iteration $k-1$. Stochastic-gradient-based methods can be applied to optimize the above objective. In particular, following \citep{rigter2022rambo}, we can calculate the gradient as
\begin{equation}
\label{eq:model_gradient}
    \begin{split}
        &\nabla \ell^k (P) := - \expect_{\tau \sim \gD^k} \ls \sum_{h=1}^H \nabla \log \lp P_h \lp s_{h+1} | s_h, a_h \rp \rp \rs
    \\
    & - \lambda_{P} \expect_{\pi^{k-1}, P} \bigg[  \sum_{h=1}^H \nabla \log \lp P_h \lp s_{h+1} | s_h, a_h \rp \rp \bigg( r^k_h (s_h, a_h) + Q^{\pi^{k-1}, P, r^k}_{h+1} (s_{h+1}, a_{h+1}) - Q^{\pi^{k-1}, P, r^k}_h (s_{h}, a_h) \bigg)  \bigg].
    \end{split}
\end{equation}
To approximate the term $Q^{\pi^{k-1}, P, r^k}_h$ in \cref{eq:model_gradient}, we employ the TD loss to learn a Q-function model.
\begin{equation}
\label{eq:mb_practical_Q_update}
\min_{Q} \ell^k (Q) := \expect_{\tau \sim \gD^{\pi^{k-1}, P}} \bigg[ \sum_{h=1}^H \bigg( Q_h (s_h, a_h) - r^k_h (s_h, a_h) - \widebar{Q}^{k-1}_{h+1} (s_{h+1}, \pi^{k-1}) \bigg)^2 \bigg].
\end{equation}
Here $\gD^{\pi^{k-1}, P}$ is the set of trajectories collected by $\pi^{k-1}$ in $P$. Once the Q-function model $Q^k$ is obtained, we update the transition model using stochastic-gradient-based methods on \eqref{eq:model_gradient}.

Finally, we update the policy by maximizing the previously learned Q-function model.
\begin{align}
\label{eq:mb_practical_policy_update}
    \max_{\pi} \ell^k (\pi) := \expect_{\tau \sim \gD^k} \ls \sum_{h=1}^H Q^{k}_{h} (s_h, \pi) \rs.
\end{align}

\section{Experiments}\label{sec:exp}
In this section, we evaluate the expert sample efficiency and environment interaction efficiency of OPT-AIL through experiments. Below, we provide a brief overview of the experimental set-up, with detailed information in Appendix B.
\subsection{Experiment Set-up}
\textbf{Environment.} We conduct experiments on 8 tasks sourced from the feature-based DMControl benchmark \citep{tassa2018deepmind}, a leading benchmark in IL that offers a diverse set of continuous control tasks. For each task, we adopt online DrQ-v2 \citep{yarats2021mastering} to train an agent with sufficient environment interactions and regard the resultant policy as the expert policy. Then we roll out this expert policy to collect expert demonstrations. Each algorithm is tested over five trials with different random seeds, and in each run, we evaluate the policy return using Monte Carlo approximation with 10 trajectories.
\\
\textbf{Baselines.} Existing theoretical AIL approaches such as MB-TAIL \citep{xu2023provably} and OGAIL \citep{liu2021provably} rely on count-based or covariance-based bonuses that are difficult to implement with neural network approximations. Thus, we do not include these methods in our experiments. 
Instead, we compare OPT-AIL against prior deep IL methods: BC \citep{Pomerleau91bc}, PPIL \citep{viano2022proximal}, FILTER \citep{swamy2023inverse}, CMIL \citep{kolev2024efficient}, HyPE \citep{ren2024hybrid} and HyPER \citep{ren2024hybrid}, despite that most of them lack theoretical guarantees. Specifically, PPIL, FILTER, and HyPE are leading model-free AIL approaches, while CMIL and HyPER represent the state-of-the-art in model-based methods. We refer to our model-free and model-based variants as MF OPT-AIL and MB OPT-AIL, respectively. Implementation details are provided in Appendix B.

\definecolor{lightblue}{RGB}{173, 216, 230}
\definecolor{pink}{RGB}{255, 182, 193}

\begin{table*}[htbp]
\caption{Policy returns on 8 DMControl tasks over 5 random seeds following 500k environment interactions for model-free methods or 200k environment interactions for model-based methods. Here $^*$ denotes the model-based methods and the remaining are model-free. We use $\pm$ to denote the standard deviation over 5 random seeds. We highlight both the maximum mean values and values within one standard deviation thereof for model-free and model-based methods, respectively.}
\label{table:averaged_discounted_return_gradient_penalty}

\begingroup
\scriptsize
\setlength{\tabcolsep}{2.5pt}
\renewcommand{\arraystretch}{1.08}
\newlength{\hlwd}
\setlength{\hlwd}{5em}
\newlength{\hlwds}
\setlength{\hlwds}{2.5em}
\newcommand{\hlblue}[1]{\colorbox{lightblue}{\makebox[\hlwd][c]{#1}}}
\newcommand{\hlpink}[1]{\colorbox{pink}{\makebox[\hlwd][c]{#1}}}
\newcommand{\hlblues}[1]{\colorbox{lightblue}{\makebox[\hlwds][c]{#1}}}
\newcommand{\hlpinks}[1]{\colorbox{pink}{\makebox[\hlwds][c]{#1}}}
\centering
\resizebox{\textwidth}{!}{%
\begin{tabular}{cc!{\vrule}c!{\vrule}cccccc!{\vrule}cc}
\toprule
\textbf{Demos} & \textbf{DMC Task} & \textbf{Expert} & \textbf{BC} & \textbf{PPIL} & \textbf{FILTER} & \textbf{HyPE} & \textbf{HyPER*} & \textbf{CMIL*} & \textbf{MF OPT-AIL} & \textbf{MB OPT-AIL*} \\ \midrule

\multirow[=]{8}{*}{1} 
&Cartpole Swingup  &858.5  &307.1  &630.3  &421.4  &664.7  &\hlpinks{861.1}  &711.8  &\hlblue{862.1{\scriptsize $\pm0.9$}}  &858.6{\scriptsize $\pm3.5$}  \\
&Cheetah Run  &890.2  &64.7  &88.1  &212.9  &121.0  &406.1  &195.1  &\hlblue{348.7{\scriptsize $\pm82.0$}}  &\hlpink{577.0{\scriptsize $\pm77.1$}}  \\
&Finger Spin  &976.4  &0.8  &18.6  &910.8  &924.4  &\hlpinks{961.1}  &1.1  &\hlblue{970.6{\scriptsize $\pm1.1$}}  &941.2{\scriptsize $\pm14.6$}  \\
&Hopper Hop  &318.7  &1.2  &20.1  &0.0  &29.1  &\hlpinks{260.4}  &138.8  &\hlblue{284.7{\scriptsize $\pm15.9$}}  &\hlpink{274.7{\scriptsize $\pm21.9$}}  \\
&Hopper Stand  &939.5  &2.4  &61.6  &2.4  &44.1  &\hlpinks{360.0}  &198.5  &\hlblue{374.8{\scriptsize $\pm16.3$}}  &\hlpink{358.0{\scriptsize $\pm30.3$}}  \\
&Walker Run  &778.2  &32.6  &73.6  &331.5  &656.4  &662.6  &81.1  &\hlblue{753.0{\scriptsize $\pm7.9$}}  &\hlpink{734.9{\scriptsize $\pm11.7$}}  \\
&Walker Stand  &970.0  &192.4  &731.6  &655.8  &937.4  &692.7  &\hlpinks{959.0}  &\hlblue{957.2{\scriptsize $\pm10.9$}}  &937.8{\scriptsize $\pm13.2$}  \\
&Walker Walk  &961.4  &54.9  &874.2  &674.2  &\hlblues{895.0}  &703.2  &\hlpinks{832.3}  &\hlblue{913.1{\scriptsize $\pm28.5$}}  &\hlpink{849.1{\scriptsize $\pm26.7$}}  \\
\noalign{\vskip 1pt}
\cdashline{2-11}
\noalign{\vskip 2pt}
&Average  &774.6  &82.0  &312.3  &401.1  &534.0  &613.4  &389.7  &\hlblue{683.0{\scriptsize $\pm10.7$}}  &\hlpink{691.4{\scriptsize $\pm9.2$}}  \\

\midrule
\multirow[=]{8}{*}{4} 
&Cartpole Swingup  &858.5  &817.4  &680.3  &264.7  &714.2  &\hlpinks{862.0}  &697.3  &\hlblue{860.5{\scriptsize $\pm0.8$}}  &857.0{\scriptsize $\pm3.4$}  \\
&Cheetah Run  &890.2  &81.1  &121.5  &162.3  &313.5  &293.1  &165.1  &\hlblue{625.0{\scriptsize $\pm64.1$}}  &\hlpink{556.0{\scriptsize $\pm59.7$}}  \\
&Finger Spin  &976.4  &6.2  &724.7  &721.0  &963.6  &\hlpinks{956.1}  &428.9  &\hlblue{971.7{\scriptsize $\pm3.3$}}  &903.4{\scriptsize $\pm90.0$}  \\
&Hopper Hop  &318.7  &2.6  &74.7  &0.1  &41.4  &261.5  &172.7  &\hlblue{283.6{\scriptsize $\pm11.4$}}  &\hlpink{287.9{\scriptsize $\pm4.9$}}  \\
&Hopper Stand  &939.5  &6.1  &125.6  &4.0  &97.7  &\hlpinks{356.6}  &\hlpinks{495.8}  &\hlblue{368.9{\scriptsize $\pm16.1$}}  &\hlpink{383.8{\scriptsize $\pm12.2$}}  \\
&Walker Run  &778.2  &30.3  &683.3  &461.5  &691.6  &\hlpinks{741.3}  &158.9  &\hlblue{766.4{\scriptsize $\pm6.0$}}  &729.8{\scriptsize $\pm12.0$}  \\
&Walker Stand  &970.0  &296.5  &902.9  &919.0  &900.2  &660.7  &547.6  &\hlblue{945.3{\scriptsize $\pm19.9$}}  &\hlpink{933.3{\scriptsize $\pm16.7$}}  \\
&Walker Walk  &961.4  &236.6  &667.0  &834.4  &\hlblues{950.5}  &\hlpinks{873.8}  &\hlpinks{870.6}  &937.3{\scriptsize $\pm20.3$}  &\hlpink{891.0{\scriptsize $\pm30.2$}}  \\
\noalign{\vskip 1pt}
\cdashline{2-11}
\noalign{\vskip 2pt}
&Average  &774.6  &184.6  &497.5  &420.9  &584.1  &625.6  &439.4  &\hlblue{719.9{\scriptsize $\pm7.7$}}  &\hlpink{692.8{\scriptsize $\pm20.0$}}  \\

\midrule
\multirow[=]{8}{*}{7} 
&Cartpole Swingup  &858.5  &861.7  &852.2  &453.6  &613.3  &\hlpinks{861.6}  &628.6  &\hlblue{863.1{\scriptsize $\pm0.6$}}  &859.9{\scriptsize $\pm3.0$}  \\
&Cheetah Run  &890.2  &66.5  &135.0  &163.7  &289.7  &372.8  &126.4  &\hlblue{733.1{\scriptsize $\pm53.3$}}  &\hlpink{534.2{\scriptsize $\pm16.0$}}  \\
&Finger Spin  &976.4  &20.1  &951.0  &874.8  &961.3  &\hlpinks{967.0}  &187.6  &\hlblue{971.5{\scriptsize $\pm2.9$}}  &\hlpink{962.9{\scriptsize $\pm6.4$}}  \\
&Hopper Hop  &318.7  &1.0  &136.8  &0.3  &60.6  &\hlpinks{259.1}  &172.8  &\hlblue{291.0{\scriptsize $\pm3.4$}}  &\hlpink{269.3{\scriptsize $\pm20.4$}}  \\
&Hopper Stand  &939.5  &3.4  &232.3  &2.3  &67.7  &\hlpinks{357.4}  &\hlpinks{381.6}  &\hlblue{399.4{\scriptsize $\pm5.9$}}  &\hlpink{356.8{\scriptsize $\pm19.0$}}  \\
&Walker Run  &778.2  &41.5  &670.5  &295.9  &655.0  &\hlpinks{738.6}  &58.2  &\hlblue{739.5{\scriptsize $\pm27.5$}}  &707.2{\scriptsize $\pm31.9$}  \\
&Walker Stand  &970.0  &276.4  &782.2  &885.7  &\hlblues{972.3}  &606.7  &513.6  &961.8{\scriptsize $\pm11.7$}  &\hlpink{941.1{\scriptsize $\pm7.6$}}  \\
&Walker Walk  &961.4  &232.3  &255.6  &772.2  &\hlblues{934.3}  &\hlpinks{909.4}  &806.7  &\hlblue{939.3{\scriptsize $\pm7.7$}}  &886.6{\scriptsize $\pm8.0$}  \\
\noalign{\vskip 1pt}
\cdashline{2-11}
\noalign{\vskip 2pt}
&Average  &774.6  &187.9  &501.9  &431.1  &569.3  &634.1  &359.4  &\hlblue{737.3{\scriptsize $\pm4.8$}}  &\hlpink{689.7{\scriptsize $\pm8.1$}}  \\

\midrule
\multirow[=]{8}{*}{10}
&Cartpole Swingup  &858.5  &\hlblues{861.8}  &809.3  &271.3  &394.8  &\hlpinks{861.5}  &742.3  &\hlblue{861.5{\scriptsize $\pm1.2$}}  &\hlpink{858.8{\scriptsize $\pm4.4$}}  \\
&Cheetah Run  &890.2  &77.0  &137.2  &206.9  &363.6  &337.3  &27.8  &\hlblue{870.7{\scriptsize $\pm28.7$}}  &\hlpink{570.1{\scriptsize $\pm91.2$}}  \\
&Finger Spin  &976.4  &41.3  &935.6  &889.6  &962.3  &\hlpinks{955.4}  &620.8  &\hlblue{971.3{\scriptsize $\pm3.2$}}  &941.9{\scriptsize $\pm34.8$}  \\
&Hopper Hop  &318.7  &2.3  &286.5  &0.0  &70.6  &\hlpinks{264.4}  &155.0  &\hlblue{294.6{\scriptsize $\pm2.0$}}  &\hlpink{267.5{\scriptsize $\pm23.1$}}  \\
&Hopper Stand  &939.5  &162.6  &368.3  &3.7  &71.1  &\hlpinks{381.8}  &\hlpinks{436.6}  &\hlblue{399.4{\scriptsize $\pm5.9$}}  &\hlpink{359.8{\scriptsize $\pm25.7$}}  \\
&Walker Run  &778.2  &40.4  &656.5  &459.4  &677.3  &\hlpinks{734.1}  &53.6  &\hlblue{739.5{\scriptsize $\pm27.5$}}  &\hlpink{743.1{\scriptsize $\pm15.1$}}  \\
&Walker Stand  &970.0  &370.3  &830.8  &915.9  &\hlblues{974.7}  &469.9  &792.9  &967.5{\scriptsize $\pm6.0$}  &\hlpink{950.8{\scriptsize $\pm13.5$}}  \\
&Walker Walk  &961.4  &257.6  &467.2  &711.3  &\hlblues{935.2}  &\hlpinks{921.7}  &772.7  &\hlblue{940.3{\scriptsize $\pm8.1$}}  &\hlpink{908.3{\scriptsize $\pm8.4$}}  \\
\noalign{\vskip 1pt}
\cdashline{2-11}
\noalign{\vskip 2pt}
&Average  &774.6  &226.7  &561.4  &432.3  &556.2  &615.8  &450.2  &\hlblue{755.6{\scriptsize $\pm3.5$}}  &\hlpink{700.0{\scriptsize $\pm14.4$}}  \\

\bottomrule
\end{tabular}
}
\endgroup
\end{table*}

\begin{figure*}[htbp]
    \centering
    \includegraphics[width=\linewidth]{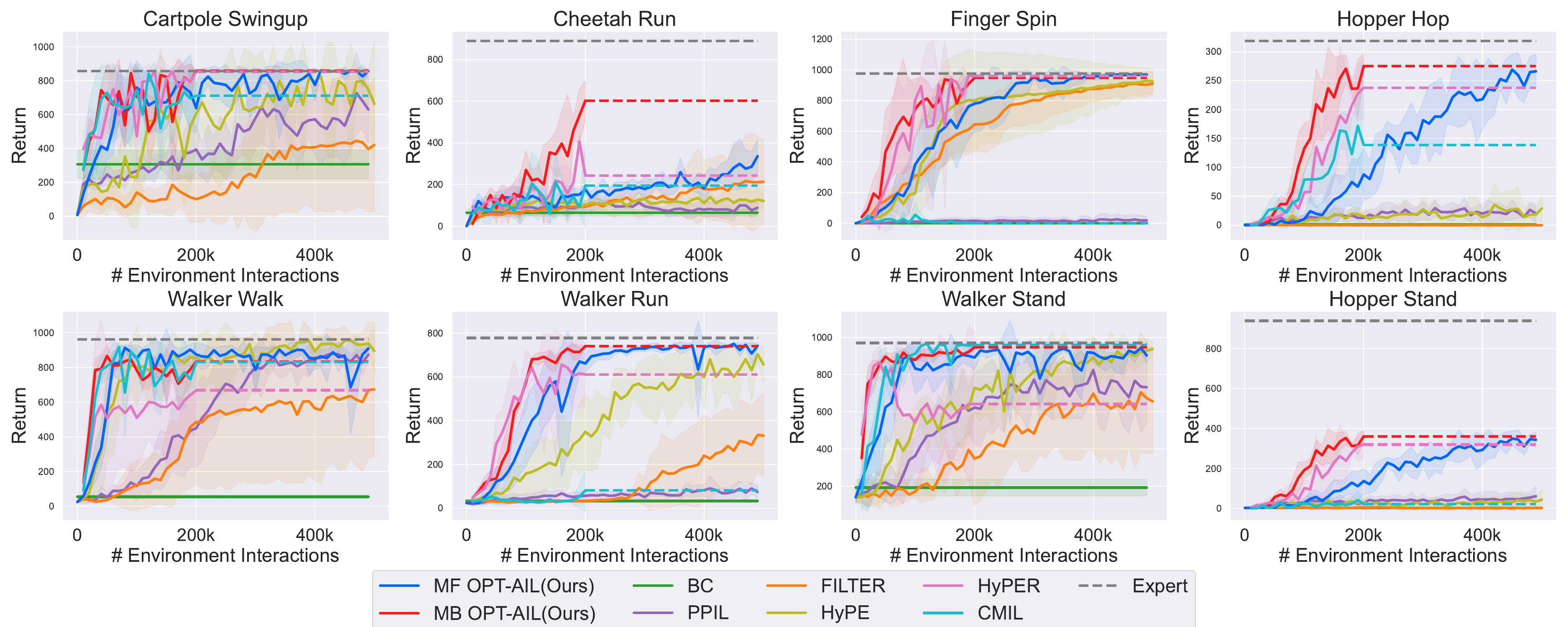}
    \caption{Learning curves on 8 DMControl tasks over 5 random seeds using 1 expert trajectory. Here the $x$-axis is the number of environment interactions and the $y$-axis is the return. The shaded region corresponds to the standard deviation over 5 random seeds. We run model-based approaches for 200k environment interactions, after which the last return is extended horizontally with dotted lines.}
    \label{fig:interaction_efficiency_n=1}
\end{figure*}

\subsection{Experiment Results}
\subsubsection{Expert Sample Efficiency}
Table \ref{table:averaged_discounted_return_gradient_penalty} presents the performance of all methods across varying numbers of expert trajectories. Model-free methods use 500k environment interactions, while model-based methods use 200k interactions due to their faster convergence properties.

The results demonstrate several key findings. First, both MF OPT-AIL and MB OPT-AIL consistently outperform BC across nearly all tasks, confirming our theoretical analysis that OPT-AIL effectively mitigates the compounding error problem inherent in BC under general function approximation. Second, our methods achieve substantial improvements over existing state-of-the-art approaches: MF OPT-AIL outperforms the leading model-free method HyPE by an average margin of 163 (approximately 29\% improvement), while MB OPT-AIL surpasses the top model-based method HyPER by 71 (approximately 11\% improvement). Particularly noteworthy is OPT-AIL's superior performance in low-data regimes, which are common in real-world applications where expert demonstrations are scarce. When trained on only a single expert trajectory, our method uniquely achieves expert-level or near-expert performance on challenging tasks such as \texttt{Walker Run} and \texttt{Walker Stand}, demonstrating exceptional sample efficiency.

\subsubsection{Environment Interaction Efficiency}
Figure \ref{fig:interaction_efficiency_n=1} shows the learning curves of different algorithms using a single expert trajectory. Overall, we can observe that model-based approaches achieve higher environment interaction efficiency than model-free ones, as expected due to their ability to leverage learned transition models to generate synthetic data. Among the model-based algorithms, MB OPT-AIL matches or exceeds the performance of CMIL and HyPER across all eight tasks in terms of interaction efficiency. Similarly, MF OPT-AIL distinguishes itself among model-free approaches, achieving near-expert performance with significantly fewer environment interactions than existing model-free methods, particularly on the \texttt{Hopper Hop}, \texttt{Walker Run}, and \texttt{Walker Stand} tasks.

\subsubsection{Ablation Studies on Optimism Regularization}
OPT-AIL incorporates an optimism regularization term in its objective. In this part, we conduct ablation studies to assess the impact of this design choice. We compare MF OPT-AIL and MB OPT-AIL with their respective variants that remove optimism regularization. As shown in \cref{fig:ablation}, we observe that both MF OPT-AIL and MB OPT-AIL match or exceed the convergence rates of their non-optimistic counterparts across all 8 tasks. The benefit is particularly evident in environments such as \texttt{Finger Spin}, \texttt{Hopper Hop}, and \texttt{Walker Run}, where OPT-AIL achieves notably faster convergence. These results suggest that optimism regularization can facilitate more effective exploration, thereby improving interaction efficiency and accelerating learning.

\begin{figure*}[htbp]
    \centering
    \includegraphics[width=0.95\linewidth]{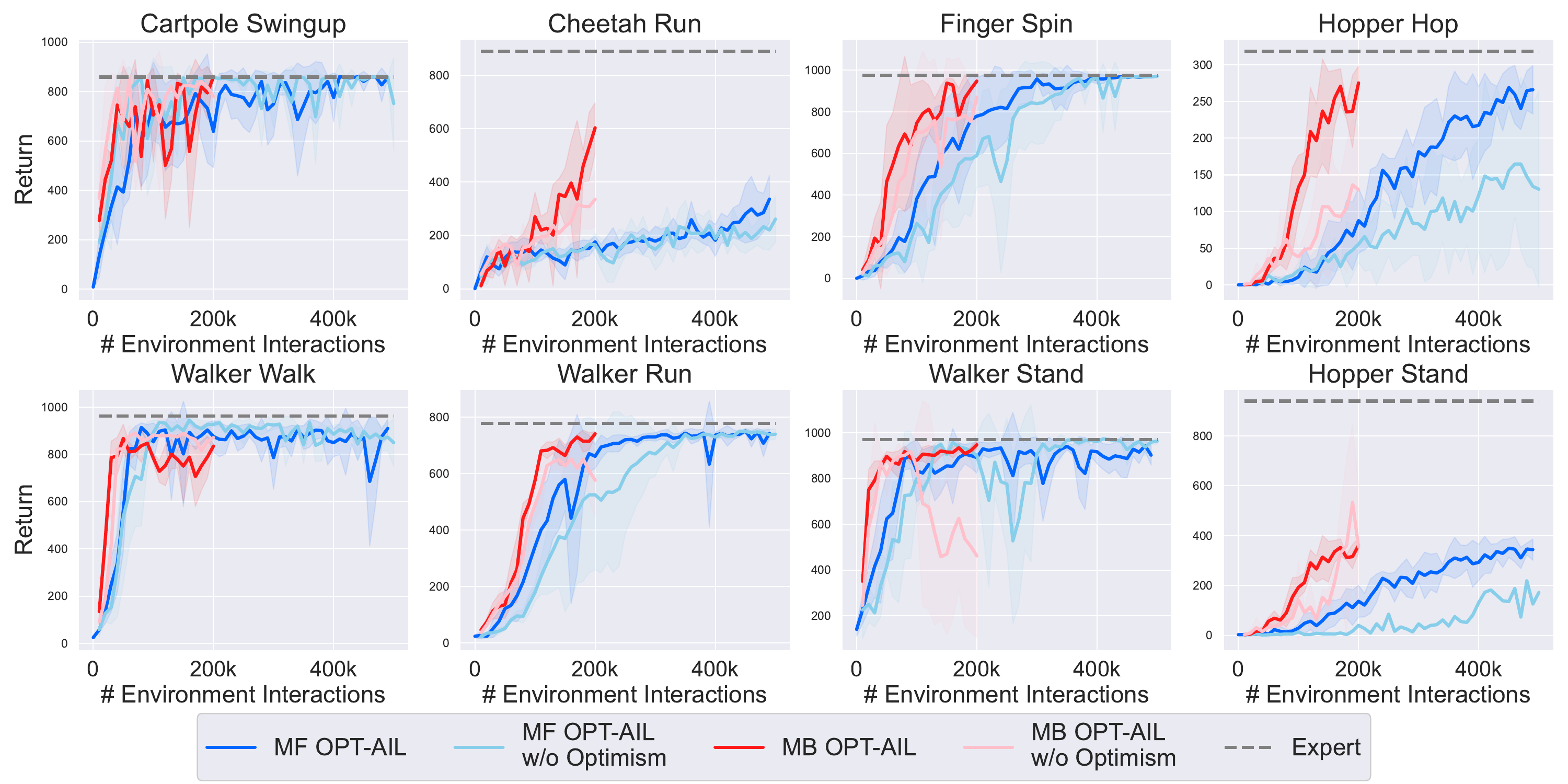}
    \caption{Learning curves on 8 DMControl tasks over 5 random seeds using 1 expert trajectory.}
    \label{fig:ablation}
\end{figure*}

\subsubsection{Performance of Reward Models}
In this part, we evaluate the quality of the learned reward models. Specifically, we measure the Pearson correlation between the ground-truth reward and the finally learned reward on a replay-buffer dataset collected by an independent RL process. We compare OPT-AIL against representative baselines that explicitly learn reward functions, including FILTER, CMIL, HyPE, and HyPER. As shown in \cref{table:reward_model_performance}, MF OPT-AIL and MB OPT-AIL achieve high average Pearson correlations of $0.7763$ and $0.7472$, respectively, substantially outperforming all competing methods. This performance advantage can be attributed to the reward update rule in OPT-AIL, which is derived from a principled online optimization process that explicitly minimizes reward error.

\begin{table*}[htbp]
    \caption{Pearson correlation between the ground-truth reward and the finally learned reward on 8 DMControl tasks over 5 random seeds.}
    \label{table:reward_model_performance}
    
    \setlength{\hlwd}{6.5em}
    \setlength{\hlwds}{3em}
    \newcommand{\hlblue}[1]{\colorbox{lightblue}{\makebox[\hlwd][c]{#1}}}
    \newcommand{\hlpink}[1]{\colorbox{pink}{\makebox[\hlwd][c]{#1}}}
    \newcommand{\hlblues}[1]{\colorbox{lightblue}{\makebox[\hlwds][c]{#1}}}
    \newcommand{\hlpinks}[1]{\colorbox{pink}{\makebox[\hlwds][c]{#1}}}
    
    \centering
    \renewcommand{\arraystretch}{1.3}
    \begin{tabular}{c!{\vrule}cccc!{\vrule}cc}
    \toprule
    \textbf{DMC Task}  & \textbf{FILTER} & \textbf{HyPE} & \textbf{HyPER*} & \textbf{CMIL*} & \textbf{MF OPT-AIL} & \textbf{MB OPT-AIL*} \\ \midrule
    Cartpole Swingup & 0.8978 & 0.8993 & 0.7999 & 0.8838 & \hlblue{0.9141{\scriptsize $\pm0.0040$}} & \hlpink{0.9109{\scriptsize $\pm0.0032$}} \\
    Cheetah Run & 0.8971 & \hlblues{0.9131} & 0.6524 & 0.8609 & \hlblue{0.9116{\scriptsize $\pm0.0006$}} & \hlpink{0.8945{\scriptsize $\pm0.0026$}} \\
    Finger Spin & \hlblues{0.8314} & 0.8213 & 0.6316 & 0.3419 & \hlblue{0.8323{\scriptsize $\pm0.0035$}} & \hlpink{0.6619{\scriptsize $\pm0.0046$}} \\
    Hopper Hop & 0.1805 & 0.1985 & 0.3439 & 0.4412 & \hlblue{0.4869{\scriptsize $\pm0.0029$}} & \hlpink{0.4673{\scriptsize $\pm0.0060$}} \\
    Hopper Stand & 0.2158 & 0.2676 & -0.0179 & 0.7717 & \hlblue{0.8141{\scriptsize $\pm0.0053$}} & \hlpink{0.8186{\scriptsize $\pm0.0067$}} \\
    Walker Run & 0.7313 & 0.6490 & 0.1947 & 0.8945 & \hlblue{0.9428{\scriptsize $\pm0.0009$}} & \hlpink{0.9104{\scriptsize $\pm0.0015$}} \\
    Walker Stand & 0.5135 & 0.3261 & -0.6355 & -0.1060 & \hlblue{0.6323{\scriptsize $\pm0.0030$}} & \hlpink{0.7716{\scriptsize $\pm0.0014$}} \\
    Walker Walk & 0.4271 & 0.2912 & \hlpinks{0.7399} & 0.5106 & \hlblue{0.6760{\scriptsize $\pm0.0098$}} & 0.5424{\scriptsize $\pm0.0061$} \\
    
    \noalign{\vskip 1pt}
    \cdashline{1-7}
    \noalign{\vskip 2pt}
    Average & 0.5868 & 0.5458 & 0.3386 & 0.5748 & \hlblue{0.7763{\scriptsize $\pm0.0037$}} & \hlpink{0.7472{\scriptsize $\pm0.0040$}} \\
    \bottomrule
    \end{tabular}
    \end{table*}

\section{Conclusion}
\label{sec:conclusion}
To narrow the gap between theory and practice in adversarial imitation learning, this paper investigates AIL with general function approximation. We develop OPT-AIL, a new AIL framework that centers on performing online optimization for reward learning and optimism-regularized optimization for policy learning. Under this framework, we propose two specific methods: model-free OPT-AIL and model-based OPT-AIL. In theory, both model-free and model-based OPT-AIL achieve polynomial expert sample complexity and interaction complexity for general function approximation. In practice, OPT-AIL only requires approximately solving two optimization problems, enabling efficient implementation with neural networks. Our experiments demonstrate that OPT-AIL outperforms prior SOTA methods in several challenging tasks, successfully bridging theoretical rigor with practical performance.

Several promising directions emerge for future work. In tabular MDPs, the current optimal expert sample complexity is $\gO (H^{3/2} / \varepsilon)$ \citep{rajaraman2020fundamental, xu2023provably}, which is better than $\gO (H^2/\varepsilon^2)$ attained in this paper. Therefore, a promising and valuable future direction would be to develop more advanced AIL approaches that achieve this expert sample complexity in the setting of general function approximation. Additionally, \citep{xu2022understanding} established horizon-free imitation gap bounds for tabular AIL, motivating the exploration of similar horizon-independent guarantees in the function approximation setting.

\section*{Acknowledgments}
We thank Ziniu Li and Yichen Li for their helpful discussions and feedback. This work was supported by the Fundamental Research Program for Young Scholars (PhD Candidates) of the National Science Foundation of China (623B2049) and Jiangsu Science Foundation (BK20243039).

\bibliographystyle{IEEEtran}
\bibliography{reference.bib}

\clearpage
\appendix
\input{appendix/appendix}

\end{document}

%% file: packages/defs.tex
\newcommand{\reals}{{\mathbb{R}}}

\newcommand{\argmax}{\mathop{\rm argmax}}

%% file: packages/header.tex
\usepackage{url}            %
\usepackage{booktabs}       %
\usepackage{multirow}    
\usepackage{amsfonts}       %
\usepackage{nicefrac}       %
\usepackage{microtype}      %
\usepackage{enumerate}
\usepackage{hhline}
\usepackage{makecell}
\usepackage{pifont}

\usepackage{graphicx} %
\usepackage{caption}
\usepackage{subcaption}
\usepackage{amsmath}
\usepackage{amsthm}
\usepackage{amssymb}
\usepackage{tikz}
\usepackage{xcolor}
\usetikzlibrary{arrows}

\allowdisplaybreaks

\usepackage{mathrsfs}

\usepackage{algorithm}
\usepackage{algorithmic}
\usepackage{hyperref}
\usepackage{bm}

\allowdisplaybreaks

\newcommand{\labs}{\left\vert}
\newcommand{\rabs}{\right\vert}

\newcommand{\opt}{\mathrm{opt}}

\newcommand{\expect}{\mathbb{E}}

\newtheorem{thm}{Theorem}
\newtheorem{lem}{Lemma}

\newtheorem{asmp}{Assumption}
\newtheorem{defn}{Definition}

\usepackage[capitalize,noabbrev]{cleveref}
\crefname{thm}{Theorem}{Theorems}
\crefname{lem}{Lemma}{Lemmas}
\crefname{cor}{Corollary}{Corollaries}
\crefname{prop}{Proposition}{Propositions}
\crefname{asmp}{Assumption}{Assumptions}
\crefname{defn}{Definition}{Definitions}
\crefname{oracle}{Oracle}{Oracles}
\crefname{fact}{Fact}{Facts}
\crefname{conj}{Conjecture}{Conjectures}
\crefname{rem}{Remark}{Remarks}
\crefname{claim}{Claim}{Claims}
\crefname{ec}{Empirical Observation}{Empirical Observations}

\definecolor{red}{rgb}{1, 0, 0}

\definecolor{green}{rgb}{0, 1, 0}
\definecolor{darkgreen}{rgb}{0.0, 0.2, 0.13}
\definecolor{darkseagreen}{rgb}{0.56, 0.74, 0.56}
\definecolor{officegreen}{rgb}{0.0, 0.5, 0.0}

\definecolor{blue}{rgb}{0, 0, 1}

\definecolor{orange}{rgb}{1, 0.4, 0.0}

%% file: appendix/appendix.tex
\newpage

\input{appendix/proof}

\input{appendix/experiment_details}

%% file: appendix/proof.tex
\section{Omitted Proof}
\label{sec:omitted_proof}

\subsection{Proof of Lemma \ref{lem:error_decomposition}}
\cref{lem:error_decomposition} presents an error decomposition theory in adversarial imitation learning. According to the definition of $\widebar{\pi}$, we have that
\begin{align*}
        V^{\piE} - V^{\widebar{\pi}}
        &= V^{\piE}_{\truereward} - V^{\widebar{\pi}}_{\truereward}
        \\
        &= \frac{1}{K} \sum_{k=1}^K V^{\piE}_{\truereward} - V^{\pi^k}_{\truereward}
        \\
        &= \frac{1}{K}   \sum_{k=1}^K \lp V^{\piE}_{\truereward} - V^{\pi^k}_{\truereward} - \lp V^{\piE}_{r^k} - V^{\pi^k}_{r^k} \rp \rp + \frac{1}{K} \sum_{k=1}^K V^{\piE}_{r^k} - V^{\pi^k}_{ r^k}.
\end{align*}
We complete the proof.

\subsection{Proof of Theorem \ref{thm:mf_opt_ail_complexity}}
\label{subsec:proof_of_theorem_ail_oe_complexity}
In this section, we present the proof of \cref{thm:mf_opt_ail_complexity}.

To prove \cref{thm:mf_opt_ail_complexity}, we need the following two useful lemmas which upper bound the reward error and policy error, respectively. Please refer to Appendix \ref{subsec:proof_of_lemma_reward_error_bound} and \ref{subsec:proof_of_lemma_policy_error_bound} for the detailed proof.

\begin{lem}[Upper Bound on Reward Error]
\label{lem:reward_error_bound}
Under \cref{asmp:realizability_reward_class}. Consider Algorithms \ref{alg:mf_ail_oe} and \ref{alg:mb_ail_oe}. For any fixed $\delta \in (0, 1]$, with probability at least $1-\delta$,
\begin{align*}
    &\frac{1}{K} \sum_{k=1}^K V^{\piE}_{\truereward} - V^{\pi^k}_{\truereward} 
    - \lp V^{\piE}_{r^k} - V^{\pi^k}_{r^k} \rp \\
    &\leq 2H \sqrt{\frac{\log (6 \max_{h \in [H]} \gN_{\rho} (\gR_h) /\delta)}{N}} 
    + 4 H \rho + 2 H \sqrt{ \frac{\log (3/\delta)}{K}} + \varepsilon^{r}_{\opt}. 
\end{align*}
\end{lem}

\begin{lem}[Upper Bound on Policy Error in Model-free OPT-AIL]
\label{lem:policy_error_bound}
Under Assumptions \ref{asmp:realizability_and_bellman_completeness_q_class} and \ref{asmp:low_gec}. Consider \cref{alg:mf_ail_oe}. For any fixed $\delta \in (0, 1]$, with probability at least $1-\delta$, it holds that 
\begin{align*}
    &\frac{1}{K} \sum_{k=1}^K V^{\piE}_{r^k} - V^{\pi^k}_{r^k} \\
    &\leq \frac{57 H^4 \log (4 K H \max_{h \in [H]} \gN_{\rho} (\gQ_h) \gN_{\rho} (\gR_h) /\delta)}{ \lambda_{Q}} + \frac{57 KH^3 \rho + \varepsilon^{Q}_{\opt}}{\lambda_{Q}} + \frac{\lambda_{Q} d_{\gec} (\varepsilon^\prime)}{2 K} + \sqrt{\frac{d_{\gec} (\varepsilon^\prime) H}{K}} + \varepsilon^\prime H.
\end{align*}
\end{lem}
Now we start to prove \cref{thm:mf_opt_ail_complexity}. With \cref{lem:error_decomposition}, we can derive that
\begin{align*}
        V^{\piE}_{\truereward} - V^{\widebar{\pi}}_{\truereward} 
        &= \frac{1}{K} \sum_{k=1}^K V^{\piE}_{\truereward} - V^{\pi^k}_{\truereward} 
        - \lp V^{\piE}_{r^k} - V^{\pi^k}_{r^k} \rp + \frac{1}{K} \sum_{k=1}^K V^{\piE}_{r^k} - V^{\pi^k}_{ r^k}.
\end{align*}
Furthermore, \cref{lem:reward_error_bound} and \cref{lem:policy_error_bound} offer upper bounds on reward error and policy error, respectively. By union bound, with probability at least $1-\delta$, we obtain
\begin{align*}
    &\quad V^{\piE}_{\truereward} - V^{\widebar{\pi}}_{\truereward}
    \\
    &\leq 2H \sqrt{\frac{\log (12 \max_{h \in [H]} \gN_{\rho} (\gR_h)  /\delta)}{N}} 
    + 4 H \rho 
    + 2 H \sqrt{ \frac{\log (6/\delta)}{K}} + \varepsilon^{r}_{\opt} 
    \\
    &\; + \frac{57 H^4 \log (8 K H \max_{h \in [H]} \gN_{\rho} (\gQ_h) \gN_{\rho} (\gR_h) /\delta)}{\lambda_{Q}} + \frac{57 KH^3 \rho + \varepsilon^{Q}_{\opt}}{\lambda_{Q}}
    + \frac{\lambda_{Q} d_{\gec} (\varepsilon^\prime)}{2 K} + \sqrt{\frac{d_{\gec} (\varepsilon^\prime) H}{K} } 
    + \varepsilon^\prime H. 
\end{align*}
We choose $\varepsilon^\prime = \varepsilon / H$ and obtain 
\begin{align*}
    &\quad V^{\piE}_{\truereward} - V^{\widebar{\pi}}_{\truereward}
    \\
    &\leq 2H \sqrt{\frac{\log (12 \max_{h \in [H]} \gN_{\rho} (\gR_h)  /\delta)}{N}} 
    + 4 H \rho + 2 H \sqrt{ \frac{\log (6/\delta)}{K}} 
    \\
    &\; + \frac{57 H^4 \log (8 K H \max_{h \in [H]} \gN_{\rho} (\gQ_h) \gN_{\rho} (\gR_h) /\delta)}{\lambda_{Q}} + \frac{57 KH^3 \rho}{\lambda_{Q}}
    + \frac{\lambda_{Q} d_{\gec}}{2 K} 
    + \sqrt{\frac{d_{\gec} H}{K} } + \varepsilon^{r}_{\opt} 
    + \frac{\varepsilon^{Q}_{\opt}}{\lambda_{Q}} 
    + \varepsilon, 
\end{align*}
where $d_{\gec} := d_{\gec} (\varepsilon/H)$. By choosing the regularization coefficient\\ $\lambda_{Q} = \sqrt{\frac{114K H^4 \log (8 K H \max_{h \in [H]} \gN_{\rho} (\gQ_h) \gN_{\rho} (\gR_h) /\delta)}{d_{\gec}} 
    + \frac{114 K^2H^3 \rho}{d_{\gec}}}$,
we further obtain
\begin{align*}
    &\quad V^{\piE}_{\truereward} - V^{\widebar{\pi}}_{\truereward}
    \\
    &\leq 2H \sqrt{\frac{\log (12 \max_{h \in [H]} \gN_{\rho} (\gR_h) /\delta)}{N}} 
    + 4 H \rho + 2 H \sqrt{ \frac{\log (6/\delta)}{K}} 
    \\
    &\; + \sqrt{\frac{114 H^4 d_{\gec} \log (8 K H \max_{h \in [H]} \gN_{\rho} (\gQ_h) \gN_{\rho} (\gR_h) /\delta) }{K}} + \sqrt{114H^3 d_{\gec} \rho}
    + \sqrt{\frac{d_{\gec} H}{K} } 
    + \varepsilon^{r}_{\opt} + \frac{\varepsilon^{Q}_{\opt}}{\lambda_{Q}} 
    + \varepsilon \\
    &\overset{\text{(a)}}{\leq} 2H \sqrt{\frac{\log (12 \max_{h \in [H]} \gN_{\rho} (\gR_h) /\delta)}{N}} 
    + 4 H \rho + 2 H \sqrt{ \frac{\log (6/\delta)}{K}} 
    \\
    &\; + \sqrt{\frac{114 H^4 d_{\gec} \log (8 K H \max_{h \in [H]} \gN_{\rho} (\gQ_h) \gN_{\rho} (\gR_h) /\delta) }{K}} + \sqrt{114H^3 d_{\gec} \rho}
    + \sqrt{\frac{d_{\gec} H}{K} } 
    + \varepsilon^{r}_{\opt} + \frac{\varepsilon^{Q}_{\opt}}{\lambda_{Q}} 
    + \varepsilon \\
    &\leq 2H \sqrt{\frac{\log (12 \max_{h \in [H]} \gN_{\rho} (\gR_h) /\delta)}{N}} 
    + 4 H \rho + 2 H \sqrt{ \frac{\log (6/\delta)}{K}} 
    \\
    &\; + 2 \sqrt{\frac{114 H^4 d_{\gec} \log (8 K H \max_{h \in [H]} \gN_{\rho} (\gQ_h) \gN_{\rho} (\gR_h) /\delta) }{K}} + \sqrt{54H^3 d_{\gec} \rho}
    + \varepsilon^{r}_{\opt} 
    + \frac{\varepsilon^{Q}_{\opt}}{\lambda_{Q}} 
    + \varepsilon \\
    &\overset{\text{(b)}}{\leq} 2H \sqrt{\frac{\log (12 \max_{h \in [H]} \gN_{\rho} (\gR_h) /\delta)}{N}}  
    + 2 H \sqrt{ \frac{\log (6/\delta)}{K}} 
    \\
    &\; + 24 \sqrt{\frac{ H^4 d_{\gec} \log (8 K H \max_{h \in [H]} \gN_{\rho} (\gQ_h) \gN_{\rho} (\gR_h) /\delta) }{K}} + \varepsilon^{r}_{\opt} 
    + \frac{\varepsilon^{Q}_{\opt}}{\lambda_{Q}} 
    + 3\varepsilon
\end{align*}
Inequality $\text{(a)}$ follows $\sqrt{a+b} \leq \sqrt{a} + \sqrt{b}, \; \forall a, b \geq 0$ and inequality (b) holds because of the choice $\rho = \varepsilon^2 / (54H^3 d_{\gec} + 4H)$. Now we determine the number of expert trajectories and the number of interaction trajectories. With \cref{lem:error_to_sample_complexity}, when the expert sample complexity and interaction complexity satisfies 
\begin{align*}
    N &\geq \lp 4H^2 \log \left( 12 \max_{h \in [H]} \gN_{\rho} (\gR_h) /\delta \right) \rp / \varepsilon^2, \\
    K &\geq \left( 2304  \left( H^4 d_{\gec}\log \left( 768 H^{3} d_{\gec}^{1/2} \max_{h \in [H]} \gN_{\rho} (\gQ_h) \gN_{\rho} (\gR_h) / (\delta \varepsilon) \right) + H^2 \log (6/\delta) \right) \right)  / \varepsilon^2,
\end{align*}
we have that
\begin{align*}
    V^{\piE}_{\truereward} - V^{\widebar{\pi}}_{\truereward} \leq 6 \varepsilon + \varepsilon^{r}_{\opt} + \frac{\varepsilon^{Q}_{\opt}}{\lambda_{Q}}.
\end{align*}
Scaling $\varepsilon$ as $\varepsilon/6$ completes the proof.

\subsection{Proof of Theorem \ref{thm:mb_opt_ail_complexity}}
To prove \cref{thm:mb_opt_ail_complexity}, we also need upper bounds on reward error and policy error. For reward error bound, we can leverage \cref{lem:reward_error_bound} because model-free OPT-AIL and model-based OPT-AIL share the same reward update procedure. Besides, we notice that model-free OPT-AIL and model-based OPT-AIL mainly differ in the policy update rule. The following lemma provides the policy error bound for model-based OPT-AIL.

\begin{lem}[Upper Bound on Policy Error in Model-based OPT-AIL]
\label{lem:policy_error_bound_mb}
Under Assumptions \ref{asmp:realizability_model_class} and \ref{asmp:low_gec}. Consider \cref{alg:mb_ail_oe}. For any fixed $\delta \in (0, 1]$, with probability at least $1-\delta$, it holds that
\begin{align*}
        &\quad \frac{1}{K} \sum_{k=1}^K V^{\piE}_{r^k} - V^{\pi^k}_{ r^k}
        \\
        &\leq \frac{2 H \log \lp H \max_{h \in [H]} \gN_{\rho} (\gP_h; \log) / \delta \rp}{\lambda_P} + \frac{2K H\rho + \varepsilon^P_{\opt}}{\lambda_P} + \frac{\lambda_{P} d_{\gec} (\varepsilon^\prime)}{2 K}  + \sqrt{\frac{d_{\gec} (\varepsilon^\prime) H}{ K}} + \varepsilon^\prime H.
\end{align*}
\end{lem}
Although \cref{lem:policy_error_bound_mb} and \cref{lem:policy_error_bound} share similar statements, their proofs differ substantially. For a detailed derivation, see Appendix \ref{subsec:proof_of_lem_policy_error_bound_mb}. Now we proceed to prove \cref{thm:mb_opt_ail_complexity}. With \cref{lem:error_decomposition}, we can derive that
\begin{align*}
        V^{\piE}_{\truereward} - V^{\widebar{\pi}}_{\truereward} 
        &= \frac{1}{K} \sum_{k=1}^K V^{\piE}_{\truereward} - V^{\pi^k}_{\truereward} 
        - \lp V^{\piE}_{r^k} - V^{\pi^k}_{r^k} \rp + \frac{1}{K} \sum_{k=1}^K V^{\piE}_{r^k} - V^{\pi^k}_{ r^k}.
\end{align*}
Furthermore, by leveraging \cref{lem:reward_error_bound} and \cref{lem:policy_error_bound_mb}, with probability at least $1-\delta$,
\begin{align*}
    &\quad V^{\piE}_{\truereward} - V^{\widebar{\pi}}_{\truereward} 
    \\
    &\leq 2H \sqrt{\frac{\log (12 \max_{h \in [H]} \gN_{\rho} (\gR_h) /\delta)}{N}} 
    + 4 H \rho + 2 H \sqrt{ \frac{\log (6/\delta)}{K}} + \varepsilon^{r}_{\opt} 
    \\
    &\; + \frac{2 H \log \lp 2 H \max_{h \in [H]} \gN_{\rho} (\gP_h; \log) / \delta \rp}{\lambda_P} + \frac{2K H\rho + \varepsilon^P_{\opt}}{\lambda_P}  + \frac{\lambda_{P} d_{\gec} (\varepsilon^\prime)}{2 K} + \sqrt{\frac{d_{\gec} (\varepsilon^\prime) H}{ K}} + \varepsilon^\prime H.
\end{align*}
We choose $\varepsilon^\prime = \varepsilon/H$ in the definition of GEC and define that $d_{\gec} = d_{\gec} (\varepsilon/H)$. Then with probability at least $1-\delta$, it holds that
\begin{align*}
    &\quad V^{\piE}_{\truereward} - V^{\widebar{\pi}}_{\truereward} 
    \\
    &\leq 2H \sqrt{\frac{\log (12 \max_{h \in [H]} \gN_{\rho} (\gR_h) /\delta)}{N}} 
    + 4 H \rho + 2 H \sqrt{ \frac{\log (6/\delta)}{K}} + \varepsilon^{r}_{\opt} + \frac{2 H \log \lp 2 H \max_{h \in [H]} \gN_{\rho} (\gP_h; \log) / \delta \rp}{\lambda_P} 
    \\
    &\;+ \frac{2K H\rho + \varepsilon^P_{\opt}}{\lambda_P}  + \frac{\lambda_{P} d_{\gec}}{2 K}  + \sqrt{\frac{d_{\gec} H}{ K}} + \varepsilon.
\end{align*}
Then we choose a proper $\lambda_{P}$ to upper bound the RHS of the above inequality. In particular, with $\lambda_P = \sqrt{ \lp 4KH \log \lp 2 H \max_{h \in [H]} \gN_{\rho} (\gP_h; \log) / \delta \rp + 4K^2H\rho \rp / d_{\gec}}$, it holds that
\begin{align*}
    &\quad V^{\piE}_{\truereward} - V^{\widebar{\pi}}_{\truereward} 
    \\
    &\leq 2H \sqrt{\frac{\log (12 \max_{h \in [H]} \gN_{\rho} (\gR_h) /\delta)}{N}} 
    + 4 H \rho + 2 H \sqrt{ \frac{\log (6/\delta)}{K}} + \frac{2 H \log \lp 2 H \max_{h \in [H]} \gN_{\rho} (\gP_h; \log) / \delta \rp}{\lambda_P} 
    \\
    &\; + \frac{2K H\rho}{\lambda_P}  + \frac{\lambda_{P} d_{\gec}}{2 K}  + \sqrt{\frac{d_{\gec} H}{ K}} + \varepsilon^{r}_{\opt} + \frac{\varepsilon^{P}_{\opt}}{\lambda_{P}} + \varepsilon
     \\
    &= 2H \sqrt{\frac{\log (12 \max_{h \in [H]} \gN_{\rho} (\gR_h) /\delta)}{N}} 
    + 4 H \rho + 2 H \sqrt{ \frac{\log (6/\delta)}{K}} 
    \\
    &\; + \sqrt{ \frac{ 2 d_{\gec}  (2 H \log \lp 2 H \max_{h \in [H]} \gN_{\rho} (\gP_h; \log) / \delta \rp + 2KH\rho ) }{K} } + \sqrt{\frac{d_{\gec} H}{ K}} + \varepsilon^{r}_{\opt} + \frac{\varepsilon^{P}_{\opt}}{\lambda_{P}} + \varepsilon
\\
    &\leq 2H \sqrt{\frac{\log (12 \max_{h \in [H]} \gN_{\rho} (\gR_h) /\delta)}{N}} 
    + 4 H \rho + 2 H \sqrt{ \frac{\log (6/\delta)}{K}} + 2 \sqrt{ \frac{  d_{\gec}  H \log \lp 2 H \max_{h \in [H]} \gN_{\rho} (\gP_h; \log) / \delta \rp }{K} } 
    \\
    &\; + \sqrt{4H d_{\gec} \rho} + \sqrt{\frac{d_{\gec} H}{ K}} + \varepsilon^{r}_{\opt} + \frac{\varepsilon^{P}_{\opt}}{\lambda_{P}} + \varepsilon
    \\
    &\leq 2H \sqrt{\frac{\log (12 \max_{h \in [H]} \gN_{\rho} (\gR_h) /\delta)}{N}} 
    + 4 H \rho + 2 H \sqrt{ \frac{\log (6/\delta)}{K}} + 3 \sqrt{ \frac{  d_{\gec}  H \log \lp 4 H \max_{h \in [H]} \gN_{\rho} (\gP_h; \log) / \delta \rp }{K} } 
    \\
    &\;+ \sqrt{4H d_{\gec} \rho}  + \varepsilon^{r}_{\opt} + \frac{\varepsilon^{P}_{\opt}}{\lambda_{P}} + \varepsilon.
\end{align*}
Here we choose that $\rho = \frac{\varepsilon^2}{(4Hd_{\gec} + 4H)}$ and obtain that
\begin{align*}
    &\quad V^{\piE}_{\truereward} - V^{\widebar{\pi}}_{\truereward} 
    \\
    &\leq 2H \sqrt{\frac{\log (12 \max_{h \in [H]} \gN_{\rho} (\gR_h) /\delta)}{N}} 
     + 2 H \sqrt{ \frac{\log (6/\delta)}{K}} + 3 \sqrt{ \frac{  d_{\gec}  H \log \lp 4 H \max_{h \in [H]} \gN_{\rho} (\gP_h; \log) / \delta \rp }{K} } 
     \\
     &\; + \varepsilon^{r}_{\opt} + \frac{\varepsilon^{P}_{\opt}}{\lambda_{P}} + 2\varepsilon + \varepsilon^2
    \\
    &\leq 2H \sqrt{\frac{\log (12 \max_{h \in [H]} \gN_{\rho} (\gR_h) /\delta)}{N}} + 5 \sqrt{ \frac{  (d_{\gec}  H + H^2) \log \lp 6 H \max_{h \in [H]} \gN_{\rho} (\gP_h; \log) / \delta \rp }{K} } 
    \\
    &\; + \varepsilon^{r}_{\opt} + \frac{\varepsilon^{P}_{\opt}}{\lambda_{P}} + 2\varepsilon + \varepsilon^2.
\end{align*}
When the expert sample complexity $N$ and interaction complexity $K$ satisfy that
\begin{align*}
    &N \geq \frac{4H^2 \log (12 \max_{h \in [H]} \gN_{\rho} (\gR_h) /\delta)}{\varepsilon^2},
    \\
    & K \geq \frac{25 (d_{\gec}  H + H^2) \log \lp 6 H \max_{h \in [H]} \gN_{\rho} (\gP_h; \log) / \delta \rp }{\varepsilon^2},
\end{align*}
we have that
\begin{align*}
    V^{\piE}_{\truereward} - V^{\widebar{\pi}}_{\truereward} \leq 5 \varepsilon + \varepsilon^{r}_{\opt} + \frac{\varepsilon^{P}_{\opt}}{\lambda_{P}}.
\end{align*}
We scale $5 \varepsilon$ as $\varepsilon$ and complete the proof.

\subsection{Proof of Lemma \ref{lem:reward_error_bound}}
\label{subsec:proof_of_lemma_reward_error_bound}
To prove \cref{lem:reward_error_bound}, we first perform the following error decomposition.
\begin{equation}
\label{eq:reward_error_decomposition}
\begin{split}
    & \quad \frac{1}{K} \sum_{k=1}^K V^{\piE}_{\truereward} - V^{\pi^k}_{\truereward} - \lp V^{\piE}_{r^k} - V^{\pi^k}_{r^k} \rp
    \\
    &= \frac{1}{K} \sum_{k=1}^K \lp \widehat{V}^{\piE}_{\truereward} - \widehat{V}^{\pi^k}_{\truereward} - \lp \widehat{V}^{\piE}_{r^k} - \widehat{V}^{\pi^k}_{r^k} \rp \rp + V^{\piE}_{\truereward} - \widehat{V}^{\piE}_{\truereward} + \frac{1}{K} \sum_{k=1}^K \widehat{V}^{\piE}_{r^k} - V^{\piE}_{r^k} + \frac{1}{K} \sum_{k=1}^K \widehat{V}^{\pi^k}_{\truereward} - V^{\pi^k}_{\truereward} 
    \\
    &\; + \frac{1}{K} \sum_{k=1}^K  V^{\pi^{k}}_{r^k} - \widehat{V}^{\pi^k}_{r^k}.
\end{split}   
\end{equation}
Recall that for any reward function $r$, $\widehat{V}^{\pi^i}_{r}$ and $\widehat{V}^{\piE}_{r}$ are unbiased estimations of $V^{\pi^i}_{r}$ and $V^{\piE}_{r}$, respectively.
\begin{align*}
     \widehat{V}^{\pi^i}_{r} = \sum_{h=1}^H r_h (s^i_h, a^i_h), \; \widehat{V}^{\piE}_{r} = \frac{1}{N} \sum_{\tau \in \gDE} \sum_{h=1}^H r_h (\tau (s_h), \tau (a_h)).
\end{align*}
The first term in the RHS of \cref{eq:reward_error_decomposition} is the estimated reward error while the remaining terms are estimation errors. To upper bound the first term, we have 
\begin{align*}
    &\quad \frac{1}{K} \sum_{k=1}^K  \widehat{V}^{\piE}_{\truereward} - \widehat{V}^{\pi^k}_{\truereward} - \lp \widehat{V}^{\piE}_{r^k} - \widehat{V}^{\pi^k}_{r^k} \rp
    \\
    &=  \frac{1}{K} \sum_{k=1}^K \widehat{V}^{\pi^k}_{r^k} - \widehat{V}^{\piE}_{r^k} - \lp  \widehat{V}^{\pi^k}_{\truereward} - \widehat{V}^{\piE}_{\truereward} \rp
    \\
    &\leq \frac{1}{K} \max_{r \in \gR}  \sum_{k=1}^K  \widehat{V}^{\piE}_{r^k} - \widehat{V}^{\pi^k}_{r^k} - \lp \widehat{V}^{\pi^k}_{r} - \widehat{V}^{\piE}_{r}  \rp
    \\
    &\overset{\text{(c)}}{=} \varepsilon^{r}_{\opt}.
\end{align*}
Equation (c) follows the definition of reward optimization error in \cref{def:rew_opt_error}. Then we can obtain
\begin{align*}
    & \quad \frac{1}{K} \sum_{k=1}^K V^{\piE}_{\truereward} - V^{\pi^k}_{\truereward} - \lp V^{\piE}_{r^k} - V^{\pi^k}_{r^k} \rp
    \\
    &\leq V^{\piE}_{\truereward} - \widehat{V}^{\piE}_{\truereward} + \frac{1}{K} \sum_{k=1}^K \widehat{V}^{\piE}_{r^k} - V^{\piE}_{r^k} + \frac{1}{K} \sum_{k=1}^K \widehat{V}^{\pi^k}_{\truereward} - V^{\pi^k}_{\truereward} + \frac{1}{K} \sum_{k=1}^K  V^{\pi^{k}}_{r^k} - \widehat{V}^{\pi^k}_{r^k} + \varepsilon^{r}_{\opt} .
\end{align*}
Then we proceed to upper bound the estimation errors. First, we first upper bound the estimation error caused by using $\widehat{V}^{\piE}_{r}$ to approximate $V^{\piE}_{r}$. In particular, we have that
\begin{align*}
    &\quad \labs \widehat{V}^{\piE}_{r} - V^{\piE}_{r}   \rabs 
    \\
    &= \labs \frac{1}{N} \sum_{\tau \in \gDE} \sum_{h=1}^H r_h (s_h (\tau), a_h (\tau) ) - \expect \ls \sum_{h=1}^H r_h (s_h, a_h) \bigg| \piE \rs \rabs
    \\
    &= \labs \sum_{h=1}^H \frac{1}{N} \sum_{\tau \in \gDE}  r_h (s_h (\tau), a_h (\tau) ) - \sum_{h=1}^H \expect \ls  r_h (s_h, a_h) \bigg| \piE \rs \rabs
    \\
    &\leq \sum_{h=1}^H \labs \frac{1}{N} \sum_{\tau \in \gDE}  r_h (s_h (\tau), a_h (\tau) ) - \expect \ls  r_h (s_h, a_h) \bigg| \piE \rs  \rabs.
\end{align*}
By Hoeffding's inequality \citep{wainwright2019high}, for any fixed timestep $h \in [H]$ and any fixed reward function $r_h \in \gR_h$, with probability at least $1-\delta$, we have that
\begin{align*}
    &\quad \labs \frac{1}{N} \sum_{\tau \in \gDE}  r_h (s_h (\tau), a_h (\tau) ) - \expect \ls  r_h (s_h, a_h) \bigg| \piE \rs  \rabs 
    \\
    &\leq \sqrt{\frac{\log (2 /\delta)}{N}}. 
\end{align*}
Let $(\gR_h)_{\rho}$ be a $\rho$-cover of $\gR$. By union bound, with probability at least $1-\delta$, for all $h \in [H]$ and all $\widehat{r}_h \in (\gR_h)_{\rho}$, we have that
\begin{align*}
    &\quad \labs \frac{1}{N} \sum_{\tau \in \gDE}  \widehat{r}_h (s_h (\tau), a_h (\tau) ) - \expect \ls  \widehat{r}_h (s_h, a_h) \bigg| \piE \rs  \rabs
    \\
    &\leq \sqrt{\frac{\log (2 H | (\gR_h)_{\rho} | /\delta)}{N}}. 
\end{align*}
Then with probability at least $1-\delta$, for all $\widehat{r} = (\widehat{r}_1, \ldots, \widehat{r}_H) \in (\gR_1)_{\rho} \times \ldots \times (\gR_1)_{\rho} $,
\begin{align*}
        &\quad \labs \widehat{V}^{\piE}_{\widehat{r}} - V^{\piE}_{\widehat{r}}   \rabs
        \\
        &\leq \sum_{h=1}^H \labs \frac{1}{N} \sum_{\tau \in \gDE}  \widehat{r}_h (s_h (\tau), a_h (\tau) ) - \expect \ls  \widehat{r}_h (s_h, a_h) \bigg| \piE \rs  \rabs
        \\
        &\leq \sum_{h=1}^H \sqrt{\frac{\log (2 |(\gR_h)_{\rho}| /\delta)}{N}}
        \\
        &\leq H \sqrt{\frac{\log (2 \max_{h \in [H]} |(\gR_h)_{\rho}| /\delta)}{N}}. 
\end{align*}
According to the definition of $\rho$-cover, for any reward function $r = (r_1, \ldots, r_H) \in \gR$, there exists $\widehat{r} = (\widehat{r}_1, \ldots, \widehat{r}_H) \in (\gR_1)_{\rho} \times \ldots \times (\gR_1)_{\rho}$ such that $\forall h \in [H], \; \max_{(s, a) \in \gS \times \gA } | r_h (s, a) - \widehat{r}_h (s, a)| \leq \rho$. Then we have that
\begin{align*}
    & \quad \labs \widehat{V}^{\piE}_{r}  - \widehat{V}^{\piE}_{\widehat{r}} \rabs
    \\
    &\leq \frac{1}{N} \sum_{\tau \in \gDE} \sum_{h=1}^H \labs r_h (s_h (\tau), a_h (\tau) ) - \widehat{r}_h (s_h (\tau), a_h (\tau) )  \rabs 
    \\
    &\leq H \rho.
\end{align*}
\begin{align*}
    &\quad \labs V^{\piE}_{r}  - V^{\piE}_{\widehat{r}} \rabs
    \\
    &\leq \expect \ls \sum_{h=1}^H \labs r_h (s_h, a_h) - \widehat{r}_h (s_h, a_h)  \rabs \bigg| \piE \rs 
    \\
    &\leq H \rho.
\end{align*}
Then, with probability at least $1-\delta$, for all reward function $r \in \gR$, we have that
\begin{align}
    \labs \widehat{V}^{\piE}_{r} - V^{\piE}_{r}   \rabs &\leq \labs \widehat{V}^{\piE}_{\widehat{r}} - V^{\piE}_{\widehat{r}}   \rabs + 2 H \rho \nonumber
    \\
    &\leq  H \sqrt{\frac{\log (2 \max_{h \in [H]} |(\gR_h)_{\rho}| /\delta)}{N}} + 2 H \rho \nonumber 
    \\
    &\leq H \sqrt{\frac{\log (2 \max_{h \in [H]} \gN_{\rho} (\gR_h) /\delta)}{N}} + 2 H \rho. \label{eq:reward_error_high_prob_one}      
\end{align}

    Now we have obtained the upper bound on the estimation error $| \widehat{V}^{\piE}_{r} - V^{\piE}_{r}   |$. Then we proceed to upper bound the estimation error $(1/K) \cdot \sum_{k=1}^K \widehat{V}^{\pi^k}_{\truereward} - V^{\pi^k}_{\truereward}$ and $(1/K) \cdot \sum_{k=1}^K  V^{\pi^{k}}_{r^k} - \widehat{V}^{\pi^k}_{r^k}$. With the Hoeffding's inequality \citep{wainwright2019high}, with probability at least $1-\delta$, we obtain that
\begin{align}
\label{eq:reward_error_high_prob_two}
    \frac{1}{K} \sum_{k=1}^K \widehat{V}^{\pi^k}_{\truereward} -  V^{\pi^k}_{\truereward} \leq H \sqrt{\frac{\log (1/\delta)}{K}}.
\end{align}
We proceed to analyze the term $\sum_{k=1}^K V^{\pi^k}_{r^k} - \widehat{V}^{\pi^k}_{r^k}$. Notice that $r^k$ are learned from historical trajectories $\{ \tau^1, \ldots, \tau^{k-1} \}$ and thus statistically depends on $\{ \tau^1, \ldots, \tau^{k-1} \}$. Therefore, $\widehat{V}^{\pi^1}_{r^1}, \cdots,  \widehat{V}^{\pi^k}_{r^k}$ are not independent and the standard Hoeffding's inequality is not applicable. To address this issue, we apply Azuma-Hoeffding's inequality \citep{wainwright2019high} for martingale. In particular, we define $\gF^{k}$ as the filtration induced by $\{ \tau^1, \cdots, \tau^k \}$ and can obtain that
\begin{align*}
    \expect \ls V^{\pi^k}_{r^k} - \widehat{V}^{\pi^k}_{r^k}   | \gF^{k-1}\rs = 0.
\end{align*}
Therefore, $\{ (V^{\pi^k}_{r^k} - \widehat{V}^{\pi^k}_{r^k}, \gF^{k}) \}_{k=1}^{\infty}$ is a martingale difference sequence. With Azuma-Hoeffding's inequality, we can derive that with probability at least $1-\delta$,
\begin{align}
\label{eq:reward_error_high_prob_three}
    \frac{1}{K} \sum_{k=1}^K V^{\pi^k}_{r^k} - \widehat{V}^{\pi^k}_{r^k} \leq H \sqrt{ \frac{\log (1/\delta)}{K}}. 
\end{align}
In summary, we have derived the following three high-probability inequalities: \cref{eq:reward_error_high_prob_one}, \cref{eq:reward_error_high_prob_two} and \cref{eq:reward_error_high_prob_three}. With union bound, with probability at least $1-\delta$, the following three events hold.
\begin{align*}
    \forall r \in \gR, \labs \widehat{V}^{\piE}_{r} - V^{\piE}_{r}   \rabs &\leq  H \sqrt{\frac{\log (6 \max_{h \in [H]} \gN_{\rho} (\gR_h) /\delta)}{N}} + 2 H \rho.
\end{align*}
\begin{align*}
    \frac{1}{K} \sum_{k=1}^K \widehat{V}^{\pi^k}_{\truereward} -  V^{\pi^k}_{\truereward} \leq H \sqrt{\frac{\log (3/\delta)}{K}}.
\end{align*}
\begin{align*}
    \frac{1}{K} \sum_{k=1}^K V^{\pi^k}_{r^k} - \widehat{V}^{\pi^k}_{r^k} \leq H \sqrt{ \frac{\log (3/\delta)}{K}}.
\end{align*}
With the above three inequalities, we can derive that
\begin{align*}
    & \quad \frac{1}{K} \sum_{k=1}^K V^{\piE}_{\truereward} - V^{\pi^k}_{\truereward} - \lp V^{\piE}_{r^k} - V^{\pi^k}_{r^k} \rp
    \\
    &\leq V^{\piE}_{\truereward} - \widehat{V}^{\piE}_{\truereward} + \frac{1}{K} \sum_{k=1}^K \widehat{V}^{\piE}_{r^k} - V^{\piE}_{r^k} + \frac{1}{K} \sum_{k=1}^K \widehat{V}^{\pi^k}_{\truereward} - V^{\pi^k}_{\truereward} + \frac{1}{K} \sum_{k=1}^K  V^{\pi^{k}}_{r^k} - \widehat{V}^{\pi^k}_{r^k} + \varepsilon^{r}_{\opt}
    \\
    &\leq 2H \sqrt{\frac{\log (6 \max_{h \in [H]} \gN_{\rho} (\gR_h) /\delta)}{N}} + 4 H \rho + 2 H \sqrt{ \frac{\log (3/\delta)}{K}} + \varepsilon^{r}_{\opt} .
\end{align*}
We complete the proof.

\subsection{Proof of Lemma \ref{lem:policy_error_bound}}
\label{subsec:proof_of_lemma_policy_error_bound}
To prove \cref{lem:policy_error_bound}, we need the following two auxiliary lemmas. The detailed proof is presented in \cref{subsec:proof_of_lemma_empirical_error_of_optimal_q} and \cref{subsec:proof_of_lemma_empirical_error_of_qk}.

\begin{lem}
\label{lem:empirical_error_of_optimal_q}
   For any fixed $\delta \in (0, 1]$, with probability at least $1-\delta$, for all $k \in [K]$,
    \begin{align*}
        &\quad \BE^{k} (Q^{\star, r^k}) 
        \\
        &\leq 16H^4 \log \lp KH \max_{h \in [H]} \gN_{\rho} (\gQ_h) \gN_{\rho} (\gR_h)  / \delta \rp + 30 k H^3 \rho.
    \end{align*}
\end{lem}

\begin{lem}
\label{lem:empirical_error_of_qk}
    For any fixed $\delta \in (0, 1]$, with probability at least $1-\delta$, for all $k \in [K]$,
    \begin{align*}
        &\quad \BE^{k} (Q^k) 
        \\
        &\geq \frac{1}{2}   \sum_{i=0}^{k-1}  \expect \ls \sum_{h=1}^H \lp Q^k_h (s^i_h, a^i_h) - (\gT^{r^k}_h Q^k_{h+1}) (s^i_h, a^i_h)   \rp^2  \bigg| \pi^i \rs - 41 H^4 \log \lp 2 K H \max_{h \in [H]} \gN_{\rho} (\gQ_h) \gN_{\rho} (\gR_h)/\delta \rp - 27 kH^2 \rho.
    \end{align*}
\end{lem}

Now we proceed to analyze the policy error. First of all, we perform the following error decomposition.
\begin{align*}
    &\quad \frac{1}{K} \sum_{k=1}^K V^{\piE}_{r^k} - V^{\pi^k}_{ r^k} 
    \\
    &\leq \frac{1}{K} \sum_{k=1}^K V^{\star}_{r^k} - V^{\pi^k}_{r^k} 
    \\
    &= \frac{1}{K} \sum_{k=1}^K \lp V^{\star}_{r^k} - Q^{k}_1 (s_1, \pi^k) \rp + \frac{1}{K} \sum_{k=1}^K \lp Q^{k}_1 (s_1, \pi^k) -  V^{\pi^k}_{ r^k} \rp
    \\
    &= \frac{1}{K} \sum_{k=1}^K \lp \max_{a \in \gA} Q^{\star, r^k}_1 (s_1, a) - \max_{a \in \gA} Q^{k}_1 (s_1, a) \rp + \frac{1}{K} \sum_{k=1}^K \lp Q^{k}_1 (s_1, \pi^k) -  V^{\pi^k}_{ r^k} \rp.
\end{align*}
Here $V^{\star}_{r^k}$ denotes the optimal policy value under reward $r^k$. 

From line 4 in \cref{alg:mf_ail_oe}, we know that $Q^k$ is an approximate solution of $\min_{Q \in \gQ} \gL^k (Q)$ with an error $\varepsilon^{Q}_{\opt}$. With $Q^{\star, r^k} \in \gQ$ from \cref{asmp:realizability_and_bellman_completeness_q_class}, we have that
\begin{align*}
    &\quad \BE^{k} (Q^k) - \lambda_{Q} \max_{a \in \gA} Q^k_1 (s_1, a) 
    \\
    &\leq \BE^{k} (Q^{\star, r^k}) - \lambda_{Q} \max_{a \in \gA} Q^{\star, r^k}_1 (s_1, a) + \varepsilon^{Q}_{\opt}.
\end{align*}
Rearrange the above inequality yields that
\begin{align*}
    &\quad \max_{a \in \gA} Q^{\star, r^k}_1 (s_1, a) - \max_{a \in \gA} Q^{k}_1 (s_1, a) 
    \\
    &\leq \frac{1}{\lambda_{Q}} \lp \BE^{k} (Q^{\star, r^k}) - \BE^{k} (Q^k)    \rp + \frac{\varepsilon^{Q}_{\opt}}{\lambda_{Q}}.
\end{align*}
From \cref{lem:empirical_error_of_optimal_q}, with probability at least $1-\delta$, we have 
\begin{align*}
    \BE^{k} (Q^{\star, r^k}) &\leq  16H^4 \log \lp KH \max_{h \in [H]} \gN_{\rho} (\gQ_h) \gN_{\rho} (\gR_h)  / \delta \rp + 30 k H^3 \rho.
\end{align*}
On the other hand, with probability at least $1-\delta$, we have
\begin{align*}
        & \quad \BE^{k} (Q^k)
        \\
        &\geq \frac{1}{2}   \sum_{i=0}^{k-1}  \expect \ls \sum_{h=1}^H \lp Q^k_h (s^i_h, a^i_h) - (\gT^{r^k}_h Q^k_{h+1}) (s^i_h, a^i_h)   \rp^2  \bigg| \pi^i \rs - 41 H^4 \log \lp 2 K H \max_{h \in [H]} \gN_{\rho} (\gQ_h) \gN_{\rho} (\gR_h)/\delta \rp 
        \\
        &\; - 27 kH^2 \rho.
\end{align*}
By union bound, with probability at least $1-\delta$,
\begin{align*}
    &\quad \max_{a \in \gA} Q^{\star, r^k}_1 (s_1, a) - \max_{a \in \gA} Q^{k}_1 (s_1, a) 
    \\
    &\leq \lp -\frac{1}{2 \lambda_{Q}} \rp  \sum_{i=0}^{k-1}  \expect \ls \sum_{h=1}^H \lp Q^k_h (s^i_h, a^i_h) - (\gT^{r^k}_h Q^k_{h+1}) (s^i_h, a^i_h)   \rp^2  \bigg| \pi^i \rs 
    \\
    &\;+ \frac{57 H^4 \log (4 K H \max_{h \in [H]} \gN_{\rho} (\gQ_h) \gN_{\rho} (\gR_h) /\delta) }{\lambda_{Q}}  + \frac{57 kH^3 \rho + \varepsilon^{Q}_{\opt}}{\lambda_{Q}}.
\end{align*}
Then we have that
\begin{align*}
    &\quad \frac{1}{K} \sum_{k=1}^K V^{\piE}_{r^k} - V^{\pi^k}_{ r^k} 
    \\
    &\leq   \lp -\frac{1}{2 \lambda_{Q}} \frac{1}{K} \rp \sum_{k=1}^K   \sum_{i=0}^{k-1}  \expect \ls \sum_{h=1}^H \lp Q^k_h (s^i_h, a^i_h) - (\gT^{r^k}_h Q^k_{h+1}) (s^i_h, a^i_h)   \rp^2  \bigg| \pi^i \rs 
    \\
    &\; + \frac{57 H^4 \log (4 K H \max_{h \in [H]} \gN_{\rho} (\gQ_h) \gN_{\rho} (\gR_h) /\delta)   }{\lambda_{Q}} + \frac{57 KH^3 \rho + \varepsilon^{Q}_{\opt}}{\lambda_{Q}} + \frac{1}{K} \sum_{k=1}^K \lp Q^{k}_1 (s_1, \pi^k) -  V^{\pi^k}_{ r^k} \rp.
\end{align*}
Now we upper bound the last term in RHS of the above inequality. From \cref{asmp:low_gec}, for any $\mu \geq 0$, it holds that 
\begin{align*}
    &\quad \frac{1}{K} \sum_{k=1}^K  Q^{k}_1 (s_1, \pi^k) -  V^{\pi^k}_{ r^k}
    \\
    &\leq  \lp \frac{\mu}{2K} \rp \sum_{k=1}^K \sum_{i=1}^{k-1} \expect \ls \sum_{h=1}^H \lp Q^k_h (s_h, a_h) - \gT^{r^k}_h Q^{k}_{h+1} (s_h, a_h) \rp^2 \bigg| \pi^{i} \rs + \frac{d}{2 \mu K} + \sqrt{\frac{d H}{K} } + \varepsilon H
    \\
    &= \lp \frac{1}{2\lambda_{Q} K} \rp \sum_{k=1}^K \sum_{i=1}^{k-1} \expect \ls \sum_{h=1}^H \lp Q^k_h (s_h, a_h) - \gT^{r^k}_h Q^{k}_{h+1} (s_h, a_h) \rp^2 \bigg| \pi^{i} \rs + \frac{\lambda_{Q} d}{2 K}  + \sqrt{\frac{d H}{K} } + \varepsilon H.
\end{align*}
The last equation is obtained by setting $\mu = 1/\lambda_{Q}$. Combining the above two inequalities yields that
\begin{align*}
    &\quad \frac{1}{K} \sum_{k=1}^K V^{\piE}_{r^k} - V^{\pi^k}_{ r^k} 
    \\
    &\leq \frac{57 H^4 \log (4 K H \max_{h \in [H]} \gN_{\rho} (\gQ_h) \gN_{\rho} (\gR_h) /\delta)  }{\lambda_{Q}} + \frac{57 KH^3 \rho + \varepsilon^{Q}_{\opt}}{\lambda_{Q}}+ \frac{\lambda_{Q} d}{2 K} + \sqrt{\frac{d H}{K} } + \varepsilon H.
\end{align*}

\subsection{Proof of Lemma \ref{lem:policy_error_bound_mb}}
\label{subsec:proof_of_lem_policy_error_bound_mb}
In this part, we present the proof of \cref{lem:policy_error_bound_mb}, which provides the policy error bound for model-based OPT-AIL. Notice that the policy error analysis for model-based OPT-AIL substantially differs from that for model-free OPT-AIL. Before proving \cref{lem:policy_error_bound_mb}, we introduce a useful lemma, which characterizes the concentration property of the transition learning objective.

\begin{lem}
\label{lem:transition_model_concentration}
Consider \cref{alg:mb_ail_oe}. With probability at least $1-\delta$, for all $k \in [K]$,
\begin{align*}
    &\quad \LOG^k (P^\star) - \LOG^k (P^k) 
    \\
    &\leq - \frac{1}{2}  \sum_{i=0}^{k-1} \sum_{h=1}^H \expect_{(s^i_h, a^i_h) \sim d^{\pi^i}_h (\cdot, \cdot)} \bigg[ \SH (P^\star_h (\cdot |s^i_h, a^i_h ), P^k_h (\cdot |s^i_h, a^i_h )  )   \bigg] + 2 H \log \lp H \max_{h \in [H]} \gN_{\rho} (\gP_h; \log) /\delta \rp + 2K H\rho. 
\end{align*}
Here $\LOG^k (P) = - \sum_{i=0}^{k-1} \sum_{h=1}^H \log \lp P_h \lp s^i_{h+1} | s^i_h, a^i_h \rp \rp$.
\end{lem}
The proof of \cref{lem:transition_model_concentration} can be found in Appendix \ref{subsec:lem_transition_model_concentration}.

Now we start to prove \cref{lem:policy_error_bound_mb}. First, we decompose the policy error into two parts.
\begin{align*}
    &\quad \frac{1}{K} \sum_{k=1}^K V^{\piE}_{r^k} - V^{\pi^k}_{ r^k} 
    \\
    &\leq \frac{1}{K} \sum_{k=1}^K V^{\star}_{r^k} - V^{\pi^k}_{r^k} 
    \\
    &= \frac{1}{K} \sum_{k=1}^K \lp V^{\star}_{r^k} - V^{\star}_{P^k, r^k} \rp + \frac{1}{K} \sum_{k=1}^K \lp V^{\star}_{P^k, r^k} -  V^{\pi^k}_{ r^k} \rp
    \\
    &= \frac{1}{K} \sum_{k=1}^K \lp V^{\star}_{r^k} - V^{\star}_{P^k, r^k} \rp + \frac{1}{K} \sum_{k=1}^K \lp V^{\pi^k}_{P^k, r^k} -  V^{\pi^k}_{ r^k} \rp.
\end{align*}
Here $V^{\star}_{P^k, r^k}$ denotes the optimal value under transition function $P^k$ and reward function $r^k$. Similarly, $V^{\pi^k}_{P^k, r^k}$ denotes the policy value of $\pi^k$ under transition function $P^k$ and reward function $r^k$. The last equation follows $\pi^k = \argmax_{\pi} V^{\pi}_{P^k, r^k}$.

We first analyze the first term in the RHS of the above equation. Because $P^k$ is an $\varepsilon_{\opt}^P$-approximate optimal solution of $\min_{P \in \gP} \gL^k (P)$ and $P^\star \in \gP$, we have that
\begin{align*}
    \gL^k (P^k) \leq \gL^k (P^\star) + \varepsilon_{\opt}^P.  
\end{align*}
This implies that
\begin{align*}
    V^{\star}_{P^\star, r^k} - V^{\star}_{P^k, r^k} \leq \frac{1}{\lambda_P} \lp \LOG^k (P^\star) - \LOG^k (P^k) + \varepsilon^{P}_{\opt}  \rp.
\end{align*}
By \cref{lem:transition_model_concentration}, with probability at least $1-\delta$, for all $k \in [K]$,
\begin{align*}
        &\quad \LOG^k (P^\star) - \LOG^k (P^k) 
        \\
        &\leq - \frac{1}{2}  \sum_{i=0}^k \sum_{h=1}^H \expect_{(s^i_h, a^i_h) \sim d^{\pi^i}_h (\cdot, \cdot)} \bigg[ \SH (P^\star_h (\cdot |s^i_h, a^i_h ), P^k_h (\cdot |s^i_h, a^i_h )  )   \bigg] + 2 H \log \lp H \max_{h \in [H]} \gN_{\rho} (\gP_h; \log) \delta^{-1} \rp + 2K H\rho. 
\end{align*}
Then we get that
\begin{align*}
    &\quad V^{\star}_{P^\star, r^k} - V^{\star}_{P^k, r^k} 
    \\
    &\leq - \frac{1}{2 \lambda_P}  \sum_{i=0}^{k-1} \sum_{h=1}^H \expect_{(s^i_h, a^i_h) \sim d^{\pi^i}_h (\cdot, \cdot)} \bigg[ \SH (P^\star_h (\cdot |s^i_h, a^i_h ), P^k_h (\cdot |s^i_h, a^i_h )  )   \bigg] + \frac{2 H \log \lp H \max_{h \in [H]} \gN_{\rho} (\gP_h; \log) \delta^{-1} \rp}{\lambda_P} 
    \\
    &\; + \frac{2K H\rho}{\lambda_P} + \frac{\varepsilon^P_{\opt}}{\lambda_P}.
\end{align*}
We take a summation over $k \in [K]$ and plug the obtained inequality into the policy error.
\begin{align*}
    &\quad \frac{1}{K} \sum_{k=1}^K V^{\piE}_{r^k} - V^{\pi^k}_{ r^k} 
    \\
    &\leq - \frac{1}{2 \lambda_{P} K} \sum_{k=1}^K  \sum_{i=0}^{k-1} \sum_{h=1}^H \expect_{(s^i_h, a^i_h) \sim d^{\pi^i}_h (\cdot, \cdot)} \big[ \SH (P^\star_h (\cdot |s^i_h, a^i_h ), P^k_h (\cdot |s^i_h, a^i_h )  )   \big] + \frac{2 H \log \lp H \max_{h \in [H]} \gN_{\rho} (\gP_h; \log) \delta^{-1} \rp}{\lambda_P} 
    \\
    &\;+ \frac{2K H\rho}{\lambda_P} + \frac{\varepsilon^P_{\opt}}{\lambda_P} + \frac{1}{K} \sum_{k=1}^K \lp V^{\pi^k}_{P^k, r^k} -  V^{\pi^k}_{ r^k} \rp.
\end{align*}
To analyze the last term in RHS, we leverage \cref{asmp:low_gec} for the model-based class. In particular, as discussed in \cref{rem:gec}, we set that $D (f^k, s_h, a_h; r^k) = \SH (P^\star_h (\cdot|s_h, a_h), f^k_h (\cdot|s_h, a_h) )$, $\pi^k = \argmax_{\pi} V^{\pi}_{f^k, r^k}$ and $V (f^k; r^k) = \max_{\pi} V^{\pi}_{f^k, r^k}$. With $\mu = 1/\lambda_{P}$, we can derive that
\begin{align*}
    &\quad \frac{1}{K} \sum_{k=1}^K V^{\piE}_{r^k} - V^{\pi^k}_{ r^k} 
    \\
    &\leq - \frac{1}{2 \lambda_{P} K} \sum_{k=1}^K  \sum_{i=0}^{k-1} \sum_{h=1}^H \expect_{(s^i_h, a^i_h) \sim d^{\pi^i}_h (\cdot, \cdot)} \big[ \SH (P^\star_h (\cdot |s^i_h, a^i_h ), P^k_h (\cdot |s^i_h, a^i_h )  )   \big] + \frac{2 H \log \lp H \max_{h \in [H]} \gN_{\rho} (\gP_h; \log) \delta^{-1} \rp}{\lambda_P} 
    \\
    &\;+ \frac{2K H\rho}{\lambda_P} + \frac{\varepsilon^P_{\opt}}{\lambda_P} + \frac{1}{K} \sum_{k=1}^K \lp V^{\pi^k}_{P^k, r^k} -  V^{\pi^k}_{ r^k} \rp
        \\
        &\leq - \frac{1}{2 \lambda_P K} \sum_{k=1}^K  \sum_{i=0}^{k-1} \sum_{h=1}^H \expect_{(s^i_h, a^i_h) \sim d^{\pi^i}_h (\cdot, \cdot)} \big[ \SH (P^\star_h (\cdot |s^i_h, a^i_h ), P^k_h (\cdot |s^i_h, a^i_h )  )   \big] + \frac{2 H \log \lp H \max_{h \in [H]} \gN_{\rho} (\gP_h; \log) \delta^{-1} \rp}{\lambda_P} 
        \\
        &\; + \frac{2K H\rho}{\lambda_P} + \frac{\varepsilon^P_{\opt}}{\lambda_P} + \frac{1}{2\lambda_P K} \sum_{k=1}^K \sum_{i=1}^{k-1} \expect \ls \sum_{h=1}^H \SH (P^\star_h (\cdot|s_h, a_h), P^k_h (\cdot|s_h, a_h) ) \bigg| \pi^{i} \rs + \frac{\lambda_P d_{\gec} (\varepsilon^\prime)}{2 K}  + \sqrt{\frac{d_{\gec} (\varepsilon^\prime) H}{ K}} + \varepsilon^\prime H
        \\
        &\leq \frac{2 H \log \lp H \max_{h \in [H]} \gN_{\rho} (\gP_h; \log) \delta^{-1} \rp + 2K H\rho + \varepsilon^P_{\opt}}{\lambda_P} + \frac{\lambda_P d_{\gec} (\varepsilon^\prime)}{2 K}  + \sqrt{\frac{d_{\gec} (\varepsilon^\prime) H}{ K}} + \varepsilon^\prime H.
\end{align*}
We finish the proof.

\subsection{Proof of Lemma \ref{lem:empirical_error_of_optimal_q}}
\label{subsec:proof_of_lemma_empirical_error_of_optimal_q}

    Recall the definition of the estimated Bellman error.
    \begin{align*}
       &\quad \BE^k (Q^{\star, r^k})
       \\
       &= \sum_{h=1}^H \bigg( \gE_{h} (Q^{\star, r^k}_h, Q^{\star, r^k}_{h+1}; \gD^{k}, r^{k}) - \inf_{Q^\prime_h \in \gQ_h} \BE_{h} (Q^\prime_h, Q^{\star, r^k}_{h+1}; \gD^{k}, r^{k}) \bigg)
        \\
        &= \sum_{h=1}^H \sum_{i=0}^{k-1} \bigg( Q^{\star, r^k}_h (s^i_h, a^i_h) - r^k_h (s^i_h, a^i_h) - \max_{a^\prime} Q^{\star, r^k}_{h+1} (s^i_{h+1}, a^\prime) \bigg)^2 
        \\
        &\; - \sum_{h=1}^H \inf_{Q^\prime_h \in \gQ_h} \sum_{i=0}^{k-1} \bigg( Q^{\prime}_h (s^i_h, a^i_h) - r^k_h (s^i_h, a^i_h) - \max_{a^\prime} Q^{\star, r^k}_{h+1} (s^i_{h+1}, a^\prime) \bigg)^2.
    \end{align*}
    For any fixed tuple $(k, h, Q^\prime, r) \in [K] \times [H] \times \gQ \times \gR$, we define the random variable
    \begin{align*}
        &\quad Z^i_h (Q^\prime, r) 
        \\
        &:= \lp Q^\prime_h (s^i_h, a^i_h) - r_h (s^i_h, a^i_h) - \max_{a^\prime \in \gA} Q^{\star, r}_{h+1} (s^i_{h+1}, a^\prime)  \rp^2 - \lp Q^{\star, r}_h (s^i_h, a^i_h) - r_h (s^i_h, a^i_h) - \max_{a^\prime \in \gA} Q^{\star, r}_{h+1} (s^i_{h+1}, a^\prime) \rp^2.
    \end{align*}
    Furthermore, we define the filtration $\gF^{i}_h = \sigma ( \{ (s^j_1, a^j_1, \ldots, s^j_H, a^j_H) \}_{j=0}^{i-1} \cup \{s^i_1, a^i_1, \ldots, s^i_h, a^i_h \} )$. Then we calculate the expectation and variance of $Z^i_h (Q^\prime, r)$ conditioned on $\gF^{i}_h$.
    \begin{align*}
        & \quad \expect \ls Z^i_h (Q^\prime, r) | \gF^{i}_h   \rs
        \\
        &= \expect \bigg[   \bigg( Q^\prime_h (s^i_h, a^i_h) - Q^{\star, r}_h (s^i_h, a^i_h) + Q^{\star, r}_h (s^i_h, a^i_h)- r_h (s^i_h, a^i_h) - \max_{a^\prime \in \gA} Q^{\star, r}_{h+1} (s^i_{h+1}, a^\prime)  \bigg)^2 \bigg| \gF^{i}_h   \bigg] 
        \\
        &\; - \expect \bigg[ \bigg( Q^{\star, r}_h (s^i_h, a^i_h) - r_h (s^i_h, a^i_h) - \max_{a^\prime \in \gA} Q^{\star, r}_{h+1} (s^i_{h+1}, a^\prime) \bigg)^2    \bigg| \gF^{i}_h   \bigg]
        \\
        &= \expect \ls  \lp Q^\prime_h (s^i_h, a^i_h) -  Q^{\star, r}_h (s^i_h, a^i_h) \rp^2  \bigg| \gF^{i}_h   \rs 
        \\
        &\; + 2 \expect \bigg[ \lp Q^\prime_h (s^i_h, a^i_h) - Q^{\star, r}_h (s^i_h, a^i_h) \rp \lp Q^{\star, r}_h (s^i_h, a^i_h) - r_h (s^i_h, a^i_h) - \max_{a^\prime \in \gA} Q^{\star, r}_{h+1} (s^i_{h+1}, a^\prime) \rp \bigg| \gF^{i}_h   \bigg]
        \\
        &= \lp Q^\prime_h (s^i_h, a^i_h) -  Q^{\star, r}_h (s^i_h, a^i_h) \rp^2 
        \\
        &\; + 2 \lp Q^\prime_h (s^i_h, a^i_h) - Q^{\star, r}_h (s^i_h, a^i_h) \rp  \expect \ls   Q^{\star, r}_h (s^i_h, a^i_h) - r_h (s^i_h, a^i_h) - \max_{a^\prime \in \gA} Q^{\star, r}_{h+1} (s^i_{h+1}, a^\prime) \bigg| \gF^{i}_h   \rs
        \\
        &= \lp Q^\prime_h (s^i_h, a^i_h) -  Q^{\star, r}_h (s^i_h, a^i_h) \rp^2 + 2 \lp Q^\prime_h (s^i_h, a^i_h) - Q^{\star, r}_h (s^i_h, a^i_h) \rp \lp Q^{\star, r}_h (s^i_h, a^i_h) - (\gT^{r}_{h} Q^{\star, r}_{h+1}) (s^i_h, a^i_h)  \rp
        \\
        &= \lp Q^\prime_h (s^i_h, a^i_h) -  Q^{\star, r}_h (s^i_h, a^i_h) \rp^2.
    \end{align*}
For the conditional variance, we have that
\begin{align*}
    &\quad \Var \ls Z^i_h (Q^\prime, r) \bigg| \gF^{i}_h  \rs
    \\
    &\leq  \expect \ls \lp Z^i_h (Q^\prime, r) \rp^2 \bigg| \gF^{i}_h  \rs
    \\
    &= \expect \bigg[ \lp Q^\prime_h (s^i_h, a^i_h) - Q^{\star, r}_h (s^i_h, a^i_h)   \rp^2 \bigg( Q^\prime_h (s^i_h, a^i_h) + Q^{\star, r}_h (s^i_h, a^i_h) - 2 \lp r_h (s^i_h, a^i_h) + \max_{a^\prime \in \gA} Q^{\star, r}_{h+1} (s^i_{h+1}, a^\prime) \rp  \bigg)^2 \bigg| \gF^{i}_h  \bigg]
    \\
    &\overset{\text{(a)}}{\leq} 16H^2  \lp Q^\prime_h (s^i_h, a^i_h) - Q^{\star, r}_h (s^i_h, a^i_h)   \rp^2
    \\
    &= 16H^2  \expect \ls  Z^i_h (Q^\prime, r) \bigg| \gF^{i}_h  \rs .   
\end{align*}
Here inequality $(a)$ holds since $| Q^\prime_h (s^i_h, a^i_h) + Q^{\star, r}_h (s^i_h, a^i_h) - 2 ( r_h (s^i_h, a^i_h) + \max_{a^\prime \in \gA} Q^{\star, r}_{h+1} (s^i_{h+1}, a^\prime) )  | \leq 4H$ almost surely.

Notice that $\{ Z^i_h (Q^\prime, r) - \expect \ls Z^i_h (Q^\prime, r) | \gF^i_h \rs  \}_{i=0}^{k-1}$ is the martingale difference sequence adapted to $\{ \gF^i_h \}_{i=0}^{k-1}$. Besides, almost surely, we have that
\begin{align*}
    &\quad \labs Z^i_h (Q^\prime, r) \rabs
    \\
    &\leq \max \bigg\{ \lp Q^\prime_h (s^i_h, a^i_h) - r_h (s^i_h, a^i_h) - \max_{a^\prime \in \gA} Q^{\star, r}_{h+1} (s^i_{h+1}, a^\prime)  \rp^2,  \lp Q^{\star, r}_h (s^i_h, a^i_h) - r_h (s^i_h, a^i_h) - \max_{a^\prime \in \gA} Q^{\star, r}_{h+1} (s^i_{h+1}, a^\prime) \rp^2 \bigg\}
        \\
        &\leq 4H^2.
\end{align*}
Then we immediately get that $|Z^i_h (Q^\prime, r) - \expect \ls Z^i_h (Q^\prime, r) | \gF^i_h \rs | \leq 8H^2$ almost surely.
    Thus we can apply Lemma \ref{lem:freedman1} and obtain that for any $\eta \in (0, 1/(4H^2)]$, with probability at least $1-\delta$, 
    \begin{align*}
        &\quad \labs \sum_{i=0}^{k-1} Z^i_h (Q^\prime, r) - \sum_{i=0}^{k-1} \expect \ls  Z^i_h (Q^\prime, r) \bigg| \gF^{i}_h  \rs \rabs  
        \\
        &\leq \eta \sum_{i=0}^{k-1} \Var \ls Z^i_h (Q^\prime, r) \bigg| \gF^{i}_h  \rs + \frac{\log (1/\delta)}{\eta}
        \\
        &\leq 36 H^2 \eta \sum_{i=0}^{k-1} \expect \ls  Z^i_h (Q^\prime, r) \bigg| \gF^{i}_h  \rs + \frac{\log (1/\delta)}{\eta}. 
    \end{align*}
    This implies that
    \begin{align*}
        &\quad - \sum_{i=0}^{k-1} Z^i_h (Q^\prime, r) 
        \\
        &\leq \lp 36 H^2 \eta - 1 \rp \sum_{i=0}^{k-1} \expect \ls  Z^i_h (Q^\prime, r) \bigg| \gF^{i}_h  \rs + \frac{\log (1/\delta)}{\eta} 
        \\ 
        &\leq 16H^2 \log (1/\delta).
    \end{align*}
    The last equation is obtained by choosing $\eta = 1/(16H^2)$.

    We define $(\gQ_h)_{\rho}$ and $(\gR_h)_{\rho}$ as the $\rho$-cover of $\gQ_h$ and $\gR_h$, respectively. It is direct to have that $\gQ_{\rho} = (\gQ_1)_{\rho} \times \ldots (\gQ_H)_{\rho}$ and $\gR_{\rho} = (\gR_1)_{\rho} \times \ldots (\gR_H)_{\rho}$ are $\rho$-covers of $\gQ$ and $\gR$, respectively. By union bound, with probability at least $1-\delta$, for all $(k, h, \widehat{Q}, \widehat{r}) \in [K] \times [H] \times \gQ_{\rho} \times \gR_{\rho} $, we have that
    \begin{align*}
        - \sum_{i=0}^{k-1} Z^i_h (\widehat{Q}, \widehat{r}) &\leq 16H^2 \log \lp KH \prod_{h=1}^H ( |(\gQ_h)_{\rho}| |(\gR_h)_{\rho}|) / \delta \rp
        \\
        &\leq 16H^3 \log \lp KH \max_{h \in [H]}  |(\gQ_h)_{\rho}| |(\gR_h)_{\rho}| / \delta \rp  . 
    \end{align*}
    Furthermore, for any $(Q, r) \in \gQ \times \gR$, there exists $(\widehat{Q}, \widehat{r}) \in \gQ_{\rho} \times \gR_{\rho}$ such that $\| Q - \widehat{Q} \|_{\infty} \leq \rho$ and $\| r - \widehat{r} \|_{\infty} \leq \rho$. Then we have that
    \begin{align*}
        \labs \sum_{i=0}^{k-1} Z^i_h (Q, r) - \sum_{i=0}^{k-1} Z^i_h (\widehat{Q}, \widehat{r})   \rabs \leq \sum_{i=0}^{k-1} \labs Z^i_h (Q, r) - Z^i_h (\widehat{Q}, \widehat{r})   \rabs.
    \end{align*}
    For each term, we have that
\begin{align*}
        &\quad \labs Z^i_h (Q, r)  - Z^i_h (\widehat{Q}, \widehat{r})   \rabs
        \\
        &\leq \bigg| \lp Q_h (s^i_h, a^i_h) - r_h (s^i_h, a^i_h) - \max_{a^\prime \in \gA} Q^{\star, r}_{h+1} (s^i_{h+1}, a^\prime)  \rp^2 - \lp \widehat{Q}_h (s^i_h, a^i_h) - \widehat{r}_h (s^i_h, a^i_h) - \max_{a^\prime \in \gA} Q^{\star, \widehat{r}}_{h+1} (s^i_{h+1}, a^\prime)  \rp^2 \bigg| 
        \\
        &\; + \bigg| \lp Q^{\star, r}_h (s^i_h, a^i_h) - r_h (s^i_h, a^i_h) - \max_{a^\prime \in \gA} Q^{\star, r}_{h+1} (s^i_{h+1}, a^\prime) \rp^2 - \lp Q^{\star, \widehat{r}}_h (s^i_h, a^i_h) - \widehat{r}_h (s^i_h, a^i_h) - \max_{a^\prime \in \gA} Q^{\star, \widehat{r}}_{h+1} (s^i_{h+1}, a^\prime) \rp^2 \bigg|.
\end{align*}
For the first term in RHS, we have that
\begin{align*}
        &\quad \bigg| \lp Q_h (s^i_h, a^i_h) - r_h (s^i_h, a^i_h) - \max_{a^\prime \in \gA} Q^{\star, r}_{h+1} (s^i_{h+1}, a^\prime)  \rp^2 - \lp \widehat{Q}_h (s^i_h, a^i_h) - \widehat{r}_h (s^i_h, a^i_h) - \max_{a^\prime \in \gA} Q^{\star, \widehat{r}}_{h+1} (s^i_{h+1}, a^\prime)  \rp^2 \bigg|
        \\
        &\leq \bigg| Q_h (s^i_h, a^i_h) - r_h (s^i_h, a^i_h) - \max_{a^\prime \in \gA} Q^{\star, r}_{h+1} (s^i_{h+1}, a^\prime) + \widehat{Q}_h (s^i_h, a^i_h) - \widehat{r}_h (s^i_h, a^i_h) - \max_{a^\prime \in \gA} Q^{\star, \widehat{r}}_{h+1} (s^i_{h+1}, a^\prime)   \bigg| \cdot
        \\
        &\; \bigg| Q_h (s^i_h, a^i_h) - \widehat{Q}_h (s^i_h, a^i_h)  - r_h (s^i_h, a^i_h) + \widehat{r}_h (s^i_h, a^i_h) - \max_{a^\prime \in \gA} Q^{\star, r}_{h+1} (s^i_{h+1}, a^\prime) + \max_{a^\prime \in \gA} Q^{\star, \widehat{r}}_{h+1} (s^i_{h+1}, a^\prime)  \bigg|
        \\
        &\leq 4H \bigg( \labs Q_h (s^i_h, a^i_h) - \widehat{Q}_h (s^i_h, a^i_h) \rabs + \labs r_h (s^i_h, a^i_h) - \widehat{r}_h (s^i_h, a^i_h) \rabs + \max_{a^\prime \in \gA} \labs  Q^{\star, r}_{h+1} (s^i_{h+1}, a^\prime) - Q^{\star, \widehat{r}}_{h+1} (s^i_{h+1}, a^\prime)  \rabs \bigg)
        \\
        &\leq 12 H^2 \rho.
\end{align*}
    The last inequality follows Lemma \ref{lem:perturbabtion_analysis}. Similarly, for the second term in RHS, we have that
    \begin{align*}
        &\quad \bigg| \lp Q^{\star, r}_h (s^i_h, a^i_h) - r_h (s^i_h, a^i_h) - \max_{a^\prime \in \gA} Q^{\star, r}_{h+1} (s^i_{h+1}, a^\prime) \rp^2 - \lp Q^{\star, \widehat{r}}_h (s^i_h, a^i_h) - \widehat{r}_h (s^i_h, a^i_h) - \max_{a^\prime \in \gA} Q^{\star, \widehat{r}}_{h+1} (s^i_{h+1}, a^\prime) \rp^2 \bigg|
        \\
        &\leq \bigg| Q^{\star, r}_h (s^i_h, a^i_h) - r_h (s^i_h, a^i_h) - \max_{a^\prime \in \gA} Q^{\star, r}_{h+1} (s^i_{h+1}, a^\prime) + Q^{\star, \widehat{r}}_h (s^i_h, a^i_h) - \widehat{r}_h (s^i_h, a^i_h) - \max_{a^\prime \in \gA} Q^{\star, \widehat{r}}_{h+1} (s^i_{h+1}, a^\prime)  \bigg| 
        \\
        &\; \cdot \bigg| Q^{\star, r}_h (s^i_h, a^i_h) - Q^{\star, \widehat{r}}_h (s^i_h, a^i_h) - r_h (s^i_h, a^i_h) + \widehat{r}_h (s^i_h, a^i_h) - \max_{a^\prime \in \gA} Q^{\star, r}_{h+1} (s^i_{h+1}, a^\prime) + \max_{a^\prime \in \gA} Q^{\star, \widehat{r}}_{h+1} (s^i_{h+1}, a^\prime)  \bigg|
        \\
        &\leq 6H \bigg( \labs Q^{\star, r}_h (s^i_h, a^i_h) - Q^{\star, \widehat{r}}_h (s^i_h, a^i_h) \rabs + \labs r_h (s^i_h, a^i_h) - \widehat{r}_h (s^i_h, a^i_h)   \rabs + \max_{a^\prime \in \gA} \labs Q^{\star, r}_{h+1} (s^i_{h+1}, a^\prime) - Q^{\star, \widehat{r}}_{h+1} (s^i_{h+1}, a^\prime)    \rabs \bigg)
        \\
        &\leq 18H^2 \rho.
    \end{align*}
    Combining the above four inequalities yields that
    \begin{align*}
        \labs \sum_{i=0}^{k-1} Z^i_h (Q, r) - \sum_{i=0}^{k-1} Z^i_h (\widehat{Q}, \widehat{r})   \rabs &\leq \sum_{i=0}^{k-1} \labs Z^i_h (Q, r) - Z^i_h (\widehat{Q}, \widehat{r})   \rabs
        \\
        &\leq 30 k H^2 \rho.
    \end{align*}
    Therefore, for all $(Q, r) \in \gQ \times \gR$, 
    \begin{align*}
        &\quad - \sum_{i=0}^{k-1} Z^i_h (Q, r) 
        \\
        &\leq  - \sum_{i=0}^{k-1} Z^i_h (\widehat{Q}, \widehat{r}) + \labs \sum_{i=0}^{k-1} Z^i_h (Q, r) - \sum_{i=0}^{k-1} Z^i_h (\widehat{Q}, \widehat{r}) \rabs
        \\
        & \leq   16H^3 \log ( KH \max_{h \in [H]}  |(\gQ_h)_{\rho}| |(\gR_h)_{\rho}|  / \delta) + 30 k H^2 \rho
        \\
        &\leq 16H^3 \log ( KH \max_{h \in [H]} \gN_{\rho} (\gQ_h) \gN_{\rho} (\gR_h)  / \delta) + 30 k H^2 \rho. 
    \end{align*}
    This implies that
    \begin{align*}
        &\quad \lp Q^{\star, r}_h (s^i_h, a^i_h) - r_h (s^i_h, a^i_h) - \max_{a^\prime \in \gA} Q^{\star, r}_{h+1} (s^i_{h+1}, a^\prime) \rp^2 \leq
        \\
        &\inf_{Q_h \in \gQ_h} \lp Q_h (s^i_h, a^i_h) - r_h (s^i_h, a^i_h) - \max_{a^\prime \in \gA} Q^{\star, r}_{h+1} (s^i_{h+1}, a^\prime)  \rp^2 + 16H^3 \log ( KH \max_{h \in [H]} \gN_{\rho} (\gQ_h) \gN_{\rho} (\gR_h)  / \delta) + 30 k H^2 \rho.
    \end{align*}
    Therefore, we can derive the upper bound on $\BE^k (Q^{\star, r^k})$. 
    \begin{align*}
         &\quad \BE^k (Q^{\star, r^k})
        \\
        &= \sum_{h=1}^H  \sum_{i=0}^{k-1} \bigg( Q^{\star, r^k}_h (s^i_h, a^i_h) - r^k_h (s^i_h, a^i_h) - \max_{a^\prime} Q^{\star, r^k}_{h+1} (s^i_{h+1}, a^\prime) \bigg)^2 
        \\
        &\; - \sum_{h=1}^H \inf_{Q^\prime_h \in \gQ_h} \sum_{i=0}^{k-1} \bigg( Q^{\prime}_h (s^i_h, a^i_h) - r^k_h (s^i_h, a^i_h) - \max_{a^\prime} Q^{\star, r^k}_{h+1} (s^i_{h+1}, a^\prime) \bigg)^2
        \\
        &\leq 16H^4 \log ( KH \max_{h \in [H]} \gN_{\rho} (\gQ_h) \gN_{\rho} (\gR_h)  / \delta) + 30 k H^3 \rho.
    \end{align*}
    We complete the proof.

\subsection{Proof of Lemma \ref{lem:empirical_error_of_qk}}
\label{subsec:proof_of_lemma_empirical_error_of_qk}

    For any fixed tuple $(k, h, Q, r) \in [K] \times [H] \times \gQ \times \gR$, we define the random variable.
    \begin{align*}
         X^i_h (Q, r) &:= \lp Q_h (s^i_h, a^i_h) - r_h (s^i_h, a^i_h) - \max_{a^\prime} Q_{h+1} (s^i_{h+1}, a^\prime)  \rp^2 
         \\
         &\; - \lp (\gT^{r}_h Q_{h+1}) (s^i_h, a^i_h) - r_h (s^i_h, a^i_h) - \max_{a^\prime} Q_{h+1} (s^i_{h+1}, a^\prime)  \rp^2.
    \end{align*}
    We define the filtration $\gF^{i} = \sigma ( \{ (s^j_1, a^j_1, \ldots, s^j_H, a^j_H) \}_{j=0}^{i-1} )$. In the following part, we calculate the expectation and variance of $X^i_h (Q, r)$ conditioned on $\gF^i$.
    \begin{align*}
    &\quad \expect \bigg[ X^i_h (Q, r) \bigg| \gF^{i} \bigg]
    \\
    &= \expect \bigg[ \bigg( Q_h (s^i_h, a^i_h) - r_h (s^i_h, a^i_h) - \max_{a^\prime} Q_{h+1} (s^i_{h+1}, a^\prime) \bigg)^2 \bigg| \gF^i \bigg] 
    \\
    &\; - \expect \bigg[ \bigg( (\gT^{r}_h Q_{h+1}) (s^i_h, a^i_h) - r_h (s^i_h, a^i_h) - \max_{a^\prime} Q_{h+1} (s^i_{h+1}, a^\prime) \bigg)^2 \bigg| \gF^i \bigg]
    \\
    &= \expect \bigg[ \bigg( Q_h (s^i_h, a^i_h) - (\gT^{r}_h Q_{h+1}) (s^i_h, a^i_h) + (\gT^{r}_h Q_{h+1}) (s^i_h, a^i_h) - r_h (s^i_h, a^i_h) - \max_{a^\prime} Q_{h+1} (s^i_{h+1}, a^\prime) \bigg)^2 \bigg| \gF^i \bigg] 
    \\
    &\; - \expect \bigg[ \bigg( (\gT^{r}_h Q_{h+1}) (s^i_h, a^i_h) - r_h (s^i_h, a^i_h) - \max_{a^\prime} Q_{h+1} (s^i_{h+1}, a^\prime) \bigg)^2 \bigg| \gF^i \bigg]
    \\
    &= \expect \bigg[ \bigg( Q_h (s^i_h, a^i_h) - (\gT^{r}_h Q_{h+1}) (s^i_h, a^i_h) \bigg)^2 \bigg| \gF^i \bigg] 
    \\
    &\; + 2 \expect \bigg[ \bigg( Q_h (s^i_h, a^i_h) - (\gT^{r}_h Q_{h+1}) (s^i_h, a^i_h) \bigg) \cdot \bigg( (\gT^{r}_h Q_{h+1}) (s^i_h, a^i_h) - r_h (s^i_h, a^i_h) - \max_{a^\prime} Q_{h+1} (s^i_{h+1}, a^\prime) \bigg) \bigg| \gF^i \bigg]
    \\
    &= \expect \bigg[ \bigg( Q_h (s^i_h, a^i_h) - (\gT^{r}_h Q_{h+1}) (s^i_h, a^i_h) \bigg)^2 \bigg| \gF^i \bigg] 
    \\
    &\; + 2 \expect \bigg[ \bigg( Q_h (s^i_h, a^i_h) - (\gT^{r}_h Q_{h+1}) (s^i_h, a^i_h) \bigg) \cdot \expect \bigg[ \bigg( (\gT^{r}_h Q_{h+1}) (s^i_h, a^i_h) - r_h (s^i_h, a^i_h) - \max_{a^\prime} Q_{h+1} (s^i_{h+1}, a^\prime) \bigg) \bigg| s^i_h, a^i_h \bigg] \bigg| \gF^i \bigg]
    \\
    &= \expect \bigg[ \bigg( Q_h (s^i_h, a^i_h) - (\gT^{r}_h Q_{h+1}) (s^i_h, a^i_h) \bigg)^2 \bigg| \pi^i \bigg].
\end{align*}
    \begin{align*}
        &\quad \Var \ls X^i_h (Q, r) | \gF^{i}   \rs
        \\
        &\leq \expect \ls  \lp X^i_h (Q, r) \rp^2 | \gF^{i}   \rs
        \\
        &= \expect \bigg[ \bigg( Q_h (s^i_h, a^i_h) + (\gT^{r}_h Q_{h+1}) (s^i_h, a^i_h) - 2r_h (s^i_h, a^i_h) - 2\max_{a^\prime} Q_{h+1} (s^i_{h+1}, a^\prime) \bigg)^2 \cdot \lp  Q_h (s^i_h, a^i_h) - (\gT^{r}_h Q_{h+1}) (s^i_h, a^i_h)  \rp^2 \bigg| \gF^{i}   \bigg]
        \\
        &\leq 16H^2 \expect \ls \lp Q_h (s^i_h, a^i_h) - (\gT^{r}_h Q_{h+1}) (s^i_h, a^i_h)   \rp^2  \bigg| \pi^i \rs
        \\
        &= 16H^2 \expect \ls X^i_h (Q, r) | \gF^{i}   \rs.  
    \end{align*}
    Furthermore, $\{ X^i_h (Q, r) - \expect [ X^i_h (Q, r) | \gF^{i}   ]   \}_{i=0}^{k-1}$ is a martingale difference sequence adapted to $\{ \gF^i \}_{i=0}^{k-1}$. Besides, it is easy to obtain that $| X^i_h (Q, r) | \leq 9H^2$ almost surely. Thus, we can apply Lemma \ref{lem:freedman1} and obtain that with probability at least $1-\delta$, for any $\eta \in (0, 1/(9H^2)]$,
    \begin{align*}
        &\quad \labs \sum_{i=0}^{k-1} X^i_h (Q, r) - \sum_{i=0}^{k-1} \expect [ X^i_h (Q, r) | \gF^{i}   ] \rabs
        \\
        &\leq \eta \sum_{i=0}^{k-1} \Var \ls X^i_h (Q, r) | \gF^{i}   \rs + \frac{\log (2/\delta)}{\eta}
        \\
        &\leq 16H^2 \eta \sum_{i=0}^{k-1} \expect \ls X^i_h (Q, r) | \gF^{i}   \rs + \frac{\log (2/\delta)}{\eta}.
    \end{align*}
    By choosing $\eta = \min \{ 1/(9H^2), \\ \sqrt{\log (2/\delta)  / (16H^2 \sum_{i=0}^{k-1} \expect \ls X^i_h (Q, r) | \gF^{i}   \rs )} \}$, we have that
    \begin{align*}
       &\quad  \labs \sum_{i=0}^{k-1} X^i_h (Q, r) - \sum_{i=0}^{k-1} \expect [ X^i_h (Q, r) | \gF^{i}   ] \rabs
       \\
       &\leq  8 H \sqrt{ \sum_{i=0}^{k-1} \expect \ls X^i_h (Q, r) | \gF^{i}   \rs \log (2/\delta)} + 9H^2 \log (2/\delta).
    \end{align*}
    This implies that
    \begin{align*}
        &\quad \sum_{i=0}^{k-1} \expect [ X^i_h (Q, r) | \gF^{i}   ] - 8 H \sqrt{ \sum_{i=0}^{k-1} \expect \ls X^i_h (Q, r) | \gF^{i}   \rs \log (2/\delta)}
        \\
        &\leq \sum_{i=0}^{k-1} X^i_h (Q, r) + 9H^2 \log (2/\delta).    
    \end{align*}
    This establishes a quadratic formula of $x^2 - bx -c \leq 0$ with $x = \sqrt{\sum_{i=0}^{k-1} \expect [ X^i_h (Q, r) | \gF^{i}   ] }$, $b = 8 H \sqrt{\log (2/\delta)}$ and $c = \sum_{i=0}^{k-1} X^i_h (Q, r) + 9H^2 \log (2/\delta)$. Solving this quadratic formula yields that $ (b-\sqrt{b^2+4c})/2 \leq x \leq (b+\sqrt{b^2+4c})/2$, which implies that
    \begin{align*}
        x^2 \leq \frac{(b + \sqrt{b^2 + 4c})^2}{4} \leq \frac{2 \lp b^2 + b^2 + 4c   \rp}{4} = b^2 + 2c.
    \end{align*}
    Thus we obtain that
    \begin{align*}
        \sum_{i=0}^{k-1} \expect [ X^i_h (Q, r) | \gF^{i}   ]  \leq 2 \sum_{i=0}^{k-1} X^i_h (Q, r) + 82 H^2 \log (2/\delta). 
    \end{align*}
    We define $(\gQ_h)_{\rho}$ and $(\gR_h)_{\rho}$ as the $\rho$-covers of $\gQ_h$ and $\gR_h$, respectively. It is direct to have that $\gQ_{\rho} = (\gQ_1)_{\rho} \times \ldots (\gQ_H)_{\rho}$ and $\gR_{\rho} = (\gR_1)_{\rho} \times \ldots (\gR_H)_{\rho}$ are $\rho$-covers of $\gQ$ and $\gR$, respectively. By union bound, with probability at least $1-\delta$, for all $(k, h, \widehat{Q}, \widehat{r}) \in [K] \times [H] \times \gQ_{\rho} \times \gR_{\rho}$,
    \begin{align*}
        &\quad \sum_{i=0}^{k-1} \expect [ X^i_h (\widehat{Q}, \widehat{r}) | \gF^{i}   ]  
        \\
        &\leq 2 \sum_{i=0}^{k-1} X^i_h (\widehat{Q}, \widehat{r}) + 82 H^2 \log (2 K H |\gQ_{\rho}| |\gR_{\rho}|/\delta)
        \\
        &= 2 \sum_{i=0}^{k-1} X^i_h (\widehat{Q}, \widehat{r}) + 82 H^2 \log \lp 2 K H \prod_{h=1}^H \lp |(\gQ_h)_{\rho}| |(\gR_h)_{\rho}| \rp /\delta \rp
        \\
        &\leq 2 \sum_{i=0}^{k-1} X^i_h (\widehat{Q}, \widehat{r}) + 82 H^3 \log \lp 2 K H \max_{h \in [H]} |(\gQ_h)_{\rho}| |(\gR_h)_{\rho}| /\delta \rp.
    \end{align*}
    We have calculated the conditional expectation in the LHS and obtain that
    \begin{align*}
        &\quad \sum_{i=0}^{k-1} \expect \ls \lp \widehat{Q}_h (s^i_h, a^i_h) - (\gT^{\widehat{r}}_h \widehat{Q}_{h+1})(s^i_h, a^i_h)   \rp^2  \bigg| \pi^i \rs
        \\
        &\leq 2 \sum_{i=0}^{k-1} X^i_h (\widehat{Q}, \widehat{r}) + 82 H^3 \log \lp 2 K H \max_{h \in [H]} |(\gQ_h)_{\rho}| |(\gR_h)_{\rho}| /\delta \rp
        \\
        &\leq 2 \sum_{i=0}^{k-1} X^i_h (\widehat{Q}, \widehat{r}) + 82 H^3 \log \lp 2 K H \max_{h \in [H]} \gN_{\rho} (\gQ_h) \gN_{\rho} (\gR_h)/\delta \rp.
    \end{align*}
    
    According to the definition of $\rho$-cover, for $(Q^k, r^k)$, there exists $(\widehat{Q}, \widehat{r}) \in \gQ_{\rho} \times \gR_{\rho}$ such that 
    \begin{align*}
        &\max_{(s, a, h) \in \gS \times \gA \times [H]} \labs \widehat{Q}_h (s, a) - Q^k_h (s, a)  \rabs \leq \rho , 
        \\
        &\max_{(s, a, h) \in \gS \times \gA \times [H]} \labs \widehat{r}_h (s, a) - r^k_h (s, a)  \rabs \leq \rho. 
    \end{align*}
    Then we can upper bound the errors caused by approximating $(Q^k, r^k)$ with $(\widehat{Q}, \widehat{r})$.
    \begin{align*}
    &\quad \bigg| \bigg( \widehat{Q}_h (s^i_h, a^i_h) - (\gT^{\widehat{r}}_h \widehat{Q}_{h+1})(s^i_h, a^i_h)   \bigg)^2 - \bigg( Q^k_h (s^i_h, a^i_h) - (\gT^{r^k}_h Q^k_{h+1})(s^i_h, a^i_h)   \bigg)^2 \bigg|
    \\
    &\leq \bigg| \widehat{Q}_h (s^i_h, a^i_h) - (\gT^{\widehat{r}}_h \widehat{Q}_{h+1})(s^i_h, a^i_h) + Q^k_h (s^i_h, a^i_h) - (\gT^{r^k}_h Q^k_{h+1})(s^i_h, a^i_h)  \bigg| 
    \\
    &\; \cdot \bigg| \widehat{Q}_h (s^i_h, a^i_h) - (\gT^{\widehat{r}}_h \widehat{Q}_{h+1})(s^i_h, a^i_h) - Q^k_h (s^i_h, a^i_h) + (\gT^{r^k}_h Q^k_{h+1})(s^i_h, a^i_h) \bigg|
    \\
    &\leq 2H \bigg| \widehat{Q}_h (s^i_h, a^i_h) - (\gT^{\widehat{r}}_h \widehat{Q}_{h+1})(s^i_h, a^i_h) - Q^k_h (s^i_h, a^i_h) + (\gT^{r^k}_h Q^k_{h+1})(s^i_h, a^i_h) \bigg|
    \\
    &\leq 6H \rho.
\end{align*}
    \begin{align*}
        &\quad \labs X^i_h (\widehat{Q}, \widehat{r}) - X^i_h (Q^k, r^k)  \rabs
        \\
        &\leq \bigg| \widehat{Q}_h (s^i_h, a^i_h) + Q^k_h (s^i_h, a^i_h) - \widehat{r}_h (s^i_h, a^i_h) - r^k_h (s^i_h, a^i_h) - \max_{a^\prime} \widehat{Q}_{h+1} (s^i_{h+1}, a^\prime) - \max_{a^\prime} Q^k_{h+1} (s^i_{h+1}, a^\prime)  \bigg| 
        \\
        &\; \cdot \bigg| \widehat{Q}_h (s^i_h, a^i_h) - Q^k_h (s^i_h, a^i_h) - \widehat{r}_h (s^i_h, a^i_h) + r^k_h (s^i_h, a^i_h) - \max_{a^\prime} \widehat{Q}_{h+1} (s^i_{h+1}, a^\prime) + \max_{a^\prime} Q^k_{h+1} (s^i_{h+1}, a^\prime)  \bigg| 
        \\
        &\; + \bigg| (\gT^{\widehat{r}}_h \widehat{Q}_{h+1})(s^i_h, a^i_h) + (\gT^{r^k}_h Q^k_{h+1})(s^i_h, a^i_h) - \widehat{r}_h (s^i_h, a^i_h) - r^k_h (s^i_h, a^i_h) - \max_{a^\prime} \widehat{Q}_{h+1} (s^i_{h+1}, a^\prime)   - \max_{a^\prime} Q^k_{h+1} (s^i_{h+1}, a^\prime)  \bigg| 
        \\
        &\; \cdot \bigg| (\gT^{\widehat{r}}_h \widehat{Q}_{h+1})(s^i_h, a^i_h) - (\gT^{r^k}_h Q^k_{h+1})(s^i_h, a^i_h) - \widehat{r}_h (s^i_h, a^i_h) + r^k_h (s^i_h, a^i_h) - \max_{a^\prime} \widehat{Q}_{h+1} (s^i_{h+1}, a^\prime) + \max_{a^\prime} Q^k_{h+1} (s^i_{h+1}, a^\prime)  \bigg|
        \\
        &\leq 4H \bigg| \widehat{Q}_h (s^i_h, a^i_h) - Q^k_h (s^i_h, a^i_h) - \widehat{r}_h (s^i_h, a^i_h) + r^k_h (s^i_h, a^i_h) - \max_{a^\prime} \widehat{Q}_{h+1} (s^i_{h+1}, a^\prime) + \max_{a^\prime} Q^k_{h+1} (s^i_{h+1}, a^\prime)  \bigg| 
        \\
        &\; + 4H \bigg| (\gT^{\widehat{r}}_h \widehat{Q}_{h+1})(s^i_h, a^i_h) - (\gT^{r^k}_h Q^k_{h+1})(s^i_h, a^i_h) - \widehat{r}_h (s^i_h, a^i_h) + r^k_h (s^i_h, a^i_h) - \max_{a^\prime} \widehat{Q}_{h+1} (s^i_{h+1}, a^\prime) + \max_{a^\prime} Q^k_{h+1} (s^i_{h+1}, a^\prime)  \bigg|
        \\
        &\leq 24 H \rho.
    \end{align*}
    With the above bounds, we can obtain that
    \begin{align*}
         &\quad \sum_{i=0}^{k-1} \expect \ls \lp Q^k_h (s^i_h, a^i_h) - (\gT^{r^k}_h Q^k_{h+1}) (s^i_h, a^i_h)   \rp^2  \bigg| \pi^i \rs   
         \\
         & \leq 2 \sum_{i=0}^{k-1} X^i_h (Q^k, r^k) + 82 H^3 \log \lp 2 K H \max_{h \in [H]} \gN_{\rho} (\gQ_h) \gN_{\rho} (\gR_h)/\delta \rp + 54 kH \rho.
    \end{align*}
    According to the definition of $\BE^k$, we have that
    \begin{align*}
        &\quad \BE^k (Q^k) 
        \\
        &= \sum_{h=1}^H \sum_{i=0}^{k-1} \big( Q^k_h (s^i_h, a^i_h) - r^k_h (s^i_h, a^i_h) - \max_{a^\prime} Q^k_{h+1} (s^i_{h+1}, a^\prime)  \bigg)^2 - \inf_{Q^\prime \in \gQ} \sum_{h=1}^H \sum_{i=0}^{k-1} \bigg( Q^\prime_h (s^i_h, a^i_h)  - r^k_h (s^i_h, a^i_h) - \max_{a^\prime} Q^k_{h+1} (s^i_{h+1}, a^\prime)  \bigg)^2
        \\
        &\overset{\text{(a)}}{\geq} \sum_{h=1}^H \sum_{i=0}^{k-1}  X^i_h (Q^k, r^k)
        \\
        &\geq \frac{1}{2} \sum_{h=1}^H  \sum_{i=0}^{k-1}  \expect \ls \lp Q^k_h (s^i_h, a^i_h) - (\gT^{r^k}_h Q^k_{h+1})(s^i_h, a^i_h)   \rp^2  \bigg| \pi^i \rs - 41 H^4 \log \lp 2 K H \max_{h \in [H]} \gN_{\rho} (\gQ_h) \gN_{\rho} (\gR_h)/\delta \rp - 27 kH^2 \rho.
    \end{align*}
Inequality $\text{(a)}$ follows \cref{asmp:realizability_and_bellman_completeness_q_class} that $\gT^{r^k}_h Q^{r^k}_{h+1} \in \gQ_{h} $. We complete the proof.

\subsection{Proof of Lemma \ref{lem:transition_model_concentration}}
\label{subsec:lem_transition_model_concentration}
First, for any fixed $P$, we have that
\begin{align*}
        \LOG^k (P^\star) - \LOG^k (P) = \sum_{i=0}^{k-1} \sum_{h=1}^H \log \lp \frac{P_h (s^i_{h+1} |s^i_h, a^i_h)}{P^\star_h (s^i_{h+1} |s^i_h, a^i_h)} \rp.
\end{align*}
To upper bound $\LOG^k (P^\star) - \LOG^k (P)$, we will conduct a concentration analysis for the RHS in the above equation. Specifically, for any fixed $0 \leq i \leq K-1$ and $h \in [H]$ and $P_h \in \gP_h$, we define the random variable
\begin{align*}
    Y^i_h (P_h) = \frac{1}{2} \log \lp \frac{P_h (s^i_{h+1} |s^i_h, a^i_h )}{P^\star_h (s^i_{h+1} |s^i_h, a^i_h )} \rp.
\end{align*}
Then for $P = (P_1, \ldots, P_H) \in \gP$, we have that
\begin{align*}
    \LOG^k (P^\star) - \LOG^k (P) = 2 \sum_{i=0}^{k-1} \sum_{h=1}^H Y^i_h (P_h). 
\end{align*}
Notice that $\{ Y^i_h \}_{ 0 \leq i \leq K-1, h \in [H]}$ are statistically dependent. Therefore, we apply a martingale concentration inequality to analyze it. In particular, we define $ \gF^i = \sigma (\{s^0_1, a^0_1, \ldots, s^0_H, a^0_H, \ldots, s^i_H, a^i_H  \} ), \forall 0 \leq i \leq K-1$ as the sigma-field generated by the random variables of the first $i$ trajectories.

It is easy to check that $\{ Y^i_h \}_{0 \leq i \leq K-1}$ is adapted to the filtration $\{ \gF^i \}_{0 \leq i \leq K-1} $. With \cref{lem:martingale_chernoff}, with probability at least $1-\delta$, $\forall 1 \leq k \leq K$,
\begin{align*}
    \sum_{i=0}^{k-1} Y^i_h (P_h)  \leq \sum_{i=0}^{k-1} \log \lp \expect \ls \exp (Y^i_h (P_h)) | \gF^{i-1} \rs \rp + \log (1/\delta).    
\end{align*}
For $\log \lp \expect \ls \exp (Y^i_h (P_h)) | \gF^{i-1} \rs \rp$ in the RHS, we have that
\begin{align*}
    & \quad \log \lp \expect \ls \exp (Y^i_h (P_h)) | \gF^{i-1} \rs \rp 
    \\
    &\leq \expect \ls \exp (Y^i_h (P_h)) | \gF^{i-1} \rs - 1
    \\
    &= \expect \ls \exp \lp \frac{1}{2} \log \lp \frac{P_h (s^i_{h+1} |s^i_h, a^i_h )}{P^\star_h (s^i_{h+1} |s^i_h, a^i_h )} \rp \rp \bigg| \gF^{i-1} \rs -1
    \\
    &= \expect \ls  \lp \frac{P_h (s^i_{h+1} |s^i_h, a^i_h )}{P^\star_h (s^i_{h+1} |s^i_h, a^i_h )} \rp^{1/2} \bigg| \gF^{i-1} \rs -1.  
\end{align*}
Notice that $\pi^i$ is measurable with respect to $\gF^{i-1}$. Then we can derive that
\begin{align*}
    & \quad \expect \ls  \lp \frac{P_h (s^i_{h+1} |s^i_h, a^i_h )}{P^\star_h (s^i_{h+1} |s^i_h, a^i_h )} \rp^{1/2} \bigg| \gF^{i-1} \rs
    \\
    &= \expect \ls \expect \ls   \lp \frac{P_h (s^i_{h+1} |s^i_h, a^i_h )}{P^\star_h (s^i_{h+1} |s^i_h, a^i_h )} \rp^{1/2} \bigg| \pi^i \rs \bigg| \gF^{i-1} \rs
    \\
    &= \expect \ls    \lp \frac{P_h (s^i_{h+1} |s^i_h, a^i_h )}{P^\star_h (s^i_{h+1} |s^i_h, a^i_h )} \rp^{1/2} \bigg| \pi^i \rs
    \\
    &= \expect_{(s^i_h, a^i_h) \sim d^{\pi^i}_h (\cdot, \cdot), s^{i+1}_h \sim P^\star_h (\cdot|s^i_h, a^i_h)} \ls  \lp \frac{P_h (s^i_{h+1} |s^i_h, a^i_h )}{P^\star_h (s^i_{h+1} |s^i_h, a^i_h )} \rp^{1/2} \rs
    \\
    &= \expect_{(s^i_h, a^i_h) \sim d^{\pi^i}_h (\cdot, \cdot)} \bigg[ \sum_{s^\prime \in \gS} \bigg( P_h (s^\prime |s^i_h, a^i_h ) \cdot P^\star_h (s^\prime |s^i_h, a^i_h ) \bigg)^{1/2} \bigg]
    \\
    &= -\frac{1}{2} \expect_{(s^i_h, a^i_h) \sim d^{\pi^i}_h (\cdot, \cdot)} \bigg[ \HES  \bigg( P_h (\cdot |s^i_h, a^i_h ), P^\star_h (\cdot |s^i_h, a^i_h ) \bigg) \bigg] + 1.
\end{align*}
Combining the above two equations, we have that
\begin{align*}
    & \quad \log \lp \expect \ls \exp (Y^i_h (P_h)) | \gF^{i-1} \rs \rp 
    \\
    &\leq -\frac{1}{2} \expect_{(s^i_h, a^i_h) \sim d^{\pi^i}_h (\cdot, \cdot)} \bigg[ \HES  \bigg( P_h (\cdot |s^i_h, a^i_h ), P^\star_h (\cdot |s^i_h, a^i_h ) \bigg) \bigg]
\end{align*}
Then we have that, with probability at least $1-\delta$, for any fixed $P_h \in \gP_h$ and for all $ k \in [K]$, 
\begin{align*}
    &\quad \sum_{i=0}^{k-1} Y^i_h (P_h)  
    \\
    &\leq -\frac{1}{2} \sum_{i=0}^{k-1} \expect_{(s^i_h, a^i_h) \sim d^{\pi^i}_h (\cdot, \cdot)} \bigg[ \HES  \bigg( P_h (\cdot |s^i_h, a^i_h ), P^\star_h (\cdot |s^i_h, a^i_h ) \bigg) \bigg]  + \log (1/\delta). 
\end{align*}
Let $\gP^\prime_h$ be a $\rho$-cover of $\gP_h$ based on \cref{def:covering_number}. By union bound, with probability at least $1-\delta$, $\forall k \in [K], \forall h \in [H], \forall P^\prime_h \in \gP_h$, it holds that
\begin{align*}
    &\quad \sum_{i=0}^{k-1} Y^i_h (P^\prime_h)  
    \\
    &\leq -\frac{1}{2} \sum_{i=0}^{k-1} \expect_{(s^i_h, a^i_h) \sim d^{\pi^i}_h (\cdot, \cdot)} \bigg[ \HES  \bigg( P^\prime_h (\cdot |s^i_h, a^i_h ), P^\star_h (\cdot |s^i_h, a^i_h ) \bigg) \bigg]  + \log (H |\gP^\prime_h|/\delta). 
\end{align*}
Taking a summation over $h \in [H]$ on both sides yields that
\begin{align*}
    &\quad \sum_{i=0}^{k-1} \sum_{h=1}^H Y^i_h (P^\prime_h) \leq - \frac{1}{2} \sum_{i=0}^{k-1} \sum_{h=1}^H  \expect_{(s^i_h, a^i_h) \sim d^{\pi^i}_h (\cdot, \cdot)} \bigg[ \HES  \bigg( P^\prime_h (\cdot |s^i_h, a^i_h ), P^\star_h (\cdot |s^i_h, a^i_h ) \bigg) \bigg]  + \sum_{h=1}^H \log (H |\gP^\prime_h|/\delta).
\end{align*}
Then we obtain that with probability at least $1-\delta$, $\forall P^\prime = (P^\prime_1, \ldots, P^\prime_H) \in \gP^\prime $ with $\gP^\prime = \gP^\prime_1\times \ldots \gP^\prime_H$,
\begin{align*}
    & \quad \LOG^k (P^\star) - \LOG^k (P^\prime) \leq - \sum_{i=0}^{k-1} \sum_{h=1}^H  \expect_{(s^i_h, a^i_h) \sim d^{\pi^i}_h (\cdot, \cdot)} \bigg[ \HES  \bigg( P^\prime_h (\cdot |s^i_h, a^i_h ), P^\star_h (\cdot |s^i_h, a^i_h ) \bigg) \bigg] + 2\sum_{h=1}^H \log (H |\gP^\prime_h|/\delta). 
\end{align*}
For $P^k = (P^k_1, \ldots, P^k_H)$, let $P^\prime_h \in \gP^\prime_h$ be the element that covers $P^k_h$ and we define that $P^\prime = (P^\prime_1, \ldots, P^\prime_H)$. Then we have $\forall (s, a, h, s^\prime) \in \gS \times \gA \times [H] \times \gS, \; | \log (P^k_h (s^\prime|s, a) / P^\prime_h (s^\prime|s, a) ) |\leq \rho $. Then we have that
\begin{align*}
    &\quad \LOG^k (P^k) 
    \\
    &=  - \sum_{i=0}^{k-1} \sum_{h=1}^H \log \lp P^k_h \lp s^i_{h+1} | s^i_h, a^i_h \rp \rp
    \\
    &\geq  - \sum_{i=0}^{k-1} \sum_{h=1}^H \log \lp P^\prime_h \lp s^i_{h+1} | s^i_h, a^i_h \rp \rp - kH \rho
    \\
    &= \LOG^k (P^\prime) - kH \rho. 
\end{align*}
Furthermore, we have that
\begin{align*}
    &\sum_{i=0}^{k-1} \sum_{h=1}^H  \expect_{(s^i_h, a^i_h) \sim d^{\pi^i}_h (\cdot, \cdot)} \bigg[ \HES  \bigg( P^k_h (\cdot |s^i_h, a^i_h ), P^\star_h (\cdot |s^i_h, a^i_h ) \bigg) \bigg]
    \\
    &\leq 2 \sum_{i=0}^{k-1} \sum_{h=1}^H  \expect_{(s^i_h, a^i_h) \sim d^{\pi^i}_h (\cdot, \cdot)} \bigg[ \HES  \bigg( P^k_h (\cdot |s^i_h, a^i_h ), P^\prime_h (\cdot |s^i_h, a^i_h ) \bigg) \bigg] + 2\sum_{i=0}^{k-1} \sum_{h=1}^H  \expect_{(s^i_h, a^i_h) \sim d^{\pi^i}_h (\cdot, \cdot)} \bigg[ \HES  \bigg( P^\prime_h (\cdot |s^i_h, a^i_h ), P^\star_h (\cdot |s^i_h, a^i_h ) \bigg) \bigg]
    \\
    &\leq 2 \sum_{i=0}^{k-1} \sum_{h=1}^H  \expect_{(s^i_h, a^i_h) \sim d^{\pi^i}_h (\cdot, \cdot)} \bigg[ \KL  \bigg( P^k_h (\cdot |s^i_h, a^i_h ), P^\prime_h (\cdot |s^i_h, a^i_h ) \bigg) \bigg] + 2\sum_{i=0}^{k-1} \sum_{h=1}^H  \expect_{(s^i_h, a^i_h) \sim d^{\pi^i}_h (\cdot, \cdot)} \bigg[ \HES  \bigg( P^\prime_h (\cdot |s^i_h, a^i_h ), P^\star_h (\cdot |s^i_h, a^i_h ) \bigg) \bigg]
    \\
    &\leq 2\sum_{i=0}^{k-1} \sum_{h=1}^H  \expect_{(s^i_h, a^i_h) \sim d^{\pi^i}_h (\cdot, \cdot)} \bigg[ \HES  \bigg( P^\prime_h (\cdot |s^i_h, a^i_h ), P^\star_h (\cdot |s^i_h, a^i_h ) \bigg) \bigg] + 2kH \rho.
\end{align*}
Combining the above three inequalities yields that
\begin{align*}
    & \quad \LOG^k (P^\star) - \LOG^k (P^k)  
    \\
    & \leq \LOG^k (P^\star) - \LOG^k (P^\prime) + kH\rho
    \\
    &\leq - \sum_{i=0}^{k-1} \sum_{h=1}^H  \expect_{(s^i_h, a^i_h) \sim d^{\pi^i}_h (\cdot, \cdot)} \bigg[ \HES  \bigg( P^\prime_h (\cdot |s^i_h, a^i_h ), P^\star_h (\cdot |s^i_h, a^i_h ) \bigg) \bigg] + 2\sum_{h=1}^H \log (H |\gP^\prime_h|/\delta) + kH \rho
    \\
    & \leq - \frac{1}{2} \sum_{i=0}^{k-1} \sum_{h=1}^H  \expect_{(s^i_h, a^i_h) \sim d^{\pi^i}_h (\cdot, \cdot)} \bigg[ \HES  \bigg( P^k_h (\cdot |s^i_h, a^i_h ), P^\star_h (\cdot |s^i_h, a^i_h ) \bigg) \bigg] + 2\sum_{h=1}^H \log (H |\gP^\prime_h|/\delta) + 2kH \rho
     \\
    & \leq - \frac{1}{2} \sum_{i=0}^{k-1} \sum_{h=1}^H  \expect_{(s^i_h, a^i_h) \sim d^{\pi^i}_h (\cdot, \cdot)} \bigg[ \HES  \bigg( P^k_h (\cdot |s^i_h, a^i_h ), P^\star_h (\cdot |s^i_h, a^i_h ) \bigg) \bigg] + 2 H \log (H \max_{h \in [H]} |\gP^\prime_h|/\delta) + 2kH \rho
    \\
    & \leq - \frac{1}{2} \sum_{i=0}^{k-1} \sum_{h=1}^H 
    \\
    & \expect_{(s^i_h, a^i_h) \sim d^{\pi^i}_h (\cdot, \cdot)} \bigg[ \HES  \bigg( P^k_h (\cdot |s^i_h, a^i_h ), P^\star_h (\cdot |s^i_h, a^i_h ) \bigg) \bigg] + 2 H \log (H \max_{h \in [H]} \gN_{\rho} (\gP_h; \log) /\delta) + 2kH \rho.
\end{align*}
We finish the proof.

\subsection{Technical Lemmas}
\begin{lem}[Freedman's inequality \citep{agarwal2014taming}]
\label{lem:freedman1}
Let $\{ X_t\}_{t \leq T}$ be a real-valued martingale difference sequence adapted to filtration $\{ \gF_t \}_{t \leq T}$, and let $\E_t[\cdot]=\E[\cdot\  | \ \gF_t]$. If $|X_t|\leq R$ almost surely, then for any $\eta \in [0,\frac{1}{R}]$ it holds that with probability at least $1-\delta$,
$$
	\sum_{t=1}^{T}X_t \leq \eta \sum_{t=1}^{T}\E_{t-1}[X_t^2]+\frac{\log(1/\delta)}{\eta}.
$$
\end{lem}

\begin{lem}[Lemma A.4 in \citep{foster2021statistical}]
    \label{lem:martingale_chernoff}
    Let $\{ X_t \}_{t\leq{}T}$ be a real-valued martingale difference sequence adapted to filtration $\{ \gF_t \}_{t\leq{}T}$, and let $\E_t[\cdot]=\E[\cdot\  | \ \gF_t]$. It holds that with probability at least
    $1-\delta$, for all $T'\leq{}T$,
    \begin{equation*}
      \sum_{t=1}^{T^{\prime}} X_t \leq \sum_{t=1}^{T^{\prime}} \log \left(\mathbb{E}_{t-1}\left[\exp (X_t) \right]\right)+\log \left(1/\delta\right)
    \end{equation*}
\end{lem}

\begin{lem}
\label{lem:perturbabtion_analysis}
    For any reward functions $r, \widehat{r}$, we have that $\forall (s, a, h) \in \gS \times \gA \times [H]$,
    \begin{align*}
    \labs Q^{\star, r}_h (s, a) - Q^{\star, \widehat{r}}_h (s, a) \rabs \leq \sum_{h^\prime=h}^H \max_{s \in \gS, a \in \gA} \labs  r_h (s, a) - \widehat{r}_h (s, a) \rabs.
    \end{align*}
    Here $Q^{\star, r}$ is the optimal Q-value function of $r$.
\end{lem}
\begin{proof}
    According to the Bellman optimality equation, we have that
    \begin{align*}
        &\quad \labs Q^{\star, r}_h (s, a) - Q^{\star, \widehat{r}}_h (s, a) \rabs
        \\
        &= \bigg| r_h (s, a) - \widehat{r}_h (s, a) + \expect_{s^\prime \sim P_h (\cdot|s, a)} \ls \max_{a^\prime \in \gA} Q^{\star, r}_{h+1} (s^\prime, a^\prime) - \max_{a^\prime \in \gA} Q^{\star, \widehat{r}}_{h+1} (s^\prime, a^\prime)  \rs \bigg|
        \\
        &\leq \labs r_h (s, a) - \widehat{r}_h (s, a) \rabs + \expect_{s^\prime \sim P_h (\cdot|s, a)} \ls \labs \max_{a^\prime \in \gA} Q^{\star, r}_{h+1} (s^\prime, a^\prime) - \max_{a^\prime \in \gA} Q^{\star, \widehat{r}}_{h+1} (s^\prime, a^\prime) \rabs  \rs.
    \end{align*}
    We analyze the term $| \max_{a^\prime \in \gA} Q^{\star, r}_{h+1} (s^\prime, a^\prime) - \max_{a^\prime \in \gA} Q^{\star, \widehat{r}}_{h+1} (s^\prime, a^\prime) |$.
    \begin{align*}
        &\quad \max_{a^\prime \in \gA} Q^{\star, r}_{h+1} (s^\prime, a^\prime) - \max_{a^\prime \in \gA} Q^{\star, \widehat{r}}_{h+1} (s^\prime, a^\prime)
        \\
        &= Q^{\star, r}_{h+1} (s^\prime, a^1) - Q^{\star, \widehat{r}}_{h+1} (s^\prime, a^2)
        \\
        &\leq Q^{\star, r}_{h+1} (s^\prime, a^1) - Q^{\star, \widehat{r}}_{h+1} (s^\prime, a^1),
        \\
        & \quad \max_{a^\prime \in \gA} Q^{\star, r}_{h+1} (s^\prime, a^\prime) - \max_{a^\prime \in \gA} Q^{\star, \widehat{r}}_{h+1} (s^\prime, a^\prime) \\
        &= Q^{\star, r}_{h+1} (s^\prime, a^1) - Q^{\star, \widehat{r}}_{h+1} (s^\prime, a^2) 
        \\
        &\geq Q^{\star, r}_{h+1} (s^\prime, a^2) - Q^{\star, \widehat{r}}_{h+1} (s^\prime, a^2). 
    \end{align*}
    Here $a^1 \in \argmax_{a^\prime \in \gA} Q^{\star, r}_{h+1} (s^\prime, a^\prime), a^2 \in \argmax_{a^\prime \in \gA} Q^{\star, \widehat{r}}_{h+1} (s^\prime, a^\prime)$. Thus, we can get that
    \begin{equation}
    \label{eq:max_inequality}
        \begin{split}
            &\quad \labs \max_{a^\prime \in \gA} Q^{\star, r}_{h+1} (s^\prime, a^\prime) - \max_{a^\prime \in \gA} Q^{\star, \widehat{r}}_{h+1} (s^\prime, a^\prime) \rabs
        \\
        &\leq  \max_{a^\prime \in \gA} Q^{\star, r}_{h+1} (s^\prime, a^\prime) - Q^{\star, \widehat{r}}_{h+1} (s^\prime, a^\prime) \nonumber 
        \\
        &\leq \max_{a^\prime \in \gA} \labs  Q^{\star, r}_{h+1} (s^\prime, a^\prime) - Q^{\star, \widehat{r}}_{h+1} (s^\prime, a^\prime)  \rabs.
        \end{split}
    \end{equation}
    Then we have that $\forall (s, a) \in \gS \times \gA$,
    \begin{align*}
        &\quad \labs Q^{\star, r}_h (s, a) - Q^{\star, \widehat{r}}_h (s, a) \rabs
        \\
        &\leq \labs r_h (s, a) - \widehat{r}_h (s, a) \rabs + \expect_{s^\prime \sim P_h (\cdot|s, a)} \ls \labs \max_{a^\prime \in \gA} Q^{\star, r}_{h+1} (s^\prime, a^\prime) - \max_{a^\prime \in \gA} Q^{\star, \widehat{r}}_{h+1} (s^\prime, a^\prime) \rabs  \rs
        \\
        &\leq \labs r_h (s, a) - \widehat{r}_h (s, a) \rabs + \max_{s^\prime \in \gS, a^\prime \in \gA} \labs  Q^{\star, r}_{h+1} (s^\prime, a^\prime) - Q^{\star, \widehat{r}}_{h+1} (s^\prime, a^\prime) \rabs .
    \end{align*}
    Applying the above recursion inequality repeatedly from $h^\prime=h$ to $h^\prime=H$ with $Q^{\star, r}_{H+1} (s,a) =  Q^{\star, \widehat{r}}_{H+1} (s, a) = 0$ completes the proof.  
\end{proof}

\begin{lem}
\label{lem:error_to_sample_complexity}
    For $a \geq 1$ and $\varepsilon \leq 1$, when $K \geq 4 \log (4a/\varepsilon) / \varepsilon^2$, we have that
    \begin{align*}
        \sqrt{\frac{\log (a K)}{K}} \leq \varepsilon.
    \end{align*}
\end{lem}
\begin{proof}
We consider the function $f (K) = \sqrt{\log (a K)/K}$ and calculate the gradient.
\begin{align*}
    f^\prime (K) = \frac{1}{2} \lp \frac{\log (aK)}{K} \rp^{-1/2} \lp \frac{1-\log (aK)}{K^2} \rp.
\end{align*}
When $K \geq 4 \log (4a/\varepsilon) / \varepsilon^2 \geq 4$, we have that $f^\prime (K) \leq 0$, implying that $ f (K)$ is a monotonically decreasing function in this range. Then we have that
    \begin{align*}
        \sqrt{\frac{\log (a K)}{K}} &\leq \sqrt{\frac{\log \lp 4 a \log (4a/\varepsilon) / \varepsilon^2  \rp}{4 \log (4a/\varepsilon)}} \varepsilon
        \\
        &= \sqrt{\frac{ \log (4a/\varepsilon) + \log ( \log (4a/\varepsilon) ) + \log (1/\varepsilon)  }{4 \log (4a/\varepsilon)}} \varepsilon
        \\
        &\overset{(a)}{\leq} \sqrt{\frac{ \log (4a/\varepsilon) +  \log (4a/\varepsilon) + \log (1/\varepsilon)  }{4 \log (4a/\varepsilon)}} \varepsilon
        \\
        &\overset{(b)}{\leq} \varepsilon.
    \end{align*}
    Inequality $(a)$ follows that $\log (x) \leq x + 1$ and inequality $(b)$ follows that $a \geq 1$. 

\end{proof}

%% file: appendix/experiment_details.tex
\section{Implementation Details}
\label{sec:implementation_details}

\subsection{Implementation Details of OPT-AIL}

\textbf{Reward Update. } As mentioned in \cref{subsec:practical_reward_update}, we choose $\psi(r)$ in \cref{eq:practical_reward_update} as the gradient penalty (GP) regularization of the reward model \citep{Arjovsky2017wgan}, which can help stabilize the online optimization process by enforcing 1-Lipschitz continuity of the reward model $r$. Here $\mathcal{D}^I$ is a linear interpolation between the replay buffer $\gD^k$ and expert demonstrations $\gDE$.
\begin{align*}
    \psi(r) = \expect_{\tau\sim\mathcal{D}^I}\ls\sum_{h=1}^H(\|\nabla r_h (s_h,a_h)\| - 1)^2\rs
\end{align*}
\\
\textbf{Model-free Policy Update. } Here we present the implementation details of policy updates. Firstly, to stabilize the training process, we refine the optimism regularization term by subtracting a baseline Q-value function from a random policy $\mu\equiv \text{Unif}(\mathcal{A})$, which has been utilized in \citep{kumar2020conservative, liu2024maximize}. Furthermore, recognizing that initial state samples can be limited and lack diversity, we employ both the replay buffer $\gD^k$ and expert demonstrations $\gDE$ to compute the Q-value loss, which is a common data augmentation approach and has been validated in many deep AIL methods \citep{Kostrikov20value_dice,garg2021iq-learn,viano2022proximal}. Incorporating these two enhancements, we reformulate the Q-value model training objective as follows.
\begin{align*}
         \min_{Q \in \mathcal{Q}} &\expect_{\tau \sim \gD^k\cup\gDE} \ls \sum_{h=1}^H \lp Q_h (s_h, a_h) - r^k_h - \widebar{Q}_{h+1} (s_{h+1}, \pi^k) \rp^2 \rs - \lambda \expect_{\tau \sim \gD^k\cup\gDE} \ls\sum_{h=1}^H \lp Q_h(s_h,\pi^k) - Q_h(s_h,\mu)\rp\rs .
\end{align*}
\\
\textbf{Model-based Policy Update. }Here we present the implementation details of the model-based policy update. In particular, similar to the model-free policy update, we utilize the data augmentation technique, which uses both the replay buffer $\gD^k$ and expert demonstrations $\gDE$ to calculate the negative log-likelihood loss for transition functions.
\begin{equation*}
    \begin{split}
        &\nabla \ell^k (P) := - \expect_{\tau \sim \gD^k \cup \gDE} \ls \sum_{h=1}^H \nabla \log \lp P_h \lp s_{h+1} | s_h, a_h \rp \rp \rs
    \\
    & - \lambda_{P} \expect_{\tau\sim\gD^k \cup \gDE} \bigg[  \sum_{h=1}^H \nabla \log \lp P_h \lp s_{h+1} | s_h, a_h \rp \rp \bigg( r^k_h (s_h, a_h) + Q^{\pi^{k-1}, P, r^k}_{h+1} (s_{h+1}, a_{h+1}) - Q^{\pi^{k-1}, P, r^k}_h (s_{h}, a_h) \bigg)  \bigg].
    \end{split}
\end{equation*}

\subsection{Architecture and Training Details}
The experiments are conducted on a machine with 64 CPU cores and 4 RTX4090 GPU cores. Each experiment is replicated five times using different random seeds. For each task, we adopt online DrQ-v2 \citep{yarats2021mastering} to train an agent with sufficient environment interactions and regard the resultant policy as the expert policy. Specifically, we use 3M training steps for \texttt{Cheetah Run}, \texttt{Hopper Hop}, and \texttt{Walker Run}, and 1M training steps for other tasks. Then we roll out this expert policy to collect expert demonstrations. The architecture and training details of OPT-AIL and all baselines are listed below.
\begin{itemize}
    \item \textbf{MF OPT-AIL and MB OPT-AIL:} Our codebase of OPT-AIL and MB OPT-AIL extends the open-sourced framework of \href{https://github.com/Div99/IQ-Learn}{IQLearn}. We retain the structure and parameter design of the actor and critic from the original framework while employing SAC \citep{haarnoja2018sac} with a fixed temperature for policy updates. We also implement a discriminator with a similar architecture to the critic network, and additionally incorporate layer normalization and tanh activation before the output to improve training stability. For MB OPT-AIL, we assume the dynamics model outputs a Gaussian distribution with a fixed standard deviation of $0.01$. We use a two-layer MLP as the backbone of the dynamics model, which outputs the mean of the distribution. A comprehensive enumeration of the hyperparameters of OPT-AIL and MB OPT-AIL is provided in Table \ref{tab:params} and Table \ref{tab:mb_params}, respectively.

    \item \textbf{BC:} We implement BC based on our codebase. The actor model is trained using Mean Squared Error (MSE) loss over 10k training steps;

    \item \textbf{PPIL:} We use the author's codebase, which is available at \href{https://github.com/lviano/p2il}{https://github.com/lviano/p2il};

    \item \textbf{IQLearn:} We use the author's codebase, which is available at \href{https://github.com/Div99/IQ-Learn}{https://github.com/Div99/IQ-Learn};

    \item \textbf{FILTER:} We use the author's codebase, which is available at \href{https://github.com/gkswamy98/fast_irl}{https://github.com/gkswamy98/fast\_irl};

    \item \textbf{HyPE and HyPER:} We use the author's codebase, which is available at \href{https://github.com/gkswamy98/hyper}{https://github.com/gkswamy98/hyper}.

    \item \textbf{CMIL:} We adapt the original CMIL codebase from \href{https://github.com/victorkolev/cmil}{https://github.com/victorkolev/cmil}, which is designed for vision-based tasks, to support vector-based tasks.    

\end{itemize}

We emphasize that for a fair comparison, all algorithms use the same hyperparameters (or the default in the original implementation) except for the gradient penalty coefficient. Specifically, in OPT-AIL, the gradient penalty coefficient is set to 1 for \texttt{Cartpole Swingup}, \texttt{Walker Walk}, and \texttt{Walker Stand}, and 10 for other tasks. For baselines, the gradient penalty coefficient is always set to 10 as provided by the authors. We also attempt to adjust this parameter for the baselines but find that the default parameters provided by the authors work well.

\begin{table}[htbp]
	\renewcommand{\arraystretch}{1.1}
	\centering
	\caption{MF OPT-AIL Hyper-parameters.}
	\label{tab:params}
	\vspace{1mm}
	\begin{tabular}{l l| l }
	\toprule
	\multicolumn{2}{l|}{Parameter} &  Value\\
	\midrule
	& discount ($\discount$) &  0.99\\
        & gradient penalty coefficient ($\beta$) & 1, 10\\
        & optimism regularization coefficient ($\lambda$) & $10^{-3}$ \\
        & temperature ($\alpha$) & $10^{-2}$\\
	& replay buffer size & $5\cdot10^5$\\
	& batch size & 256\\
	& optimizer & Adam \\
        \multicolumn{2}{l|}{\it{Discriminator}}& \\
        & learning rate & $3 \cdot 10^{-5}$\\  
        & number of hidden layers  & 2\\
        & number of hidden units per layer  & 256\\
        & activation & ReLU\\
        \multicolumn{2}{l|}{\it{Actor}}& \\
        & learning rate & $3 \cdot 10^{-5}$\\
        & number of hidden layers  & 2\\
        & number of hidden units per layer  & 256\\
        & activation & ReLU\\
        \multicolumn{2}{l|}{\it{Critic}}& \\
        & learning rate & $3 \cdot 10^{-4}$\\  
        & number of hidden layers  & 2\\
        & number of hidden units per layer  & 256\\
        & activation & ReLU\\
\bottomrule
\end{tabular}
\end{table}

\begin{table}[htbp]
	\renewcommand{\arraystretch}{1.1}
	\centering
	\caption{MB OPT-AIL Hyper-parameters.}
	\label{tab:mb_params}
	\vspace{1mm}
	\begin{tabular}{l l| l }
	\toprule
	\multicolumn{2}{l|}{Parameter} &  Value\\
	\midrule
	& discount ($\discount$) &  0.99\\
        & gradient penalty coefficient ($\beta$) & 1, 10\\
        & optimism regularization coefficient ($\lambda_P$) & $10^{-2}$ \\
        & temperature ($\alpha$) & $10^{-2}$\\
	& replay buffer size & $5\cdot10^5$\\
	& batch size & 256\\
        & dynamics rollout horizon & 1 \\
        & generated data ratio & 0.2\\
	& optimizer & Adam \\
        \multicolumn{2}{l|}{\it{Discriminator}}& \\
        & learning rate & $3 \cdot 10^{-5}$\\  
        & number of hidden layers  & 2\\
        & number of hidden units per layer  & 256\\
        & activation & ReLU\\
        \multicolumn{2}{l|}{\it{Actor}}& \\
        & learning rate & $3 \cdot 10^{-5}$\\
        & number of hidden layers  & 2\\
        & number of hidden units per layer  & 256\\
        & activation & ReLU\\
        \multicolumn{2}{l|}{\it{Critic}}& \\
        & learning rate & $3 \cdot 10^{-4}$\\  
        & number of hidden layers  & 6\\
        & number of hidden units per layer  & 256\\
        & activation & ReLU\\
        \multicolumn{2}{l|}{\it{Dynamics}}& \\
        & learning rate & $3 \cdot 10^{-5}$\\  
        & number of hidden layers  & 2\\
        & number of hidden units per layer  & 256\\
        & activation & ReLU\\
        & number of ensemble & 7\\
\bottomrule
\end{tabular}
\end{table}

\section{Additional Experimental Results}
\label{appendix:results}
In this section, we list the learning curves for 8 DMControl tasks with 4, 7, and 10 expert trajectories respectively. The corresponding results are depicted in Figure \ref{fig:expert_4}, Figure \ref{fig:expert_7}, and Figure \ref{fig:expert_10}. Here the x-axis is the number of environment interactions and the y-axis is the return. The solid lines are the mean of results while the shaded region corresponds to the standard deviation over 5 random seeds. Among the model-based algorithms, MB OPT-AIL consistently matches or exceeds HyPER's performance in terms of interaction efficiency. Similarly, MF OPT-AIL achieves comparable or better interaction efficiency compared to prior model-free approaches.

\begin{figure*}[htbp]
    \centering
    \includegraphics[width=\linewidth]{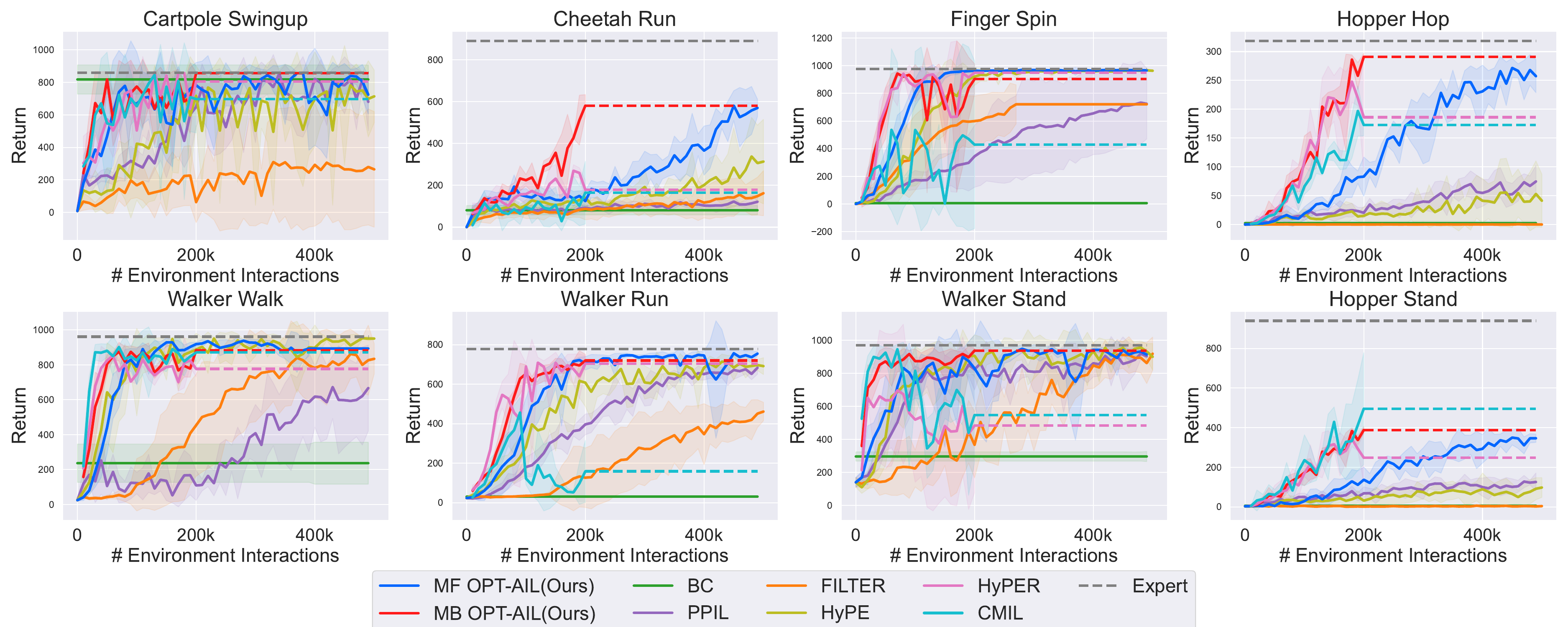}
    \caption{Learning curves on 8 DMControl tasks over 5 random seeds using 4 expert trajectories.}
    \label{fig:expert_4}
\end{figure*}

\begin{figure*}[htbp]
    \centering
    \includegraphics[width=\linewidth]{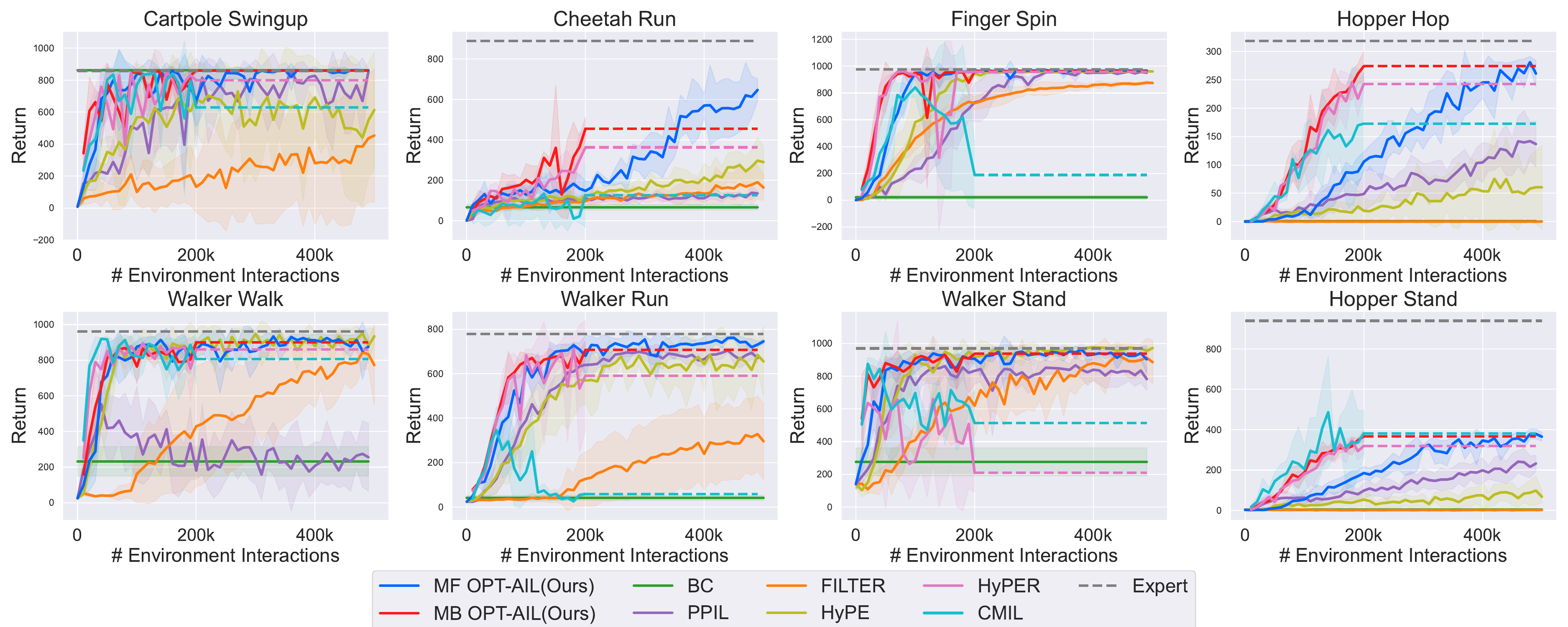}
    \caption{Learning curves on 8 DMControl tasks over 5 random seeds using 7 expert trajectories.}
    \label{fig:expert_7}
\end{figure*}

\begin{figure*}[htbp]
    \centering
    \includegraphics[width=\linewidth]{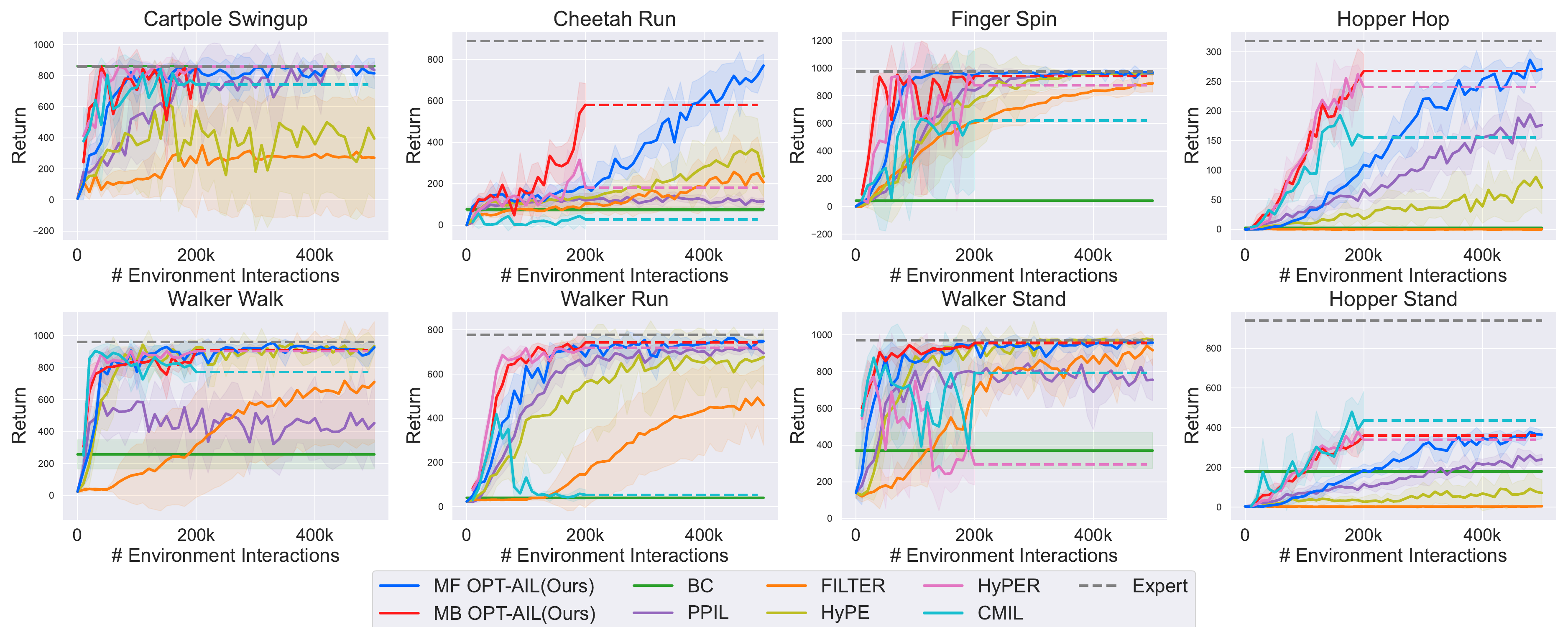}
    \caption{Learning curves on 8 DMControl tasks over 5 random seeds using 10 expert trajectories.}
    \label{fig:expert_10}
\end{figure*}